\documentclass[11pt,letterpaper]{article}
\usepackage[letterpaper,margin=1in]{geometry}
\usepackage[T1]{fontenc}
\usepackage{lmodern,microtype}
\usepackage{amsmath,amssymb,amsthm,mathtools}
\usepackage{booktabs,array,graphicx,longtable,enumitem}
\usepackage{tikz}
\usetikzlibrary{arrows.meta,positioning}
\usepackage[hidelinks]{hyperref}
\hypersetup{pdftitle={The Optimal Sample Complexity of Pointwise Replicable Agnostic Learning},pdfauthor={Guangjian Zhang},pdfsubject={Zenodo preprint; DOI: 10.5281/zenodo.23446957},pdfcreator={LaTeX}}
\renewenvironment{abstract}{\begin{center}\bfseries\abstractname\end{center}\begin{quotation}\noindent}{\end{quotation}}
\newtheorem{theorem}{Theorem}[section]
\newtheorem{lemma}[theorem]{Lemma}
\newtheorem{proposition}[theorem]{Proposition}

\theoremstyle{definition}
\newtheorem{definition}[theorem]{Definition}
\newcommand{\E}{\mathbb E}
\newcommand{\Pp}{\mathbb P}
\newcommand{\one}{\mathbf 1}
\newcommand{\OPT}{\operatorname{OPT}}
\newcommand{\Var}{\operatorname{Var}}
\newcommand{\clip}{\operatorname{clip}}
\newcommand{\norm}[1]{\left\lVert #1\right\rVert}
\newcommand{\abs}[1]{\left\lvert #1\right\rvert}
\newcommand{\eps}{\varepsilon}
\newcommand{\sgn}{\operatorname{sgn}}
\newcommand{\legacyref}[2]{\hyperlink{legacy:#1}{#2~S#1}}
\newcommand{\legacyheading}[2]{\subsubsection{#2}\label{legacy:section:#1}}
\newenvironment{legacytheorem}[2]{\par\medskip\noindent\textbf{#1}\def\thmtmp{#2}\ifx\thmtmp\empty\else\ (#2)\fi\textbf{.}\ \itshape\ignorespaces}{\par\medskip}
\newenvironment{legacydefinition}[2]{\par\medskip\noindent\textbf{#1}\def\thmtmp{#2}\ifx\thmtmp\empty\else\ (#2)\fi\textbf{.}\ \ignorespaces}{\par\medskip}
\title{The Optimal Sample Complexity of\\Pointwise Replicable Agnostic Learning}
\author{Guangjian Zhang\\\href{mailto:guangjian.zhang@student.unsw.edu.au}{\texttt{guangjian.zhang@student.unsw.edu.au}}}
\date{}
\begin{document}
\maketitle
\begin{center}
DOI: \href{https://doi.org/10.5281/zenodo.23446957}{10.5281/zenodo.23446957}
\end{center}
\begin{abstract}
Pointwise replicability asks that two independently trained predictors, using the same random seed,
agree at each fixed query with high probability. We determine its optimal agnostic sample complexity
on finite and countable domains. For every nonempty binary class of finite VC dimension $d\ge1$, the
required number of independent examples is
\[
 \Theta\!\left(\frac{d}{\rho^2\alpha^2}+\frac{\log(1/\beta)}{\alpha^2}\right),
\]
uniformly over independent accuracy, failure, and disagreement parameters
$\alpha,\beta,\rho\in(0,1/10)$. The learner returns one binary function, may be improper, and uses
unrestricted shared randomness. Its sample budget is fixed before observing data, and the
replicability guarantee includes queries of zero marginal mass. The main construction combines
an ensemble of empirical-risk minimizers with a smooth, signal-dependent risk decomposition.
A single shared phase selects the final binary predictor. Record times certify its weak-signal
risk, while a uniform bound protects the strong-signal risk at that same selected time.
A phase-comparison theorem controls the actual shared-seed execution under simultaneous changes
in the marginals and risk coefficients. A rare repair step then makes the confidence cost
additive. Together with matching lower bounds, this removes the remaining logarithmic overhead
and parameter restrictions in the general countable-domain sample-complexity problem.
This extended version also retains the earlier explicit constructions and their proofs,
including proper learning for VC-one classes, the full finite-atom class, a faster
realizable bound, and counterexamples that separate restrictions of particular algorithms
from restrictions of the learning problem.
\end{abstract}
\vfill
\begingroup
\small
\noindent\textbf{All rights reserved}\par
\noindent Copyright \textcopyright\ 2026 Guangjian Zhang. All rights reserved. Except as permitted by
applicable law or under the license granted to arXiv upon submission, no reproduction,
distribution, modification, or commercial use is permitted
without the copyright holder's prior written permission.
\par
\endgroup
\clearpage
\section{Introduction}
An accurate predictor need not make the same decision after retraining. Pointwise replicability
addresses this instability at the level of a fixed query: two independent training samples,
processed with the same random seed, should produce the same prediction with high probability.
Unlike agreement of entire learned functions, this requirement applies separately to each
query. It therefore offers a natural way to study reproducible prediction without requiring
the learned models themselves to coincide.

Hopkins, Impagliazzo, and Ye~\cite{HIY} introduced pointwise replicability and showed that every binary VC class can be
learned in this sense. Their countable-domain agnostic bound has the form
$\widetilde O(d/(\rho^2\alpha^2)+1/(\rho^2\alpha^4))$, where $d$ is the VC dimension,
$\alpha$ is excess risk, and $\rho$ is the allowed disagreement probability. They asked whether
the additional accuracy-dependent term was necessary. Our first version, published on
9 September 2026~\cite{Zhang2026}, answered this question: the additional term is unnecessary.
It also established the exact sample-complexity order at fixed confidence.
This version resolves the joint problem: what is the exact sample cost when accuracy,
confidence, and replicability vary independently?

We prove that the answer is
\begin{equation}\label{eq:intro-rate}
 \Theta\!\left(\alpha^{-2}\left[d\rho^{-2}+\log(1/\beta)\right]\right),
\end{equation}
where $\beta$ is the allowed accuracy-failure probability. The constants are universal and
the bound holds for every class of positive finite VC dimension on a countable domain. Thus replicability
multiplies the dimension-dependent estimation cost, but not the confidence cost. This identifies
the precise statistical price of stable prediction in the agnostic model.

\paragraph{Why the last step is difficult.}
Averaging independently trained predictors makes each fractional prediction stable. Turning
those fractions into one accurate binary function is the obstacle. Independent shared
rounding controls a fixed prediction, but its total risk can fluctuate. Selecting a rounding
with small empirical risk addresses accuracy, yet couples every query to a data-dependent
global choice. A small change far from the query can then change its prediction. Even a
query absent from the data must remain stable.

Three requirements must therefore be met together. First, the rounding must be accurate
with failure probability on the scale of $\rho$, rather than merely with constant probability.
Second, its global selection must be stable under the actual common random seed. Third,
confidence amplification must not multiply the entire budget by $\log(1/\beta)$. These
requirements suggest the following construction.

\paragraph{Our construction and technical contributions.}
We train an ensemble, estimate normalized label signals on a separate sample, and use smooth
weights to split risk into weak and strong parts. Both parts retain the same ensemble as their
reference. Independent coordinate phases generate a curve of binary functions. We choose one
time on this curve by sampling increases of a running maximum of integrated empirical risk
gain. That choice certifies weak risk. A separate strong-risk bound holds at every time, so
it also holds at the chosen time. The output is one function, not a query-dependent selection
of different models.

Two estimates make this construction useful. The phase-comparison theorem
(Theorem~\ref{thm:phase}) controls prediction disagreement by a local marginal error and a
global squared-energy error. It allows both the marginals and the risk weights to vary.
The strong-risk theorem (Theorem~\ref{thm:risk}) keeps the random screening weights coupled
to their estimated signs. It uses the risk improvement of correctly oriented coordinates
to pay for uniform control over the phase curve. Together they yield a base learner with
accuracy failure and disagreement both of order $\rho$, without an extra logarithmic factor
(Proposition~\ref{prop:base}). A validation-based rare repair then gives arbitrary confidence
at the additive cost in \eqref{eq:intro-rate} (Theorem~\ref{thm:upper}).

\paragraph{Relation to earlier work.}
The stronger notion of full-output replicability was introduced by Impagliazzo, Lei, Pitassi, and Sorrell~\cite{ILPS}.
Hopkins, Impagliazzo, and Ye~\cite{HIY} introduced the pointwise formulation and analyzed
an ensemble-based learner. Our result determines the optimal agnostic sample complexity
in their countable-domain model.
The September version~\cite{Zhang2026} established matching agnostic lower bounds and
exact additive-confidence rates for several structural cases. Here a uniform upper bound
for arbitrary VC classes and all three parameters matches those lower bounds.
Related confidence-boosting methods include the separate confidence term of
Banihashem et al.~\cite{BBEGH} and the rare fallback of Aamand et al.~\cite{AAC}.
Our September version established the agnostic pointwise repair. The new phase
construction achieves the base guarantee needed to attain the optimal sample budget
uniformly in accuracy, confidence, and replicability.
The ensemble analysis uses the log-free VC expectation bound following from Haussler's
packing estimate~\cite{Haussler1995}; Appendix~\ref{app:statistics} gives the derivation.

\paragraph{The September version and the present extension.}
The first version of this work was published on Zenodo on \textbf{9 September 2026},
under the version-specific DOI
\href{https://doi.org/10.5281/zenodo.22658876}{10.5281/zenodo.22658876}~\cite{Zhang2026}.
That version already states and proves the exact fixed-confidence rate, the matching
agnostic lower bounds in every positive VC dimension, the exact proper bounds for
VC-one and full finite-atom classes, the realizable upper bound, and the gate-and-round,
SOFT-VC, and rare-repair guarantees retained here. Table~\ref{tab:provenance} identifies their original
statement numbers and PDF pages and locates their proofs in this manuscript.
The general all-parameter upper bound and its phase construction are the subsequent
extension developed in this version.

\paragraph{Reading order.}
Section~\ref{sec:model} states the model and final rate. Section~\ref{sec:construction}
builds the base learner in execution order and states the guarantees that certify each
stage. Sections~\ref{sec:confidence} and~\ref{sec:lower} complete the upper and lower
bounds. Section~\ref{sec:retained} presents the complementary proper and realizable
guarantees and the earlier explicit algorithms. All proofs are in the appendices,
beginning with the dependency maps in Appendix~\ref{app:roadmap}.
Appendix~\ref{app:retained} retains the full earlier mathematical development and explains
its relation to the new theorem.

\section{Model and optimal rate}\label{sec:model}
Let $\mathcal X$ be a finite or countably infinite set with its discrete sigma field.
Let $\varnothing\ne\mathcal H\subseteq\{0,1\}^{\mathcal X}$ have finite VC dimension
$d$. For a distribution $D$ on $\mathcal X\times\{0,1\}$, define
\[
 R(h)=\Pp_D\{h(X)\ne Y\},\qquad \OPT=\inf_{h\in\mathcal H}R(h).
\]
The output need not belong to $\mathcal H$. Its space carries the coordinate sigma field.
The learner takes exactly $m$ independent examples, with $m$ chosen before sampling,
and a public seed independent of the data. There is no bound on the number of shared
random bits or on computation.

\begin{definition}[Accuracy and pointwise replicability]\label{def:model}
A measurable learner $A$ is an agnostic $(\alpha,\beta)$-learner if, for every $D$,
$\Pp_{S,\omega}\{R(A(S;\omega))>\OPT+\alpha\}\le\beta$.
It is $\rho$-pointwise replicable if, for every $D$ and every fixed $x\in\mathcal X$,
\[
 \Pp_{S,S',\omega}\{A(S;\omega)(x)\ne A(S';\omega)(x)\}\le\rho,
\]
where $S,S'$ are independent samples from $D^m$ and the entire seed $\omega$ is shared.
Write $m_{\mathcal H}(\alpha,\beta,\rho)$ for the smallest budget achieving both properties.
\end{definition}

The quantifier over $x$ includes points of zero $D_X$ mass. It is outside the probability:
agreement is required for each fixed point, not simultaneously over all points. One run
nevertheless defines one full binary function before any query is answered.

\begin{theorem}[Optimal sample complexity]\label{thm:main}
There are universal constants $c,C>0$ such that, for every such class with $d\ge1$ and
every independent choice $\alpha,\beta,\rho\in(0,1/10)$,
\[
 c\alpha^{-2}\bigl(d\rho^{-2}+\log(1/\beta)\bigr)
 \le m_{\mathcal H}(\alpha,\beta,\rho)
 \le C\alpha^{-2}\bigl(d\rho^{-2}+\log(1/\beta)\bigr).
\]
If $d=0$, then $m_{\mathcal H}(\alpha,\beta,\rho)=0$.
\end{theorem}
The upper bound is Theorem~\ref{thm:upper}; the lower bound is
Theorem~\ref{thm:lower}. Their combination is recorded in Appendix~\ref{app:final}.
The constants do not depend on the domain, the class, or any of the three parameters.

\paragraph{Conventions used by the algorithm.}
Fix an enumeration of $\mathcal X$ and a hypothesis $h_0\in\mathcal H$.
For each finite labeled sample, fix one proper empirical-risk minimizer (ERM) by a
predetermined rule. A minimizer exists because the set of possible empirical risks is finite.
There are countably many finite samples, so these choices define a measurable map even
when $\mathcal H$ is uncountable. The population infimum need not be attained.
All logarithms are natural. Symbols $C,c$ denote positive universal constants that may
change from one bound to another.

For a fractional predictor $q:\mathcal X\to[0,1]$ and a weight vector
$\lambda:\mathcal X\to[0,1]$, define the affine weighted risk
\begin{equation}\label{eq:weighted-risk}
 R_\lambda(q)=\sum_{i\in\mathcal X}\lambda_i
       \bigl[D(i,1)(1-q_i)+D(i,0)q_i\bigr].
\end{equation}
Here $q_i=q(i)$ and $D(i,y)=\Pp_D\{X=i,Y=y\}$. We write $R(q)$ when
$\lambda_i=1$ everywhere. On an independent validation sample of size $v$, its estimate is
\[
 \widehat R_\lambda(q)=\frac1v\sum_{j=1}^v\lambda_{X_j}
       \bigl[Y_j(1-q_{X_j})+(1-Y_j)q_{X_j}\bigr].
\]
The same definitions apply to binary functions. These risks are measurable and their
series converge absolutely.

\section{Constructing the base learner}\label{sec:construction}
This section specifies the complete base algorithm. Its inputs are $\mathcal H$, a fixed
sample, and two internal tolerances $\eps,r\in(0,1/4]$. Its output is one binary function
with excess risk $O(\eps)$, failure probability $O(r)$, and pointwise disagreement $O(r)$.
Section~\ref{sec:confidence} converts these guarantees to the external parameters.

Reserve three disjoint sample portions: an ensemble portion, a histogram pool, and a
validation block. Fix their sizes and the auxiliary scales by
\begin{equation}\label{eq:parameters}
 \begin{split}
 T&=\lceil r^{-2}\rceil,\quad n=\lceil(d+1)/\eps^2\rceil,\quad
 M=\left\lceil64/(\eps^2r^2)\right\rceil,\quad v=\left\lceil1/(\eps^2r^2)\right\rceil,\\
 L&=\log(8T/r),\qquad s=\eps/L,\qquad s_0=\eps.
 \end{split}
\end{equation}
The public seed is the product $\omega=(K,(U_i)_{i\in\mathcal X},Z)$.
Its components are mutually independent: $K$ is Poisson with mean $M$, each $U_i$ is
uniform on $[0,1)$, and $Z$ is uniform on $(0,1)$. The entire seed is independent of
all data. Both executions use exactly the same $K$, phases $U_i$, and inverse-CDF seed $Z$.

\subsection{Average independent ERMs}\label{mod:ensemble}
Split the first $Tn$ examples into $T$ blocks of size $n$. Train the predetermined proper
ERM $h_j$ on block $j$, and return the fractional predictor and its distance to an endpoint:
\[
 p_i=\frac1T\sum_{j=1}^T h_j(i),\qquad m_i=\min(p_i,1-p_i).
\]
In particular $p_i$ lies on the $T$-grid. This discretization will also limit the number
of arc lengths needed in the strong-risk analysis.

\begin{proposition}[Ensemble guarantees]\label{prop:ensemble}
The ensemble satisfies $\Pp\{R(p)>\OPT+C\eps\}\le C/T$.
For an independent ensemble $p'$ trained in the same way,
$\E(p_i-p_i')^2\le1/(2T)$ at every fixed coordinate $i$.
\end{proposition}
This supplies an accurate reference and a stable fractional prediction. The log-free
VC input and the proof are in Appendix~\ref{app:statistics}.

\subsection{Estimate label signals and adjust the marginals}\label{mod:histogram}
The histogram pool has $2M$ examples. If $K>2M$, output $h_0$ and stop.
Otherwise inspect the first $K$ examples in that pool. For each $i$, let $N_i$ be its
number of occurrences and let $S_i$ be the sum of the signs $2Y-1$ on those occurrences.
Define
\[
 \widehat b_i=S_i/M,\qquad
 \widehat\theta_i=\begin{cases}S_i/\sqrt{MN_i},&N_i>0,\\0,&N_i=0.\end{cases}
\]
For interpretation, write $w_i=D_X(i)$, $b_i=D(i,1)-D(i,0)$, and $a_i=|b_i|$.
The corresponding normalized population signal is $\theta_i=b_i/\sqrt{w_i}$ for
$w_i>0$, and zero otherwise. These population quantities are used only in the analysis.
A binary prediction change has risk cost $a_i$, whereas
$b_i^2=w_i\theta_i^2$ determines its contribution to squared risk fluctuations.

Use the following fixed continuous profiles:
\begin{align*}
 f(z)&=\clip(z,-1,1), &
 \ell(z)&=\begin{cases}1,&|z|\le2,\\3-|z|,&2<|z|<3,\\0,&|z|\ge3,\end{cases}\\
 k(z)&=\begin{cases}0,&|z|\le1/2,\\2|z|-1,&1/2<|z|<1,\\1,&|z|\ge1,\end{cases}
 &\eta(t)&=\begin{cases}0,&t\le1/4,\\4t-1,&1/4<t<1/2,\\1,&t\ge1/2,
 \end{cases}\quad(t\ge0).
\end{align*}
Here $\clip(z,-1,1)=\max(-1,\min(z,1))$. Set
\begin{equation}\label{eq:qdef}
 \begin{split}
 u_i&=\ell(\widehat\theta_i/\eps),\qquad c_i=1-u_i,\qquad
 g_i=\one\{S_i>0\},\\
 q_i^0&=p_i+m_i f(\widehat\theta_i/\eps),\qquad
 k_i=k(\widehat\theta_i/\eps),\qquad \eta_i=\eta(|\widehat b_i|/s),\\
 q_i&=(1-k_i\eta_i)q_i^0+k_i\eta_i g_i.
 \end{split}
\end{equation}
The output of this stage is $(q,u,c)$. All three vectors take values in $[0,1]$ and
$u+c=1$. First $q^0$ shifts the ensemble toward the estimated label direction while
preserving its endpoints. The two gates then move it farther toward the estimated bit
when both signal and risk cost are large enough. The weights $u,c$ define two risk
components of the single output predictor. If $S_i=0$,
both gates vanish, so the convention $g_i=0$ has no effect.

The normalized count has mean-square error at most $2/M$, uniformly in $i$
(Lemma~\ref{lem:counts}). The smooth profiles convert that estimate into the marginal
and weighted stability bounds of Lemma~\ref{lem:qstability}. The role of the cost gate
is different: it makes every harmful large-cost move chargeable to squared histogram
error. The resulting risk guarantees are stated after the phase selection is defined.

\subsection{Select one binary function by its record times}\label{mod:phase}
Use the public coordinate phases to form the entire curve
\begin{equation}\label{eq:curve}
 F_t(i)=\one\{(U_i+t)\bmod1<q_i\},\qquad 0\le t\le1.
\end{equation}
Each coordinate is one circular arc and has time average $q_i$.
On the independent validation block of size $v$, calculate
\begin{equation}\label{eq:algorithmprimitive}
 \begin{split}
 \widehat W(t)&=\int_0^t[s_0+\widehat R_u(q)-\widehat R_u(F_z)]\,dz,\\
 \widehat M(t)&=\max_{0\le z\le t}\widehat W(z),\qquad
 \widehat G(t)=\widehat M(t)/\widehat M(1).
 \end{split}
\end{equation}
The maximum includes $z=0$. Also $\widehat W(1)=s_0$, so the denominator is positive.
Choose
\[
 \tau=\inf\{t\in[0,1]:\widehat G(t)\ge Z\}
\]
and output $F_\tau$. This is the sole selected time for all queries. Computationally,
only coordinates observed in validation enter $\widehat W$, which is a finite piecewise
affine function. Its running maximum and inverse are therefore defined by finitely
many breakpoints. The full predictor uses the already fixed $\tau$ and indexed phases.

\paragraph{What makes global selection stable?}
For any absolutely summable real vector $B$ and $q\in[0,1]^{\mathcal X}$, define
\[
 W_{B,q}(t)=s_0t+\sum_i B_i\int_0^t(F_z(i)-q_i)\,dz.
\]
Let $G_{B,q}$ be its normalized running maximum, as in
\eqref{eq:algorithmprimitive}, and let $\Phi(B,q)=F_{G_{B,q}^{-1}(Z)}$.
This defines one sampler for all inputs using the same $(U_i),Z$.
The following theorem is the principal coupling tool for this module.

\begin{theorem}[Shared-seed phase comparison]\label{thm:phase}
Suppose $(B,q,B',q')$, as one joint random input, is independent of the entire seed
$((U_i),Z)$. The inputs may otherwise be dependent. For each fixed $x$,
\begin{equation}\label{eq:phasemain}
 \begin{split}
 \Pp\{\Phi(B,q)(x)\ne\Phi(B',q')(x)\}
 \le {}&\E|q_x-q_x'|\\
 &+\frac{18\sqrt2}{s_0}
 \sqrt{\E\sum_i\left[B_i^2(q_i-q_i')^2+(B_i-B_i')^2\right]}.
 \end{split}
\end{equation}
\end{theorem}
The coefficients for the true weak risk are $B_i=b_i u_i$.
This estimate controls the actual algorithmic coupling, including the dependence between
the query phase $U_x$ and the selected time. Appendix~\ref{app:phase} proves the theorem and controls the
additional effect of empirical validation. Lemma~\ref{lem:qstability} supplies the
energy bounds needed to obtain $O(r)$ disagreement in Proposition~\ref{prop:base}.

\subsection{Certify both risk parts at the same time}\label{mod:risk}
The selected time is chosen using only weak weighted risk. To certify the total risk,
we need a weak certificate at that time and a strong certificate valid for every time.
The following statement gives precisely those two guarantees.

\begin{theorem}[Compatible risk certificates]\label{thm:risk}
In the ideal Poisson experiment described in Appendix~\ref{app:statistics}, the construction
satisfies, outside an event of probability at most $Cr$,
\[
 R_u(F_\tau)\le R_u(p)+5\eps,\qquad
 \sup_{0\le t\le1}\{R_c(F_t)-R_c(p)\}\le C\eps.
\]
The empirical certificate $\widehat R_u(F_\tau)\le\widehat R_u(q)+s_0$ holds almost surely.
\end{theorem}
The first conclusion is proved by the record certificate and the weak-reference estimate
(Lemmas~\ref{lem:record} and~\ref{lem:weakrisk}). The second is
Lemma~\ref{lem:strong}. Their proofs appear together in Appendix~\ref{app:accuracy}.
For every realized histogram, $u+c=1$, so these certificates add directly at the same
$\tau$.

\begin{proposition}[Base learner]\label{prop:base}
There are universal constants $C_a,C_f,C_p,C_m$ such that the complete algorithm,
including its overflow fallback, uses the deterministic budget
$Tn+2M+v\le C_m(d+1)/(\eps^2r^2)$ and satisfies
\[
 \Pp\{R(F_\tau)>\OPT+C_a\eps\}\le C_f r,
 \qquad
 \Pp\{F_\tau(x)\ne F_{\tau'}'(x)\}\le C_p r
\]
for every distribution and every fixed $x$.
\end{proposition}
Here $F_\tau$ denotes the algorithm's output, including $h_0$ on overflow.
The proof in Appendix~\ref{app:base} assembles the ensemble, risk, phase, and histogram
guarantees, and separately checks null-mass queries and the fixed sample budget.

\section{Make confidence additive}\label{sec:confidence}
Choose the constants of Proposition~\ref{prop:base} once, and run the base learner with
\[
 \eps=\frac{\alpha}{4\max(C_a,1)},\qquad r=\frac\rho{C_*},\qquad
 C_*\ge\max(4,64C_f,2C_p).
\]
Call its output $f$. It has excess at most $\alpha/4$ except with probability $\rho/64$,
and pointwise disagreement at most $\rho/2$. The following repair uses two new independent
blocks, whose sizes are fixed in advance.

Set $\eta_0=\min(\beta/8,\rho/128)$. On the first block, train a proper reference ERM
$g$ with excess at most $\alpha/8$ and failure at most $\eta_0$. On the second, estimate
both $R(f)$ and $R(g)$ to error at most $\alpha/16$, simultaneously with failure at most
$\eta_0$. Return
\begin{equation}\label{eq:repair}
 \widehat f=\begin{cases}
 f,&\widehat R(f)\le\widehat R(g)+\alpha/2,\\
 g,&\widehat R(f)>\widehat R(g)+\alpha/2.
 \end{cases}
\end{equation}
The reference protects accuracy. Its properness also ensures $R(g)\ge\OPT$ for every
realization. As a result, a successful base run will not be replaced when validation succeeds.
The repair branch is rare, and its probability rather than the stability of $g$ pays for
any additional disagreement.

\begin{theorem}[Upper bound with additive confidence]\label{thm:upper}
For $d\ge1$ and $\alpha,\beta,\rho\in(0,1/4]$, the repaired learner is an agnostic
$(\alpha,\beta)$-learner and is $\rho$-pointwise replicable, using at most
\[
 C\alpha^{-2}\bigl(d\rho^{-2}+\log(1/\beta)\bigr)
\]
independent examples. Both guarantees hold with strict inequalities for the failure
and disagreement probabilities.
\end{theorem}
Appendix~\ref{app:repair} proves this theorem from Proposition~\ref{prop:base} and the
ordinary ERM guarantee in Lemma~\ref{lem:vc}. The additional samples depend only
logarithmically on $1/\eta_0$. The resulting $\log(1/\rho)$ term is absorbed by
$d/\rho^2$, leaving the confidence term separate.

\section{Why both sample costs are necessary}\label{sec:lower}
The two terms have different sources. To force the dimension-dependent cost, put uniform
mass on $d$ shattered points and give each label a small signed bias. Accuracy forces
the learner to predict most of the signs correctly. Replicability, required also at
intermediate biases, prevents the conditional prediction probability from crossing from
zero to one too quickly. A smooth potential converts this restriction into the factor
$\rho^{-2}$ without a logarithmic loss.

To force the confidence cost, put all mass on one shattered point. Two opposite biases
then reduce accurate learning to binary testing with error at most $\beta$.
These constructions apply to every class of VC dimension $d$, not just a specially
chosen example class.

\begin{theorem}[Matching lower bound]\label{thm:lower}
For every nonempty binary class on a finite or countable domain with VC dimension $d\ge1$,
and all $\alpha,\beta,\rho\in(0,1/10)$,
\[
 m_{\mathcal H}(\alpha,\beta,\rho)\ge
 \bar\kappa\alpha^{-2}\bigl(d\rho^{-2}+\log(1/\beta)\bigr),
 \quad
 \kappa=\frac{1369}{250000\pi^2},\quad
 \bar\kappa=\frac{3\kappa}{3+50\kappa}.
\]
The bound permits improper outputs and unrestricted public randomness.
\end{theorem}
Our September version~\cite{Zhang2026} established this lower bound. Its proof combines
multi-coordinate score estimates, a cubic potential, and a separate confidence test.
For a related coin-testing lower bound, see Impagliazzo, Lei, Pitassi, and
Sorrell~\cite{ILPS}, Lemma 7.2. Appendix~\ref{app:lower} gives the complete score
calculation, endpoint estimates, integration over the public seed, and confidence test.
Together with Theorem~\ref{thm:upper}, it proves Theorem~\ref{thm:main}.

\section{Complementary guarantees and earlier constructions}\label{sec:retained}
The first version also established proper-learning guarantees, faster accuracy dependence
under realizability, and explicit alternatives to the general construction.
This section presents those complementary results. All results here, including the
explicit algorithms and counterexamples, were developed in the first version
\cite{Zhang2026}. Appendix~\ref{app:retained} gives their full statements and proofs.
Its statement and equation numbers carry an S prefix, preserving the correspondence
with the September version without conflicting with the numbering of the new proof.

\subsection{Keep the output proper by filtering and ordered selection}\label{sec:proper-retained}
For a class of VC dimension one, first choose a countable proper core with the same
population optimum. Fix an order on that core in which the hypotheses disagreeing with
a reference hypothesis at any particular query form an interval. Train independent proper
ERMs and a separate proper reference. Use fresh data to estimate their risks. A shared
threshold decides which base hypotheses have sufficiently small empirical excess over
the reference. If none passes, return the reference. Otherwise, use one shared uniform
variable to select a quantile of the accepted multiset in the fixed order.

The output is a member of the class. The interval property makes a prediction depend on
only two cumulative masses, rather than on the entire identity of the selected model.
The filtering analysis controls those masses and the risk of the selected hypothesis.
The core and interval construction are supplied by \legacyref{6.1}{Lemma}; the complete
execution rule, block sizes, and analysis are in \legacyref{6.2}{Theorem}.

\begin{theorem}[Proper learning in VC dimension one]\label{thm:retained-vcone}
Every nonempty binary class of VC dimension one on a finite or countable domain admits
a proper agnostic $(\alpha,\beta)$-learner that is $\rho$-pointwise replicable and uses
at most
\[
 C\alpha^{-2}\bigl(\rho^{-2}+\log(1/\beta)\bigr)
\]
examples, for all $\alpha,\beta,\rho\in(0,1/10)$ and a universal constant $C$.
The rate is optimal even against improper learners.
\end{theorem}
The upper bound, including the proper fallback on the quantile map's null set, is proved
in \legacyref{6.2}{Theorem}. The lower bound is Theorem~\ref{thm:lower} with $d=1$.
A VC-two counterexample following \legacyref{6.2}{Theorem} identifies the structural
boundary of this interval representation.

\subsection{Learn a full finite-atom class directly}\label{sec:atoms-retained}
For the full binary class on a domain of $d$ points, count the observations and their
signed label sum at each point. Compare the sum divided by the square root of the count
with shared positive and negative thresholds. Output the corresponding sign when a
threshold is crossed, and output the default positive label otherwise. The zero-count
statistic is defined to be zero. Under the coding of Section~\ref{sec:model}, a positive
label means $1$ and a negative label means $0$.

\begin{theorem}[Proper learning on finitely many atoms]\label{thm:retained-atoms}
For $\mathcal H=\{0,1\}^{\mathcal X}$ with $|\mathcal X|=d\ge1$, the direct gate
algorithm is proper, agnostically $(\alpha,\beta)$-accurate, and $\rho$-pointwise
replicable using at most
\[
 C\alpha^{-2}\bigl(d\rho^{-2}+\log(1/\beta)\bigr)
\]
examples, for all $\alpha,\beta,\rho\in(0,1/10)$. This rate is optimal.
\end{theorem}
\legacyref{6.3}{Theorem} specifies the thresholds and proves accuracy and stability
using the self-normalized gate in \legacyref{3.2}{Lemma}. The lower bound again follows
from Theorem~\ref{thm:lower}. This separate proof supplies a simpler algorithm than the
general phase construction when every labeling is allowed.

\subsection{Use realizability to improve the accuracy dependence}\label{sec:realizable-retained}
The gate can also be combined with an ensemble of consistent hypotheses. At a query,
use the gate's answer when it detects a sufficiently strong sign. Otherwise, average
the ensemble's predictions and round with an independent shared uniform at that query.
The sample sizes are chosen before sampling. On samples for which no consistent
hypothesis exists, a fixed proper ERM rule still defines the algorithm. This matters
because pointwise replicability is required on noisy distributions too.

\begin{theorem}[Realizable guarantee with universal stability]\label{thm:retained-realizable}
Let $\mathcal H$ be a nonempty binary class of VC dimension $d\ge1$ on a finite or
countable domain. For $\alpha,\beta,\rho\in(0,1/10)$, there is a fixed-budget learner
that is $\rho$-pointwise replicable for every distribution and, whenever $\OPT=0$,
has risk at most $\alpha$ except with probability less than $\beta$. Its sample count is
\[
 O\!\left(
 \frac{d\log(1/\alpha)+\log(1/(\beta\rho))}{\rho^2\alpha}
 +\frac{\log(1/\beta)}{\alpha}
       \log\!\frac{\log(1/\beta)}{\alpha\beta}
 \right).
\]
The output may be improper.
\end{theorem}
\legacyref{6.4}{Corollary} proves this result, including the case where the population
infimum is zero but is not attained by a member of the class. The consistent-learning
bound controls the ensemble, and a centered Bernstein bound controls its final rounding.
The dimension-one example in \legacyref{6.5}{Remark} exhibits a class whose realizable
sample complexity is independent of the accuracy parameter.

\subsection{Retain the simpler soft-shift and gate routes}\label{sec:soft-retained}
SOFT-VC uses a different conversion. Train independent proper ERMs, average their
signed predictions, and estimate a signed label histogram on a fresh block. Shift each
average toward the estimated sign, clip it to the interval $[-1,1]$, and round using
independent coordinatewise shared uniforms. The clipped-shift inequality in
\legacyref{7.1}{Lemma} and its countable product form in \legacyref{7.2}{Corollary}
link the shift to the risk of the rounded output. The histogram estimate is certified by
\legacyref{7.3}{Lemma}. These are the inputs to the following guarantee.

\begin{theorem}[Explicit soft-shift guarantees]\label{thm:retained-soft}
Under the assumptions of Theorem~\ref{thm:main}, set $L=\log(4/\beta)$ and
$J=\log(512/\rho)$. The direct SOFT-VC learner is agnostically $(\alpha,\beta)$-accurate
and $\rho$-pointwise replicable with sample count at most
\[
 2^{37}\alpha^{-2}\left(\frac{d+L^2}{\rho^2}+L\right).
\]
A deterministic parameter-based choice between this learner and its rare-repair variant
has the same guarantees with sample count at most
\[
 2^{39}\alpha^{-2}\left(\frac{d+\min\{L^2,J^2\}}{\rho^2}+L\right).
\]
In particular, for each fixed $\beta_0\in(0,1/10)$ the direct learner attains the optimal
order $\Theta_{\beta_0}(d/(\rho^2\alpha^2))$.
\end{theorem}
The direct analysis is \legacyref{7.4}{Theorem}. Its fixed-confidence consequence is
\legacyref{7.5}{Theorem}, and repair and selection are proved in \legacyref{7.6}{Theorem}.
The gate-and-round alternatives, their three ensemble-training choices, and their explicit
sample budgets remain in Appendix~\ref{legacy:section:4}. The consolidated comparison
of all these earlier bounds is \legacyref{2.2}{Theorem}. Theorem~\ref{thm:upper} now
removes their parameter-dependent overhead for the general agnostic problem; their
individual budget calculations and absorption regions still describe those algorithms.

\paragraph{What the counterexamples establish.}
The distinctions between these algorithms are substantive. \legacyref{6.6}{Remark}
separates compulsory atomwise sign identification, independent rounding, and forced
gate outputs from the learning problem itself. \legacyref{7.7}{Remark} adds a SMALL-region
variance estimate, a witnessed forced-gate obstruction, a proper two-atom reduction,
and a rational VC-two example in which zero agnostic excess does not protect independent
rounding from risk fluctuations. Their hypotheses and complete proofs are retained.
These examples motivate the phase-based conversion in the new theorem.

\section{Conclusion}
The first version of this work, published on 9 September 2026~\cite{Zhang2026},
resolved the additional-term question and established the exact fixed-confidence rate.
The present version determines the optimal rate jointly in all three parameters.
Pointwise replicable agnostic learning on countable domains has a uniform optimal sample
law: $\alpha^{-2}[d\rho^{-2}+\log(1/\beta)]$. The essential step is to obtain a base
learner whose accuracy failure is already proportional to its disagreement budget.
Smooth signal weights, a common phase curve, and record-time selection achieve this
without a logarithmic sample penalty. The weak part determines a certified time; the
strong part remains accurate throughout the curve. The shared-seed comparison makes
these risk guarantees compatible with stable predictions, including at unseen points.
This closes the full-parameter statistical gap while separating the cost of replicability
from the cost of high confidence.

The earlier constructions retain distinct roles. Ordered quantile selection keeps the
output proper in VC dimension one. The finite-atom gate gives a direct proper algorithm,
and the realizable construction improves the accuracy dependence. SOFT-VC provides a
shorter route at fixed confidence. Their complete analyses, explicit constants, and
implementation-specific counterexamples are included here alongside the new uniform theorem.

\clearpage
\appendix
\section{Proof roadmap}\label{app:roadmap}
The proofs follow the construction. Appendix~\ref{app:statistics} establishes the statistical
inputs for the ensemble and histogram modules. Appendix~\ref{app:phase} proves the common-seed
phase theorem. Appendix~\ref{app:accuracy} certifies weak and strong risk. Finally,
Appendices~\ref{app:base}--\ref{app:final} assemble the base learner, apply rare repair,
and prove the matching lower bound. The arrows below indicate proof dependencies, not
the order in which the algorithm draws its samples.
Appendix~\ref{app:retained} supplies the separate proof map and complete analyses for
the complementary guarantees in Section~\ref{sec:retained}.

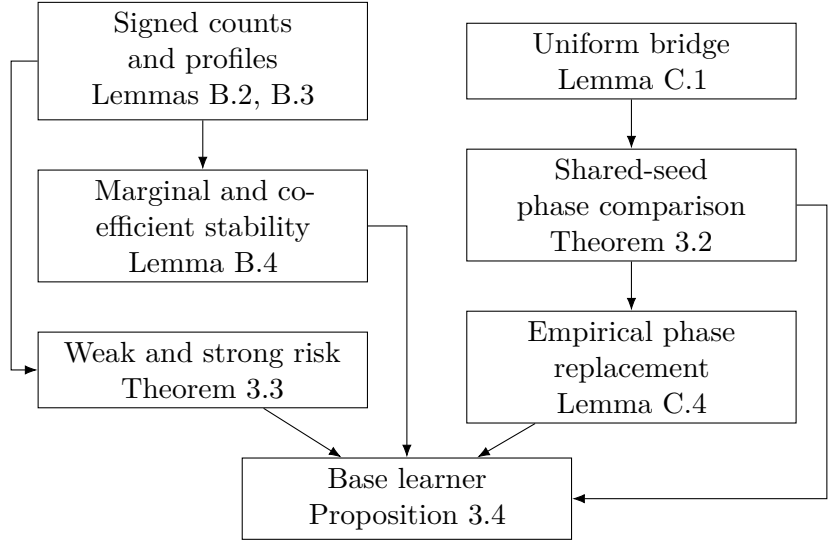
\begin{figure}[htbp]
\centering
\begin{tikzpicture}[>=Latex, every node/.style={draw=black,fill=white,text=black,
 align=center,text width=4.1cm,minimum height=1cm},node distance=0.65cm and 1.3cm]
\node (counts) {Signed counts and profiles\\Lemmas~\ref{lem:counts}, \ref{lem:profile}};
\node[right=of counts] (bridge) {Uniform bridge\\Lemma~\ref{lem:bridge}};
\node[below=of counts] (stability) {Marginal and coefficient stability\\Lemma~\ref{lem:qstability}};
\node[below=of bridge] (phase) {Shared-seed phase comparison\\Theorem~\ref{thm:phase}};
\node[below=of phase] (val) {Empirical phase replacement\\Lemma~\ref{lem:replacement}};
\node[below=of stability] (risk) {Weak and strong risk\\Theorem~\ref{thm:risk}};
\node[below=0.65cm of risk,xshift=2.7cm] (base) {Base learner\\Proposition~\ref{prop:base}};
\draw[->] (counts)--(stability);
\draw[->] (bridge)--(phase);
\draw[->] (phase)--(val);
\draw[->] (counts.west)--++(-0.35,0)|-(risk.west);
\draw[->] (stability.east)-|(base.north);
\draw[->] (phase.east)--++(0.4,0)|-(base.east);
\draw[->] (val)--(base);
\draw[->] (risk)--(base);
\end{tikzpicture}
\caption{The main upper-bound dependencies. The risk proof also uses the bridge and
uniform-validation estimates; the base learner also uses the ensemble guarantee.
These secondary arrows are omitted for legibility.}
\label{fig:upperdeps}
\end{figure}

\begin{figure}[htbp]
\centering
\begin{tikzpicture}[>=Latex,every node/.style={draw=black,fill=white,text=black,
 align=center,text width=4.1cm,minimum height=1cm},node distance=0.65cm and 1.3cm]
\node (base) {Base learner\\Proposition~\ref{prop:base}};
\node[right=of base] (lower) {Score and endpoint estimates\\Appendix~\ref{app:lower}};
\node[below=of base] (upper) {Rare repair and upper bound\\Theorem~\ref{thm:upper}};
\node[below=of lower] (lb) {Dimension and confidence bounds\\Theorem~\ref{thm:lower}};
\node[below=0.65cm of upper,xshift=2.7cm] (final) {Uniform optimal rate\\Theorem~\ref{thm:main}};
\draw[->] (base)--(upper);
\draw[->] (lower)--(lb);
\draw[->] (upper)--(final);
\draw[->] (lb)--(final);
\end{tikzpicture}
\caption{Completion of the rate theorem. Rare repair also uses an ordinary proper ERM
and independent validation. The confidence lower bound uses one shattered point.}
\label{fig:finaldeps}
\end{figure}
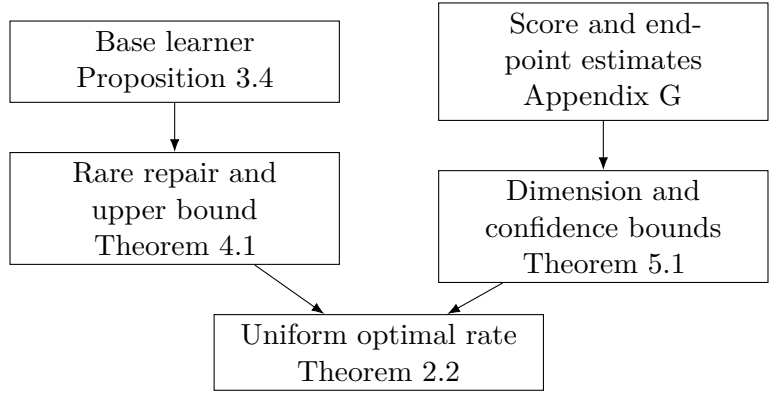
\clearpage

\section{Statistical inputs}\label{app:statistics}
We first justify the ensemble in Section~\ref{mod:ensemble}. We then prove the histogram
and profile bounds used in both the stability and accuracy analyses. All statements in
this appendix are distribution-free. We work with the predetermined measurable ERM
from Section~\ref{sec:model}.

\subsection{The log-free VC input and ensemble}
\begin{lemma}[ERM moments and tails]\label{lem:vc}
Let $\widehat h$ be a proper ERM on $n$ iid observations. Then
\begin{equation}\label{eq:vc}
 \E Z_n\le C\sqrt{(d+1)/n},\qquad
 \Pp\{Z_n>\E Z_n+t\}\le e^{-2nt^2},\qquad
 Z_n=\sup_{h\in\mathcal H}|R(h)-\widehat R_n(h)|.
\end{equation}
Consequently
\begin{equation}\label{eq:ermsecond}
 \E\bigl(R(\widehat h)-\OPT\bigr)^2\le C(d+1)/n.
\end{equation}
For $d\ge1$, a proper ERM with $C[d+\log(1/\eta)]/\delta^2$ observations has excess
at most $\delta$ except with probability $\eta$, for $\delta,\eta\in(0,1/4]$.
\end{lemma}
\begin{proof}
For completeness, we describe why the expectation in \eqref{eq:vc} has no $\log n$ factor.
The binary loss vectors $(\one\{h(X_j)\ne Y_j\})_{j=1}^n$ form a finite subset of
$\{0,1\}^n$ of VC dimension at most $d$: toggling fixed label coordinates preserves
shattering. Haussler's packing bound \cite{Haussler1995} implies that a maximal
$u$-separated subset in empirical $L_2$ distance has at most $(C/u)^{2d}$ elements,
for $d\ge1$ and $0<u\le1$.

Here is the associated chaining step. Conditional on the observations, let
$X_f=n^{-1}\sum_{j=1}^n\sigma_j f_j$, where the $\sigma_j$ are independent fair signs
and $f$ ranges over the loss vectors. The increment $X_f-X_g$ has exponential-moment
bound $\E_\sigma e^{t(X_f-X_g)}\le
\exp(t^2\|f-g\|_{L_2(P_n)}^2/(2n))$.
Choose maximal separated nets at radii $2^{-k}$ and approximate each $f$ at successive
radii. Each level has at most the product of the two adjacent net sizes possible
increments, of empirical $L_2$ length at most $3\cdot2^{-k}$.
The exponential-moment bound and the union bound give expected maximal absolute
increment at most $C2^{-k}\sqrt{d(k+1)/n}$. The sum over $k\ge0$ is
$C\sqrt{d/n}$, since $\sum_{k\ge0}2^{-k}\sqrt{k+1}<\infty$.
The nets eventually contain every distinct vector, so there is no terminal error.
A fixed base vector contributes at most $1/\sqrt n$ by its second moment.
Symmetrization with an independent ghost sample therefore proves the expectation bound
in \eqref{eq:vc}. When $d=0$, all hypotheses are identical and the single-loss second
moment proves the same bound with $d+1$.

Changing one observation changes $Z_n$ by at most $1/n$. The bounded-differences
inequality gives the stated tail. All suprema are measurable because the sample space
is countable. To compare the ERM with population risk, fix any $h\in\mathcal H$ and write
\[
 R(\widehat h)\le\widehat R_n(\widehat h)+Z_n
 \le\widehat R_n(h)+Z_n\le R(h)+2Z_n.
\]
Taking the infimum gives $0\le R(\widehat h)-\OPT\le2Z_n$.
Let $a=\E Z_n$. The inequality $Z_n^2\le2a^2+2(Z_n-a)_+^2$ and tail integration give
\[
 \E Z_n^2\le2a^2+4\int_0^\infty t e^{-2nt^2}\,dt
 =2a^2+1/n.
\]
This proves \eqref{eq:ermsecond}. Finally choose $n$ so that
$a\le\delta/4$ and $e^{-n\delta^2/8}\le\eta$; the excess is then at most $\delta$
outside an event of probability $\eta$.
\end{proof}

\begin{proof}[Proof of Proposition~\ref{prop:ensemble}]
This proposition certifies both outputs of Section~\ref{mod:ensemble}.
Put $\Delta_j=R(h_j)-\OPT$. Lemma~\ref{lem:vc} and the choice of $n$ imply
$\E\Delta_j^2\le C\eps^2$ and $\E\Delta_j\le C\eps$.
The $\Delta_j$ are independent. Affinity of risk gives
$R(p)-\OPT=T^{-1}\sum_j\Delta_j$, with variance at most $C\eps^2/T$.
Chebyshev's inequality at deviation $\eps$ proves the accuracy claim after enlarging $C$.
At any fixed $i$, $p_i$ and $p_i'$ are independent averages of $T$ identically distributed
bits. They have the same mean, and each has variance at most $1/(4T)$.
Thus $\E(p_i-p_i')^2\le1/(2T)$.
\end{proof}

\subsection{An ideal Poisson histogram with a fixed-pool implementation}
For analysis only, extend the histogram pool to an infinite iid sequence and use its
first $K$ observations for every value of $K$. In this ideal experiment, the counts of
the pairs $(i,1)$ and $(i,0)$ are mutually independent Poisson variables with respective
means $MD(i,1)$ and $MD(i,0)$. In particular the coordinate pairs $(N_i,S_i)$ are
independent within one run. This follows for every finite collection by the Poisson
generating function, and hence for the countable family by its finite-dimensional laws.
The finite-pool implementation agrees with it on $K\le2M$. Exponential Markov with
parameter $\log2$ gives
\begin{equation}\label{eq:overflow}
 \Pp\{K>2M\}\le e^{-(2\log2-1)M}.
\end{equation}
We do not condition the ideal histogram on $K\le2M$, since doing so would destroy
coordinate independence. All ideal bounds are transferred at the end by this coupling.

Recall $w_i,b_i,a_i,\theta_i$ from Section~\ref{mod:histogram}. They satisfy
$a_i\le w_i$, $\sum_iw_i=1$, and $\sum_i a_i\le1$.
Fix the Bayes bit $g_i^*=\one\{b_i>0\}$. For $b_i\ne0$, define
$\pi_i=\Pp\{b_i S_i\le0\}$, thus counting ties as errors, and set $\pi_i=1$ when
$b_i=0$. Terms with $a_i=0$ have zero risk cost.

\begin{lemma}[Normalized signed counts]\label{lem:counts}
For the ideal histogram and every $t>0$,
\begin{align}
 \E(\widehat\theta_i-\theta_i)^2&\le2/M,&
 \Pp\{\abs{\widehat\theta_i-\theta_i}>t\}&\le4e^{-Mt^2/8},\label{eq:theta}\\
 \E(\widehat b_i-b_i)^2&=w_i/M,&
 \E\sum_i(\widehat b_i-b_i)^2&=1/M.\label{eq:benergy}
\end{align}
Moreover $\pi_i\le e^{-M\theta_i^2/2}$. For every bounded nonnegative function
$H(|S_i|,N_i)$,
\begin{equation}\label{eq:pairing}
 b_i\E[H(|S_i|,N_i)\sgn(S_i)]\ge0,
 \qquad \sgn(0)=0.
\end{equation}
\end{lemma}
\begin{proof}
If $w_i=0$, all estimators and population coefficients at $i$ vanish almost surely.
Otherwise put $\mu=b_i/w_i$ and $\nu=Mw_i$. Conditional on $N_i$, the $N_i$ signs
are iid with mean $\mu$. With the first summand defined as zero at $N_i=0$,
\[
 \widehat\theta_i-\theta_i
 =\frac{S_i-\mu N_i}{\sqrt{MN_i}}
  +\frac\mu{\sqrt M}(\sqrt{N_i}-\sqrt\nu).
\]
The first summand has conditional mean zero and second moment at most $(1-\mu^2)/M$.
Its cross term with the second summand is zero by conditioning on $N_i$.
Also $(\sqrt{N_i}-\sqrt\nu)^2\le(N_i-\nu)^2/\nu$ and
$\E(N_i-\nu)^2=\nu$. The second moment of the sum is therefore at most $1/M$,
which implies the stated weaker bound.

The Poisson Chernoff rate is
$I(x;\nu)=x\log(x/\nu)-x+\nu$ and obeys
\[
 I(x;\nu)\ge(\sqrt x-\sqrt\nu)^2.
\]
Indeed, writing $x/\nu=y^2$, the difference is
$2\nu y(y\log y-y+1)\ge0$, with endpoint values given by continuity.
Applying the upper and lower Poisson tail bounds gives
$\Pp\{|\sqrt{N_i}-\sqrt\nu|\ge t\}\le2e^{-t^2}$.
Conditional Hoeffding bounds the first summand of the preceding decomposition at
absolute threshold $t/2$ by $2e^{-Mt^2/8}$. The second contributes at most
$2e^{-Mt^2/4}$, since $|\mu|\le1$. Their union proves \eqref{eq:theta}.
The signed Poisson count has mean $Mb_i$ and variance $Mw_i$. This proves the first
identity in \eqref{eq:benergy}; Tonelli proves the second.

For $b_i>0$, exponential Markov applied to $S_i$ at a negative argument gives
\[
 \Pp\{S_i\le0\}\le
 \exp\{-M(w_i-\sqrt{w_i^2-b_i^2})\}
 \le e^{-Mb_i^2/(2w_i)}.
\]
The first exponent follows by optimizing the moment generating function of the
difference of independent Poisson variables. At $b_i=w_i$ it follows by a limit,
or directly from $\Pp\{S_i=0\}=e^{-Mw_i}$. The case $b_i<0$ is symmetric;
$b_i=0$ gives the trivial bound one.

To prove the last assertion, suppose first $0\le\mu<1$. Conditional on $N_i$, for
each admissible positive value $j$,
\[
 \Pp(S_i=j\mid N_i)=
 \left(\frac{1+\mu}{1-\mu}\right)^j\Pp(S_i=-j\mid N_i).
\]
The paired contribution of $j$ and $-j$ to the weighted sign expectation is nonnegative.
The zero contribution vanishes. When $\mu=1$, all signs are nonnegative and the
claim is direct. Reverse signs for $\mu<0$. Multiplication by $b_i$ completes
\eqref{eq:pairing}, including the cases $b_i=0$ and $\mu=-1$.
\end{proof}

\subsection{Smooth profile transfer}
\begin{lemma}[Bounded-profile transfer]\label{lem:profile}
Let $A:\mathbb R\to\mathbb R$ be bounded and Lipschitz, and constant on each of the
two rays outside $[-K_0,K_0]$, for some fixed $K_0>0$. Then
\begin{equation}\label{eq:profiledet}
 \theta^2|A(z/\eps)-A(\theta/\eps)|^2\le C_A|z-\theta|^2.
\end{equation}
Consequently
\begin{equation}\label{eq:profilesum}
 \E\sum_i b_i^2|A(\widehat\theta_i/\eps)-A(\theta_i/\eps)|^2\le C_A/M.
\end{equation}
For the cost profile $\eta$ in \eqref{eq:qdef},
\begin{equation}\label{eq:costprofile}
 a_i^2|\eta(|\widehat b_i|/s)-\eta(a_i/s)|^2
 \le C|\widehat b_i-b_i|^2.
\end{equation}
\end{lemma}
\begin{proof}
If $|\theta|\le2K_0\eps$, Lipschitzness gives \eqref{eq:profiledet} with constant
$4K_0^2\operatorname{Lip}(A)^2$. Otherwise a nonzero profile difference requires $z$
to leave the constant ray containing $\theta$, so $|z-\theta|\ge|\theta|/2$.
Boundedness then gives the same assertion with constant $16\|A\|_\infty^2$.
Multiply by $w_i$ and use $b_i^2=w_i\theta_i^2$, followed by Lemma~\ref{lem:counts}
and $\sum_iw_i=1$, to prove \eqref{eq:profilesum}.
The same scalar argument for $t\mapsto\eta(|t|)$, at scale $s$, proves
\eqref{eq:costprofile}. Its plateau is reached at $|t|=1/2$.
\end{proof}
The five fixed profiles $f,\ell,k,\ell f$ and $z\mapsto k(z)\one\{z>0\}$ obey
this lemma. The last is Lipschitz because $k$ vanishes on a neighborhood of zero.
Their constants are universal.

\subsection{Coefficients and marginals for the phase theorem}
The following estimates connect the output of Section~\ref{mod:histogram} to
Theorem~\ref{thm:phase}. Put $B_i=b_i u_i$ and $u_i^*=\ell(\theta_i/\eps)$.
Since $u_i^*=0$ for $|\theta_i|\ge3\eps$,
\[
 \sum_i b_i^2(u_i^*)^2\le9\eps^2\sum_i w_i=9\eps^2.
\]
Comparison with this deterministic vector and Lemma~\ref{lem:profile} give
\begin{equation}\label{eq:weakenergy}
 \E\sum_iB_i^2\le C\eps^2+C/M,\qquad
 \E\sum_i(B_i-B_i')^2\le C/M.
\end{equation}
The second bound uses $(u_i-u_i')^2\le2(u_i-u_i^*)^2+2(u_i'-u_i^*)^2$.
It does not assume independent histograms in the two runs, which share $K$.

Let $q^*(p)$ denote \eqref{eq:qdef} with $\theta_i,b_i,g_i^*$ replacing the estimates
and empirical bit. Stars on individual profiles have the same meaning.
\begin{lemma}[Marginal stability]\label{lem:qstability}
For each $p$ independent of the ideal histogram,
\begin{align}
 \E\sum_i b_i^2|q_i-q_i^*(p)|^2&\le C/M,\label{eq:qweighted}\\
 \E|q_i-q_i^*(p)|^2&\le C/(\eps^2M)\quad\text{for every fixed }i.\label{eq:qlocal}
\end{align}
For two independent ensembles, with histogram pools independent of both ensembles
and with the prescribed shared $K$,
\begin{align}
 \E|q_i-q_i'|&\le C/\sqrt T+C/(\eps\sqrt M),\label{eq:qpoint}\\
 \E\sum_i B_i^2(q_i-q_i')^2&\le C\eps^2/T+C/M+C/(MT).\label{eq:qphaseenergy}
\end{align}
\end{lemma}
\begin{proof}
Fix $p$ first and suppress the coordinate subscript. Write $q=q^0+\eta(kg-kq^0)$.
The exact difference decomposition is
\begin{equation}\label{eq:qdecomposition}
 \begin{split}
 q-q^*={}&q^0-q^{0*}
 +\eta\bigl[(kg-k^*g^*)-(k-k^*)q^0-k^*(q^0-q^{0*})\bigr]\\
 &+(\eta-\eta^*)k^*(g^*-q^{0*}).
 \end{split}
\end{equation}
All multiplying factors have absolute value at most one.
The differences $q^0-q^{0*}=m(f-f^*)$, $k-k^*$, and $kg-k^*g^*$ are controlled
by the profiles listed above. Squaring the five terms of \eqref{eq:qdecomposition}
costs at most a factor five. Multiply by $b_i^2$, sum, and apply
\eqref{eq:profilesum} and \eqref{eq:costprofile} to obtain \eqref{eq:qweighted}.

For the unweighted estimate, ordinary Lipschitzness and \eqref{eq:theta} bound each
of the first four terms by $C/(M\eps^2)$ in second moment. The cost gate has a finer
bound than its global $1/s$ Lipschitz constant suggests. Its error is multiplied by
the deterministic $k^*$. If $k^*=0$ it contributes nothing. Otherwise
$a^2/w=\theta^2>\eps^2/4$, and \eqref{eq:costprofile} gives
\[
 (k^*)^2\E|\eta-\eta^*|^2
 \le\frac{C}{a^2}\E(\widehat b-b)^2
 =\frac{Cw}{Ma^2}\le\frac{C}{M\eps^2}.
\]
This proves \eqref{eq:qlocal} without a factor $L$. At $w=0$ the estimates vanish
and $q=q^*=p$, so the same statement holds directly.

For fixed histogram values, $p\mapsto p+\min(p,1-p)f$ is $2$-Lipschitz on $[0,1]$.
Interpolation with the fixed bit $g$ does not increase this constant.
Introduce the intermediate predictor obtained from $p'$ and the first histogram.
The ensemble-change part has mean square at most $4\E(p_i-p_i')^2\le2/T$.
For the histogram-change part, compare both histograms to $q^*(p')$ and apply
\eqref{eq:qlocal}; independence of the ensemble pair from the histogram pair
justifies conditioning on $p'$. The triangle and Cauchy--Schwarz inequalities prove
\eqref{eq:qpoint}.

For the weighted estimate, the ensemble-change part is at most
$C T^{-1}\E\sum_i B_i^2$ because $B$ is independent of the ensemble pair.
In the histogram-change part use $B_i^2\le b_i^2$ and \eqref{eq:qweighted} on each
comparison to $q^*(p')$. Their sum is at most $C/M$. Substitution of
\eqref{eq:weakenergy} proves \eqref{eq:qphaseenergy}.
\end{proof}

\section{The shared-seed phase sampler}\label{app:phase}
This appendix proves Theorem~\ref{thm:phase}, the stability certificate for
Section~\ref{mod:phase}. We first control changes of the integrated curve, then changes
of its normalized record CDF, and finally the query label. The last step must retain
the dependence between the query's phase and the selected time.

\begin{lemma}[Weighted uniform bridge and moving arcs]\label{lem:bridge}
Let $V_i$ be independent uniforms on $[0,1]$, and let $(d_i)$ be a finite or absolutely
summable family of real weights. For
$H(t)=\sum_i d_i(\one\{V_i\le t\}-t)$,
\begin{equation}\label{eq:bridge}
 \E\sup_{0\le t\le1}|H(t)|^2\le8\sum_i d_i^2.
\end{equation}
For a common fixed arc length $a\in[0,1]$, the centered process
$A(t)=\sum_i d_i(\one\{(V_i+t)\bmod1<a\}-a)$ satisfies
\begin{equation}\label{eq:arcs}
 \E\sup_{0\le t\le1}|A(t)|^2\le32\sum_i d_i^2.
\end{equation}
Fixed shifts of individual phases leave the second conclusion unchanged.
\end{lemma}
\begin{proof}
Begin with finitely many terms. In the filtration revealing arrivals of the uniforms,
$(\one\{V_i\le t\}-t)/(1-t)$ is a martingale for $t<1$.
Indeed, an arrived uniform stays arrived, while a uniform not yet seen by time $s$
has conditional arrival probability $(t-s)/(1-s)$ by time $t$.
The weighted sum $H(t)/(1-t)$ is therefore a martingale. Doob's $L_2$ inequality yields
\[
 \E\sup_{0\le t\le1/2}|H(t)|^2
 \le4\E\left|\frac{H(1/2)}{1/2}\right|^2=4\sum_i d_i^2.
\]
Reflecting the uniforms about $1/2$ gives the same bound for the other half of the
interval. The maximum of the two squared suprema is at most their sum, proving
\eqref{eq:bridge}. Endpoint conventions do not change the supremum bound, by taking
one-sided limits if necessary.

A centered count in a circular interval is a difference of two values of $H$.
If the interval wraps, use $H(0)=H(1)=0$ and its complementary interval. Thus
$\sup_t|A(t)|\le2\sup_t|H(t)|$, proving \eqref{eq:arcs}.
Shifting each independent uniform by a fixed amount modulo one preserves their joint law.
For countably many weights, absolute summability gives uniform convergence of the
processes and of all finite truncations. Passing to the limit proves both bounds.
The estimate in \eqref{eq:arcs} is an expected \emph{squared} supremum, as required
later when arc-length groups are combined.
\end{proof}

For a fixed input $(B,q)$ define
\[
 I_i(q;t)=\int_0^t\bigl(\one\{(U_i+z)\bmod1<q_i\}-q_i\bigr)\,dz,
 \quad W_{B,q}(t)=s_0t+\sum_iB_iI_i(q;t).
\]
Let $M_{B,q}(t)=\max_{0\le z\le t}W_{B,q}(z)$ and
$G_{B,q}=M_{B,q}/M_{B,q}(1)$, as in the main text.
Absolute summability of $B$ makes $W$ Lipschitz. Its running maximum is Lipschitz
as well. Since $W(0)=0$ and $W(1)=s_0$, $G$ is an absolutely continuous CDF on
$[0,1]$, with $G(0)=0$ and $G(1)=1$.

\begin{lemma}[Primitive and CDF perturbation]\label{lem:primitive}
For fixed inputs $(B,q),(B',q')$, put
\[
 E_0=\sum_i\bigl[B_i^2(q_i-q_i')^2+(B_i-B_i')^2\bigr].
\]
Then
\begin{equation}\label{eq:primitiveenergy}
 \E_U\norm{W_{B,q}-W_{B',q'}}_\infty^2\le32E_0,
\end{equation}
and, deterministically,
\begin{equation}\label{eq:cdfcomparison}
 \norm{G_{B,q}-G_{B',q'}}_\infty
 \le\frac2{s_0}\norm{W_{B,q}-W_{B',q'}}_\infty.
\end{equation}
\end{lemma}
\begin{proof}
Write $\delta q_i=q_i-q_i'$ and $\delta B_i=B_i-B_i'$.
Integrating with respect to the endpoint of the indicator arc gives
\begin{equation}\label{eq:integralidentity}
 I_i(q;t)-I_i(q';t)
 =\delta q_i\int_0^1
 \left[\one\{(q_i'+v\delta q_i-U_i)\bmod1\le t\}-t\right]dv.
\end{equation}
One can check this by differentiating $I_i$ in $q_i$: the derivative is the indicator
that the unique crossing time lies in $[0,t]$, minus $t$. Integration between $q_i'$
and $q_i$ proves the identity. Values at a crossing endpoint do not affect the integral.
For each fixed $v$, the crossing times in \eqref{eq:integralidentity} are independent
uniforms. Jensen's inequality for the integral in $v$ and Lemma~\ref{lem:bridge} imply
\[
 \E_U\norm{\sum_i B_i[I_i(q)-I_i(q')]}_\infty^2
 \le8\sum_iB_i^2\delta q_i^2
 \le16\sum_iB_i^2\delta q_i^2.
\]
Applying the same calculation between marginal $q_i'$ and zero gives
$\E_U\|\sum_i\delta B_i I_i(q')\|_\infty^2\le16\sum_i\delta B_i^2$.
The inequality $(a+b)^2\le2a^2+2b^2$ proves \eqref{eq:primitiveenergy}.

Let $\Delta=\|W-W'\|_\infty$. Taking maxima over the same initial interval gives
$\|M-M'\|_\infty\le\Delta$. Write $A=M(1)$ and $A'=M'(1)$; both are at least $s_0$.
For every $t$,
\[
 \left|\frac{M(t)}A-\frac{M'(t)}{A'}\right|
 \le\frac{|M(t)-M'(t)|}A
   +\frac{M'(t)}{A'}\frac{|A-A'|}A
 \le\frac{2\Delta}{s_0}.
\]
This proves \eqref{eq:cdfcomparison}. The bounds pass to absolutely summable series
by uniform convergence.
\end{proof}

\begin{lemma}[Rotation bound for the selected query]\label{lem:rotation}
For a fixed input $(B,q)$ and fixed $x$, the selected position
$(U_x+G_{B,q}^{-1}(Z))\bmod1$ belongs to any fixed circular arc of length
$\delta\in[0,1]$ with probability at most
\begin{equation}\label{eq:rotationbound}
 \delta+2|B_x|\delta/s_0.
\end{equation}
\end{lemma}
\begin{proof}
Condition on every phase except $U_x=u$. Write $G_u$ for the resulting CDF.
As a function of $u$, the primitive satisfies almost everywhere
\[
 \partial_u I_x(q;t)=F_t(x)-F_0(x),
\]
whose absolute value is at most one. Consequently a shift of circular length $\delta$
changes $W$ in supremum norm by at most $|B_x|\delta$. The deterministic part of
Lemma~\ref{lem:primitive} gives
\begin{equation}\label{eq:rotationlip}
 \|G_{u+\delta}-G_u\|_\infty\le2|B_x|\delta/s_0.
\end{equation}
Extend each CDF to a continuous lift by
$\widetilde G_u(t+k)=G_u(t)+k$ for $0\le t\le1$, $k\in\mathbb Z$.
The extension is consistent at the endpoints and periodic in $u$.
For an arc starting at a real $v$, the conditional selected-time law has CDF $G_u$.
After integrating $u$, its probability of lying in that arc is
\[
 \int_0^1\left[\widetilde G_u(v+\delta-u)-\widetilde G_u(v-u)\right]du.
\]
Set $H(u)=\widetilde G_u(v-u)$, so $H(u+1)=H(u)-1$.
Insert $\widetilde G_{u-\delta}(v+\delta-u)=H(u-\delta)$ in the integrand.
Since $H(u)+u$ is periodic,
$\int_0^1[H(u-\delta)-H(u)]du=\delta$.
The remaining integral is at most $2|B_x|\delta/s_0$ by \eqref{eq:rotationlip}.
This proves \eqref{eq:rotationbound} without imposing independence between $U_x$
and the selected time. Absolute continuity of $G_u$ removes all endpoint ambiguities.
\end{proof}

\begin{proof}[Proof of Theorem~\ref{thm:phase}]
The entire input pair is independent of $((U_i),Z)$, so condition on that pair first.
For a fixed query arc and fixed phases, changing its time CDF from $G$ to $G'$, with
the same inverse-CDF uniform $Z$, changes the label with probability at most
$2\|G-G'\|_\infty$. Indeed, an unwrapped arc corresponds to one interval of $Z$
values whose endpoints each move by at most this norm. For a wrapped arc take its
complement. Endpoint labels have conditional probability zero.
By Lemma~\ref{lem:primitive}, the expected fixed-arc cost is thus at most
$16\sqrt2\sqrt{E_0}/s_0$.

Insert the intermediate label $F^q_{(G')^{-1}(Z)}(x)$. After the fixed-arc change,
the remaining discrepancy is between the arcs of lengths $q_x$ and $q_x'$ at the
same selected time from the second input. Their symmetric difference is one circular
arc of length $|q_x-q_x'|$. Lemma~\ref{lem:rotation}, for that second input, bounds
this cost by $|q_x-q_x'|+2|B_x'||q_x-q_x'|/s_0$.
Finally,
\[
 |B_x'||q_x-q_x'|
 \le |B_x||q_x-q_x'|+|B_x-B_x'|
 \le\sqrt{2E_0},
\]
where $|q_x-q_x'|\le1$ was used in the first step.
The sum is the conditional bound in \eqref{eq:phasemain}. Average over the inputs
and apply Jensen's inequality to $\sqrt{E_0}$. If its expectation is infinite the
stated bound is vacuous; otherwise the calculation is justified by Tonelli and the
absolute summability assumptions.
\end{proof}

\begin{lemma}[Replacing population risk by validation]\label{lem:replacement}
For one run of the ideal base learner, replacing the true weak-risk phase sampler by
the empirical sampler changes its prediction at any fixed query with probability at
most $16\sqrt2/(s_0\sqrt v)$, using the same phases and $Z$.
\end{lemma}
\begin{proof}
Define the signed validation coefficient
\[
 \widetilde b_i=\frac1v\sum_{j=1}^v\one\{X_j=i\}(2Y_j-1),\qquad
 \widehat B_i=u_i\widetilde b_i,\qquad B_i=u_i b_i.
\]
Expansion of the affine weighted risk gives $\widehat W=W_{\widehat B,q}$ exactly.
Conditional on the ensemble and histogram, the validation sample is independent, and
\[
 \E\left[\sum_i(\widehat B_i-B_i)^2\,\middle|\,p,(N_i,S_i)_i\right]
 =\frac1v\sum_i u_i^2(w_i-b_i^2)\le1/v.
\]
Correlations between validation coordinates do not enter this sum of individual
second moments. The whole pair $(B,q,\widehat B,q)$ is independent of the phases
and $Z$. Since the query arc does not change, the fixed-arc part of the preceding
proof and Lemma~\ref{lem:primitive} give the claimed bound.
The primitives themselves depend on the phases; their input coefficients do not.
\end{proof}

\section{The two risk certificates}\label{app:accuracy}
We now prove Theorem~\ref{thm:risk}, which certifies Section~\ref{mod:risk}.
The weak part uses the selected record time. The strong part is controlled uniformly
over all times. Both are compared with the same realized ensemble $p$.

\subsection{Validation and the weak certificate}
\begin{lemma}[Uniform validation along the phase curve]\label{lem:validation}
Conditional on fixed $u,q,(U_i)$ independent of the validation block of size $v$,
\[
 \E\left[\sup_{0\le t\le1}|\widehat R_u(F_t)-R_u(F_t)|
       +|\widehat R_u(q)-R_u(q)|\right]\le C/\sqrt v.
\]
\end{lemma}
\begin{proof}
The fractional reference is one $[0,1]$-valued loss, whose empirical average has
expected absolute error at most $1/(2\sqrt v)$.
For the curve, the loss on an observation $(X_j,Y_j)$, as a function of $t$, is a
constant plus at most two signed step functions. Each coefficient has absolute value
at most one. Symmetrization bounds the expected uniform empirical error by twice
the expected supremum of its Rademacher average.

Condition on the observations and sort the endpoints of one family of steps.
Its Rademacher sum, as $t$ passes those endpoints, is a partial-sum martingale with
terminal second moment at most $v$. Doob's inequality and Cauchy--Schwarz bound
the expected maximum absolute partial sum by $2\sqrt v$. Tied endpoints merely
remove some partial sums from consideration. Apply this bound to each of the two
step families. The constant family contributes at most $\sqrt v$ by its second
moment. Divide their sum by $v$ and include the symmetrization factor.
All bounds hold uniformly for the fixed curve, proving the assertion.
\end{proof}

\begin{lemma}[Record certificate]\label{lem:record}
The sampler in \eqref{eq:algorithmprimitive} satisfies almost surely
$\widehat R_u(F_\tau)\le\widehat R_u(q)+s_0$.
Except with probability $Cr$, it also satisfies
\begin{equation}\label{eq:weakcertificate}
 R_u(F_\tau)\le R_u(q)+3\eps.
\end{equation}
\end{lemma}
\begin{proof}
The empirical primitive is absolutely continuous. Its running maximum is also
absolutely continuous. At almost every time where the running maximum has positive
derivative, the primitive is attaining a new maximum and its own derivative is positive.
This fact can also be checked on each affine segment of the empirical primitive.
The inverse-CDF time has density $\widehat M'(t)/\widehat M(1)$; hence it assigns
probability zero to all other times. At a selected time,
$\widehat W'(t)=s_0+\widehat R_u(q)-\widehat R_u(F_t)>0$ almost surely,
which proves the empirical certificate.

Lemma~\ref{lem:validation} and Markov imply that the sum of the uniform binary
validation error and the fractional reference error is at most $2\eps$ except
with probability $C/(\eps\sqrt v)\le Cr$.
On this event the empirical certificate implies
$R_u(F_\tau)\le R_u(q)+s_0+2\eps=R_u(q)+3\eps$.
\end{proof}

\begin{lemma}[Accuracy of the weak fractional reference]\label{lem:weakrisk}
In the ideal Poisson experiment,
$R_u(q)\le R_u(p)+2\eps$ except with probability $Cr$.
\end{lemma}
\begin{proof}
Condition on the ensemble $p$, independently of the histogram. The first shift gives
\[
 R_u(q^0)-R_u(p)
 =-\sum_i b_i m_i\ell(\widehat\theta_i/\eps)f(\widehat\theta_i/\eps).
\]
The summands are independent over $i$. Their means are nonpositive by the paired-sign
identity \eqref{eq:pairing}, because $\ell$ is nonnegative and even and $f$ has the
sign of its argument. Subtracting the deterministic true-profile value in each
variance bounds the sum of variances by $C/M$, using Lemma~\ref{lem:profile} for
$\ell f$. Chebyshev therefore bounds a positive deviation above $\eps$ by
$C/(M\eps^2)\le Cr^2$. Absolute summability of the risk coefficients justifies
the countable sum and passage from finite sums.

The final interpolation moves $q^0$ toward $g_i$. Such a move can increase true
risk only when the empirical sign is wrong and $\eta_i>0$. The latter condition
requires $|\widehat b_i|>s/4$. On this event put $e_i=\widehat b_i-b_i$; then
\[
 e_i^2=(a_i+|\widehat b_i|)^2\ge4a_i|\widehat b_i|\ge a_i s.
\]
The largest possible risk increase at $i$ is $a_i$, and $u_i\le1$. Thus, for each
realization,
\begin{equation}\label{eq:weakgate}
 R_u(q)-R_u(q^0)\le\sum_i e_i^2/s.
\end{equation}
Its probability of exceeding $\eps$ is at most $1/(Ms\eps)$ by
\eqref{eq:benergy}. This is at most $r^2L/64\le Cr$.
To see uniformity, $T\le17/(16r^2)$ for $r\le1/4$, so
$rL\le r\log((17/2)r^{-3})$, a bounded function on that interval.
Combining the two events proves the claim.
\end{proof}

\subsection{Uniform strong-risk control}
The next lemma is the second part of Theorem~\ref{thm:risk}.
Its random weights must be kept coupled to their empirical signs. We therefore
separate coordinates according to the \emph{expected} weight and sign-error probability
only in the analysis, not by a data-dependent event on which independence would fail.

\begin{lemma}[Strong risk along the entire curve]\label{lem:strong}
In the ideal Poisson experiment,
\[
 \sup_{0\le t\le1}\{R_c(F_t)-R_c(p)\}\le C\eps
\]
except with probability $Cr$.
\end{lemma}
\begin{proof}
Condition on $p$. On the support of $c$, $|\widehat\theta_i|>2\eps$.
Thus $S_i\ne0$, $k_i=1$, and $f(\widehat\theta_i/\eps)$ equals the empirical sign.
The intermediate marginal is
\[
 Q_i=\begin{cases}\min(2p_i,1),&g_i=1,\\\max(2p_i-1,0),&g_i=0.\end{cases}
\]
Define $Q_i$ by this formula everywhere; outside the support of $c$ it has no effect.
For now omit the cost interpolation and restrict to the population-small-cost set
$\{i:a_i\le s\}$.

\paragraph{Step 1: quantify the gain from a correct direction.}
Let $z_i=p_i$ if $g_i^*=1$ and $z_i=1-p_i$ otherwise. It is the ensemble's
Bayes-correct probability. Let $C_i(t)$ be the indicator that the binary curve with
marginal $Q_i$ predicts $g_i^*$, and put
\[
 Z_i(t)=a_i c_i[z_i-C_i(t)],\qquad \bar c_i=\E c_i.
\]
At a fixed time, averaging over the phase makes $\E_U C_i(t)$ equal to
$z_i+m_i$ for a correct empirical direction and $z_i-m_i$ for a wrong direction.
If $m_i=0$, $Q_i=p_i$ and $Z_i(t)=0$ for every time.
Use the tie-inclusive error probability $\pi_i$ defined in Appendix~\ref{app:statistics}
to form the deterministic sets
\[
 \mathcal G=\{i:\bar c_i\ge16\pi_i\},\qquad
 \mathcal B=\{i:\bar c_i<16\pi_i\}.
\]
For $i\in\mathcal G$,
\begin{align}
 \E Z_i(t)
 &=-a_i m_i\bigl(\bar c_i-2\E[c_i\one\{\text{wrong direction}\}]\bigr)
 \le-\tfrac78a_i m_i\bar c_i,\label{eq:goodmean}\\
 \E Z_i(t)^2&\le2a_i^2m_i\bar c_i.\label{eq:goodsecond}
\end{align}
For the second inequality, a Bernoulli variable of mean $z_i\pm m_i$ obeys
$\E_U(z_i-C_i)^2\le2m_i$. This follows by substitution separately for
$z_i\le1/2$ and $z_i\ge1/2$, where $m_i=\min(z_i,1-z_i)$.
Multiply by $a_i^2c_i^2$ and use $c_i^2\le c_i$.

\paragraph{Step 2: use this gain to control all times on good coordinates.}
Put $G=\sum_{i\in\mathcal G,a_i\le s}a_i m_i\bar c_i$.
At any fixed time the variance of the risk sum is at most $2sG$ and its mean
is at most $-7G/8$. Nonzero $m_i$ is at least $1/T$, so
\begin{equation}\label{eq:goodmass}
 \sum_{i\in\mathcal G,a_i\le s,m_i>0}a_i\bar c_i\le TG.
\end{equation}
Take $J=\lceil256T\rceil$ equal mesh cells. For cell $j$, let $C_j^{\rm cross}$
be the sum of $a_i c_i$ over these coordinates whose circular prediction arc has
an endpoint in the cell. Conditional on the histogram, each such probability is
at most $2/J$. Histogram and phase pairs remain independent across coordinates
when this conditioning is removed. Hence
\[
 \E C_j^{\rm cross}\le2TG/J\le G/128,
 \qquad \Var C_j^{\rm cross}\le s\E C_j^{\rm cross}.
\]
For independent centered variables bounded in absolute value by $b$ with total
variance $V$, Bernstein gives upper-tail threshold
$\sqrt{2Vz}+2bz/3$ with failure at most $e^{-z}$.
Here $|Z_i-\E Z_i|\le2s$. Applying that bound at a mesh point gives
\[
 \sum_{\mathcal G,a_i\le s}Z_i(t_j)
 \le-\tfrac78G+2\sqrt{sGz}+\tfrac43sz
 \le-\tfrac34G+Csz.
\]
The last step uses $2\sqrt{sGz}\le G/8+8sz$.
For $C_j^{\rm cross}$ use bound $s$ and its displayed variance.
The inequality $\sqrt{sGz}/8\le G/128+sz/2$ yields
$C_j^{\rm cross}\le G/64+Csz$, again outside an event of probability $e^{-z}$.
The risk changes by at most $C_j^{\rm cross}$ inside the cell.
Union bound over the mesh points and cells, with $z=\log(4J/\delta)$,
therefore gives, for $0<\delta<1$,
\begin{equation}\label{eq:gooduniform}
 \Pp\left\{\sup_t\sum_{\mathcal G,a_i\le s}Z_i(t)
       >Cs\log(CT/\delta)\right\}\le\delta.
\end{equation}
For countably many terms apply Bernstein to finite partial sums and pass to the
limit. Total absolute weight is at most $\sum_i a_i\le1$.

\paragraph{Step 3: bound bad coordinates by square energy.}
If $|\theta_i|\le\eps$, a positive $c_i$ requires normalized estimation error
greater than $\eps$. Lemma~\ref{lem:counts} implies
$\bar c_i\le4e^{-M\eps^2/8}$, and $a_i^2\le w_i\eps^2$.
If $|\theta_i|>\eps$ and $i\in\mathcal B$, then
$\bar c_i<16e^{-M\theta_i^2/2}$.
Since $M\eps^2\ge2$, the function $x e^{-Mx/2}$ decreases for $x\ge\eps^2$.
Using $a_i^2=w_i\theta_i^2$ and summing the two cases gives
\begin{equation}\label{eq:badenergy}
 E_{\mathcal B}:=\sum_{i\in\mathcal B}a_i^2\bar c_i
 \le20\eps^2 e^{-M\eps^2/16}.
\end{equation}
This controls square energy, not total bad-coordinate mass.

\paragraph{Step 4: control histogram drift and moving arcs on the bad set.}
Write
\[
 \sum_{\mathcal B,a_i\le s}Z_i(t)=D_0+W_0(t),\qquad \E_U W_0(t)=0,
\]
where
$D_0=-\sum_{\mathcal B,a_i\le s}a_i m_i c_i\sgn(b_i)\sgn(S_i)$.
The formula is valid also at ties because $c_i=0$ there, and at $b_i=0$ because
$a_i=0$. Pairing the signs with the nonnegative even weight $c_i$ shows
$\E D_0\le0$. Histogram independence gives
\begin{equation}\label{eq:baddrift}
 \E(D_0)_+^2\le\Var D_0\le E_{\mathcal B}/4.
\end{equation}
Indeed, if $D_0>0$ and $\E D_0\le0$, then $D_0\le|D_0-\E D_0|$.

Condition now on the histogram. Because $p$ is on the $T$-grid, the Bayes-correct
arc lengths take at most $2(T+1)$ values. For Bayes bit one, the correct arc has
length $Q_i$. For Bayes bit zero, set $V_i=(U_i-Q_i)\bmod1$; then its correct
indicator is exactly $\one\{(V_i+t)\bmod1<1-Q_i\}$.
These shifted phases remain independent uniforms after conditioning.
Group by correct arc length. In each group Lemma~\ref{lem:bridge} bounds the
expected squared supremum of its centered contribution by $32\sum_i a_i^2c_i^2$.
The inequality $(\sum_{j=1}^J z_j)^2\le J\sum_j z_j^2$, followed by averaging
over the histogram and $c_i^2\le c_i$, gives
\begin{equation}\label{eq:badphase}
 \E\sup_t|W_0(t)|^2\le CT E_{\mathcal B}.
\end{equation}
Markov applied to $[(D_0)_++\sup_t|W_0(t)|]^2$ and
\eqref{eq:badenergy}--\eqref{eq:badphase} yields
\begin{equation}\label{eq:baduniform}
 \Pp\left\{\sup_t\sum_{\mathcal B,a_i\le s}Z_i(t)>\eps\right\}
 \le CT e^{-M\eps^2/16}\le Cr.
\end{equation}
For the last inequality use $M\eps^2\ge64/r^2$ and
$T\le17/(16r^2)$; the function $r^{-3}e^{-4/r^2}$ is bounded on $(0,1/4]$.
In particular this argument does not compare an exponential in $1/r$ to $\eps$.

\paragraph{Step 5: restore the empirical cost gate.}
On the support of $c$, the actual marginal is $q_i=(1-\eta_i)Q_i+\eta_i g_i$.
Define an analysis-only comparison marginal
\[
 q_i^\dagger=\begin{cases}
 g_i^*,&a_i>s,\\
 (1-\eta_i)Q_i+\eta_i g_i^*,&a_i\le s,
 \end{cases}
\]
and use the same $U_i$ to define its binary curve $F_t^\dagger$.
On a small-cost coordinate, moving the marginal toward the true bit improves the
binary risk at every phase time: increasing a threshold toward one can only change
zero to one, and decreasing it toward zero can only change one to zero.
On a large-cost coordinate $q_i^\dagger$ is the Bayes bit, so its risk relative
to $p_i$ is nonpositive. Thus the ideal curve is bounded by the pure small-cost
risk sum controlled in Steps 2 and 4.

Every harmful discrepancy from this ideal curve is covered by two events.
If the empirical sign is wrong and $\eta_i>0$, then $a_i\le e_i^2/s$, as in
\eqref{eq:weakgate}. If $a_i>s$ and $\eta_i<1$, then
$|\widehat b_i|<s/2$ and $|e_i|\ge a_i/2$, so $a_i\le4e_i^2/s$.
If neither occurs, the actual and ideal curves agree whenever they can differ
harmfully. Charge the full possible binary risk cost $a_i$ for each event. This gives
the realization-wise bound
\begin{equation}\label{eq:stronggate}
 R_c(F_t)-R_c(F_t^\dagger)\le5\sum_i e_i^2/s\qquad\text{for every }t.
\end{equation}
Its probability of exceeding $\eps$ is at most $5/(Ms\eps)\le Cr$.
Taking $\delta=r$ in \eqref{eq:gooduniform}, and using
$s\log(CT/r)\le C\eps$, proves the lemma.

Finally, the suprema used above are measurable. Each coordinate curve is right
continuous on $[0,1)$, with the value at one equal to its value at zero. Each risk
series converges uniformly, dominated by $\sum_iw_i=1$. Its supremum therefore
equals the supremum over rational times in $[0,1)$, a countable set.
\end{proof}

\begin{proof}[Proof of Theorem~\ref{thm:risk}]
Intersect the events of Lemmas~\ref{lem:record}, \ref{lem:weakrisk}, and
\ref{lem:strong}. Their failure probabilities add to at most $Cr$, without any
independence assumption between the events. The first two give
$R_u(F_\tau)\le R_u(p)+5\eps$, and the third gives the strong certificate
simultaneously for every time, hence also for this $\tau$.
The empirical statement is exactly Lemma~\ref{lem:record}.
\end{proof}

\section{Assembling the base learner}\label{app:base}
\begin{proof}[Proof of Proposition~\ref{prop:base}]
This proof certifies the fixed-budget algorithm of Section~\ref{sec:construction}
under its prescribed common-seed execution.

\paragraph{Accuracy in the ideal experiment.}
Proposition~\ref{prop:ensemble} gives $R(p)\le\OPT+C\eps$ except with probability
$C/T\le Cr^2$. Intersect this event with Theorem~\ref{thm:risk}.
The realized weights satisfy $u_i+c_i=1$ at every coordinate. By the definition
of weighted risk,
\[
 \begin{split}
 R(F_\tau)&=R_u(F_\tau)+R_c(F_\tau)\\
 &\le R_u(p)+5\eps+R_c(p)+C\eps
 =R(p)+C\eps\le\OPT+C\eps.
 \end{split}
\]
The strong certificate is uniform in time, so its use at the weak-risk-selected
time is legitimate. The union of the exceptional events has probability at most $Cr$.

\paragraph{Replicability in the ideal experiment.}
For the two runs put $B_i=b_i u_i$ and $B_i'=b_i u_i'$.
Their joint input $(B,q,B',q')$ is a function of both independent data samples
and the shared count $K$. It is jointly independent of $((U_i),Z)$ because the
public seed components are a product and are independent of all data.
The histograms may depend on each other through $K$; their comparisons to population
profiles in \eqref{eq:weakenergy} and Lemma~\ref{lem:qstability} allow exactly this.
Theorem~\ref{thm:phase} and those estimates give, at each fixed $x$, the true-risk
sampler disagreement bound
\[
 C\left[\frac1{\sqrt T}+\frac1{\eps\sqrt M}
 +\frac1{s_0}\sqrt{\frac{\eps^2}{T}+\frac1M+\frac1{MT}}\right]\le Cr.
\]
For example $T^{-1/2}\le r$, $M^{-1/2}/\eps\le r/8$, and $s_0=\eps$.
Lemma~\ref{lem:replacement} permits replacing each true-risk sampler by its actual
empirical sampler using the same phases and $Z$. A union bound over the two
replacements adds at most $32\sqrt2/(s_0\sqrt v)\le Cr$.
This is the prescribed common-seed execution of the algorithm.

\paragraph{Fixed budget and overflow.}
The reserved budget is $Tn+2M+v$, a deterministic integer fixed by $d,\eps,r$.
The ceilings in \eqref{eq:parameters} imply it is at most
$C_m(d+1)/(\eps^2r^2)$ for a universal constant $C_m$.
Equation~\eqref{eq:overflow} transfers ideal accuracy with at most one overflow
probability. For the two-run comparison, the same $K$ causes simultaneous overflow,
on which both implemented learners return $h_0$. Even the looser union-bound cost
$2e^{-(2\log2-1)M}$ is absorbed into $Cr$. There is no conditioning on a truncated
Poisson count in the preceding independence arguments.

\paragraph{Null queries and measurability.}
At a fixed $x$ with $w_x=0$, the histogram and validation counts vanish almost surely.
Then $q_x=p_x$, $u_x=1$, and $b_x=B_x=0$. Ensemble stability still applies, and
Theorem~\ref{thm:phase} still controls the effect of the globally selected time on
this query. Thus the pointwise guarantee also covers every fixed zero-mass query.
The finite empirical primitive and its inverse are measurable in the data and seed.
The coordinate output map is therefore measurable for every $x$; countability
gives a measurable full predictor in the coordinate sigma field. Its risk is a
convergent nonnegative countable sum. Collecting the universal constants proves
all assertions of the proposition.
\end{proof}

\section{The rare-repair wrapper}\label{app:repair}
\begin{proof}[Proof of Theorem~\ref{thm:upper}]
This proof certifies the final decisions in Section~\ref{sec:confidence}.
The base learner there has excess at most $\alpha/4$ with failure at most
$\rho/64$ and pointwise disagreement at most $\rho/2$.
Lemma~\ref{lem:vc} gives the proper reference $g$ with
$C[d+\log(1/\eta_0)]/\alpha^2$ examples. Conditional on the already trained $f,g$,
two Hoeffding bounds on a fresh block of
$C\log(1/\eta_0)/\alpha^2$ examples give their simultaneous validation guarantee.
Both block sizes can be rounded up in advance; the ceiling terms are absorbed in $C$.

On reference success and validation success, if \eqref{eq:repair} selects $f$, then
\[
 R(f)\le\widehat R(f)+\alpha/16
 \le\widehat R(g)+\alpha/2+\alpha/16
 \le R(g)+5\alpha/8\le\OPT+3\alpha/4.
\]
If it selects $g$, reference success already gives excess at most $\alpha/8$.
Thus the final accuracy-failure probability is at most $2\eta_0\le\beta/4<\beta$.
This argument does not require base success.

For replicability, properness gives $R(g)\ge\OPT$ for every realization, including
those where the reference fails its accuracy guarantee. If the base learner and
validation both succeed, then
\[
 \widehat R(f)-\widehat R(g)
 \le R(f)-R(g)+\alpha/8\le3\alpha/8<\alpha/2.
\]
The repair branch is therefore selected only if base accuracy or validation fails.
Its probability is at most $\rho/64+\eta_0$ in each run.
Outside the union of the two repair events, the final outputs are the base outputs.
Consequently, for every fixed $x$,
\[
 \Pp\{\widehat f(x)\ne\widehat f'(x)\}
 \le\rho/2+2(\rho/64+\eta_0)
 \le35\rho/64<\rho.
\]
The disagreement cost of returning $g$ is paid for by the two repair-event probabilities.

The total budget is bounded by
\[
 C\frac{d+1}{\alpha^2\rho^2}
 +C\frac{d+\log(1/\beta)+\log(1/\rho)}{\alpha^2}.
\]
Here $\log(1/\eta_0)=\max(\log(8/\beta),\log(128/\rho))$ is bounded by the
sum of the displayed logarithms and an absolute constant. For $d\ge1$,
$d+1\le2d$ and $\log(1/\rho)\le C/\rho^2$ on the stated parameter range.
Absorbing these terms proves the desired sample bound. The parameters were scaled
only after the universal constants of Proposition~\ref{prop:base} were established.
\end{proof}

\section{Lower bounds for every class}\label{app:lower}
We reproduce the dimension and confidence arguments of Zhang~\cite{Zhang2026}, with the
full score calculation. They certify Section~\ref{sec:lower} and require only the
existence of a shattered set, so they apply to each class separately.

\subsection{The dimension--replicability term}
Fix any learner with budget $m$ satisfying Definition~\ref{def:model}. Choose distinct
points $x_1,\ldots,x_d$ shattered by $\mathcal H$ and encode both labels and predictions
as signs in $\{-1,+1\}$ in this appendix. Put $\gamma=5\alpha<1/2$.
For each $v\in\{-1,+1\}^d$, let $D_v$ have uniform input marginal on the shattered
points and conditional label mean $\gamma v_i$ at $x_i$.
Shattering makes the Bayes sign vector available in $\mathcal H$, so
\begin{equation}\label{eq:excess-hypercube}
 \OPT=(1-\gamma)/2,\qquad
 R_{D_v}(h)-\OPT=\frac\gamma d\sum_{i=1}^d\one\{h(x_i)\ne v_i\}
\end{equation}
for every binary output $h$, including improper ones.
Let $\bar e$ be the expected fraction of sign errors, averaging uniform $v$, data,
and the full seed. Accuracy and the bound of one on that fraction give
\begin{equation}\label{eq:average-error}
 \bar e\le\beta+(1-\beta)\alpha/\gamma\le7/25.
\end{equation}

\begin{lemma}[Prediction derivative bound]\label{lem:score}
Fix a coordinate $i$, the other signs $v_{-i}$, and the entire seed $\omega$.
Let $D_t$ vary the conditional mean at $x_i$ over $t\in[-\gamma,\gamma]$, leaving
all other means fixed. Define
$g_\omega(t)=\Pp_{S\sim D_t^m}\{A(S;\omega)(x_i)=+1\}$.
Then
\begin{equation}\label{eq:scorebound}
 |g_\omega'(t)|\le
 \sqrt{\frac m{d(1-t^2)}}\sqrt{g_\omega(t)(1-g_\omega(t))},
 \qquad
 2\E_\omega[g_\omega(t)(1-g_\omega(t))]\le\rho.
\end{equation}
\end{lemma}
\begin{proof}
The sample support is finite, so $g_\omega(t)$ is a polynomial.
Writing the observed sign labels as $\widetilde Y_j$, the log-likelihood score is
\[
 W_t=\sum_{j=1}^m\one\{X_j=x_i\}
             \frac{\widetilde Y_j-t}{1-t^2}.
\]
One summand has mean zero and variance
$d^{-1}(1-t^2)/(1-t^2)^2=1/[d(1-t^2)]$. Independence gives
$\E W_t=0$ and $\E W_t^2=m/[d(1-t^2)]$.
Let $J(S)=\one\{A(S;\omega)(x_i)=+1\}$. Differentiating the finite expectation yields
\[
 g_\omega'(t)=\E[J(S)W_t]
             =\E[(J(S)-g_\omega(t))W_t].
\]
Cauchy--Schwarz and $\Var J=g_\omega(t)(1-g_\omega(t))$ prove the first bound.
Given the full seed, the two samples are independent, so their prediction bits disagree
with probability exactly $2g_\omega(t)(1-g_\omega(t))$.
Replicability holds for every intermediate distribution $D_t$, proving the second bound.
\end{proof}

\begin{lemma}[Integrated endpoint constraint]\label{lem:endpoint}
The learner must satisfy $m\ge\kappa d/(\alpha^2\rho^2)$, with $\kappa$ as in
Theorem~\ref{thm:lower}.
\end{lemma}
\begin{proof}
Use the potential $\Psi(u)=3u^2-2u^3$ for $0\le u\le1$.
Its identities
\[
 \Psi(1-u)=1-\Psi(u),\qquad
 \tfrac98u-\Psi(u)=2u(u-3/4)^2\ge0
\]
imply, for each fixed $i,v_{-i},\omega$,
\[
 \Psi(g_\omega(\gamma))-\Psi(g_\omega(-\gamma))
 \ge1-\tfrac98\bigl[1-g_\omega(\gamma)+g_\omega(-\gamma)\bigr].
\]
Average over $i,v_{-i},\omega$. The expected bracket is $2\bar e$: its two terms
are the endpoint error probabilities, and a uniform sign averages them with weight
one half. By \eqref{eq:average-error}, the resulting lower bound is
$1-(9/4)\bar e\ge37/100$.

On the other hand, $\Psi'(u)=6u(1-u)$ and
$\sqrt{u(1-u)}\le1/2$. Lemma~\ref{lem:score} therefore gives
\[
 |(\Psi\circ g_\omega)'(t)|
 \le3\sqrt{\frac m{d(1-t^2)}}g_\omega(t)(1-g_\omega(t)).
\]
Integrate over $[-\gamma,\gamma]$, average over the seed, and then over $i,v_{-i}$.
The second part of \eqref{eq:scorebound} and
$\int_{-\gamma}^{\gamma}(1-t^2)^{-1/2}dt=2\arcsin\gamma$ give
\[
 \frac{37}{100}\le3\rho\arcsin\gamma\sqrt{m/d}
 \le5\pi\alpha\rho\sqrt{m/d}.
\]
For the last step, convexity of $\arcsin$ on $[0,1/2]$ bounds it by the chord
$\arcsin t\le\pi t/3$, and $\gamma=5\alpha$.
The derivative bound is uniform in the seed on this compact interval, so the
integrals and expectations are justified even for an unrestricted seed space.
Squaring proves the claimed constant.
\end{proof}

\subsection{The confidence term and the combination}
\begin{proof}[Proof of Theorem~\ref{thm:lower}]
Lemma~\ref{lem:endpoint} supplies the dimension term. For the confidence term,
put all input mass on one shattered point, and give its label means $+2\alpha$
and $-2\alpha$ in two experiments. A wrong prediction has excess $2\alpha>\alpha$.
Thus the learner's output bit is a test whose two error probabilities are at most $\beta$.

The relative entropy of the two one-example Bernoulli distributions is
\[
 2\alpha\log\frac{1+2\alpha}{1-2\alpha}
 \le\frac{8\alpha^2}{1-4\alpha^2}\le\frac{25}{3}\alpha^2.
\]
The first inequality follows from
$\log((1+t)/(1-t))=2\int_0^t(1-u^2)^{-1}du\le2t/(1-t^2)$.
Independence makes the $m$-sample relative entropy $m$ times the one-example value.
Adding the data-independent seed with the same law in both experiments does not
change relative entropy. Data processing to the output bit now gives
\[
 \frac{25}{3}m\alpha^2
 \ge(1-2\beta)\log\frac{1-\beta}{\beta}
 \ge\frac12\log(1/\beta).
\]
For clarity, the first lower bound uses the monotonicity of binary relative entropy
when its first parameter is at least $1-\beta$ and its second is at most $\beta$.
For $\beta<1/10$, the second follows from $1-2\beta\ge4/5$ and
$\log(1-\beta)\ge-\beta/(1-\beta)\ge-1/9$, together with
$\log(1/\beta)\ge\log10$.
It follows that $m\ge(3/50)\alpha^{-2}\log(1/\beta)$.

Set $A=d/(\alpha^2\rho^2)$ and $B=\alpha^{-2}\log(1/\beta)$.
Both $m\ge\kappa A$ and $m\ge(3/50)B$ apply to this same arbitrary learner and
class, though their adversarial distributions differ. Hence
$A+B\le m(1/\kappa+50/3)$, which is exactly the stated lower bound with
$\bar\kappa=3\kappa/(3+50\kappa)$.
\end{proof}

\section{Completion of the optimal-rate theorem}\label{app:final}
\begin{proof}[Proof of Theorem~\ref{thm:main}]
For $d\ge1$, Theorems~\ref{thm:upper} and~\ref{thm:lower} apply on the common
range $(0,1/10)^3$ and give the two inequalities with universal constants.
For $d=0$, no singleton is shattered. If two members of $\mathcal H$ differed at
one point, that singleton would be shattered, so all class members agree everywhere.
Return the fixed $h_0$ without observing data. Its excess risk and disagreement are
both zero for every distribution. Thus the sample complexity is zero.
\end{proof}

\clearpage
\section{Earlier constructions and complete supporting analyses}\label{app:retained}
This appendix incorporates the mathematical development of the first version of
Zhang~\cite{Zhang2026}, published on 9 September 2026 with the version-specific DOI
\href{https://doi.org/10.5281/zenodo.22658876}{10.5281/zenodo.22658876}.
It supplies the detailed proofs promised in Section~\ref{sec:retained}, including the
structural and realizable results, explicit general-purpose algorithms, lower-bound
constants, algorithm-specific bounds, and counterexamples. The consolidated earlier rate statement
comes after its ingredients, in Appendix~\ref{legacy:section:summary}. Notation tables,
constant calculations, and complete pseudocode follow.

\paragraph{Numbering and scope.}
A result numbered $k$ in that version is numbered S$k$ here. The same convention applies
to its displayed equation numbers. Thus \legacyref{6.2}{Theorem} and
\eqref{legacy:eq:6.7} retain their September identifiers with the S prefix.
External references retain their original numbers and are identified by their source citation.

The retained analysis uses signed labels $\{-1,+1\}$. Identify a label $y$ from
Section~\ref{sec:model} with $2y-1$, and identify a predictor $h$ with $2h-1$.
This bijection of classes and outputs preserves VC dimension, sample size,
properness, risk, and every prediction-disagreement event. All symbols in the following
subsections are local to this signed convention and are defined there. In particular,
their $X,H,e(f)$ correspond to $\mathcal X,\mathcal H,R(f)$ after recoding.
Their gate parameters and soft-shift strength are not the smooth-profile and phase
parameters of Section~\ref{sec:construction}. The strict probability convention is
stated in \legacyref{2.1}{Definition}; the stated bounds also hold under the non-strict
convention of Definition~\ref{def:model}.

\paragraph{How the earlier and new guarantees fit together.}
Theorem~\ref{thm:main} resolves the uniform agnostic rate left open by the earlier
algorithms. The parameter regions in \legacyref{2.2}{Theorem} therefore describe
where those particular bounds attain the lower-bound scale. The phase construction
attains that scale throughout the parameter range. The earlier numerical inequalities,
constructive guarantees, and counterexamples are retained, together with the
complementary proper and realizable results.

\subsubsection*{Dated result correspondence}
Table~\ref{tab:provenance} records the specific mathematical content inherited from
the separately accessible September v1~\cite{Zhang2026}. The source is the 63-page
\texttt{C04\_paper.pdf} attached to the version-specific record 22658876.
The page numbers below are the printed PDF pages of that file. A listed page is the
start of a statement; its proof may continue on subsequent pages.
The S-prefixed statements retain the original mathematical content, with cross-references
adapted to this manuscript. The table includes the supporting definitions and classical
inputs, whose citations are given at their statements.

\begingroup
\small
\begin{longtable}{>{\raggedright\arraybackslash}p{0.40\textwidth}>{\raggedright\arraybackslash}p{0.25\textwidth}>{\raggedright\arraybackslash}p{0.25\textwidth}}
\caption{Results already contained in the Zenodo v1 published on 9 September 2026,
DOI 10.5281/zenodo.22658876, and their locations here.}\label{tab:provenance}\\
\toprule
Inherited content & Original v1 locator & Location in this manuscript \\
\midrule
\endfirsthead
\multicolumn{3}{l}{\tablename~\thetable\ (continued)}\\
\toprule
Inherited content & Original v1 locator & Location in this manuscript \\
\midrule
\endhead
\bottomrule
\endfoot
Model, countable rounding, self-normalized gate, explicit ERM bounds, and independent-block moments
& Definition 2.1 (p.\ 3); Lemmas 3.1--3.4 (pp.\ 8, 10, 11, 14); Proposition 3.5 (p.\ 15)
& \legacyref{2.1}{Definition}; \legacyref{3.1}{Lemma}--\legacyref{3.4}{Lemma}; \legacyref{3.5}{Proposition} \\[4pt]
Gate-and-round conversion, three ensemble budgets, and additive-confidence rare repair
& Theorem 4.1 (p.\ 15); Proposition 4.2 (p.\ 17); Lemma 4.3 (p.\ 20); Theorem 4.4 (p.\ 21)
& \legacyref{4.1}{Theorem}, \legacyref{4.2}{Proposition}, \legacyref{4.3}{Lemma}, \legacyref{4.4}{Theorem} \\[4pt]
Agnostic dimension--replicability and confidence lower bounds, including explicit constants
& Theorem 5.1 (p.\ 22); Lemma 5.2 (p.\ 25)
& \legacyref{5.1}{Theorem}, \legacyref{5.2}{Lemma}; also Theorem~\ref{thm:lower} and Appendix~\ref{app:lower} \\[4pt]
Countable proper core, ordered quantile coupling, and exact proper learning for every VC-one class
& Lemma 6.1 (p.\ 26); Theorem 6.2 (p.\ 28)
& \legacyref{6.1}{Lemma}, \legacyref{6.2}{Theorem}; Theorem~\ref{thm:retained-vcone} \\[4pt]
Exact proper finite-atom learning and the faster realizable bound with stability on all distributions
& Theorem 6.3 (p.\ 30); Corollary 6.4 (p.\ 31)
& \legacyref{6.3}{Theorem}, \legacyref{6.4}{Corollary}; Theorems~\ref{thm:retained-atoms} and~\ref{thm:retained-realizable} \\[4pt]
Statement qualifications, interface-specific obstructions, SMALL-region variance, and explicit counterexamples
& Remarks 6.5 (p.\ 33), 6.6 (p.\ 34), and 7.7 (p.\ 48)
& \legacyref{6.5}{Remark}, \legacyref{6.6}{Remark}, \legacyref{7.7}{Remark}; Section~\ref{sec:retained} \\[4pt]
Clipped-shift risk inequality, countable product form, and signed-histogram concentration
& Lemma 7.1 (p.\ 38); Corollary 7.2 (p.\ 39); Lemma 7.3 (p.\ 41)
& \legacyref{7.1}{Lemma}, \legacyref{7.2}{Corollary}, \legacyref{7.3}{Lemma} \\[4pt]
SOFT-VC, the exact fixed-confidence order, and the repaired full-parameter upper bounds
& Theorems 7.4 (p.\ 41), 7.5 (p.\ 44), and 7.6 (p.\ 45)
& \legacyref{7.4}{Theorem}--\legacyref{7.6}{Theorem}; Theorem~\ref{thm:retained-soft} \\[4pt]
Consolidated earlier rates and their explicit absorption regions
& Theorem 2.2 (p.\ 4); proof completion following Remark 7.7 (p.\ 54)
& \legacyref{2.2}{Theorem} and Appendix~\ref{legacy:section:summary} \\[4pt]
Notation, constant calculations, logarithm accounting, complete algorithms, and external statement locators
& Appendices A--B (pp.\ 56--62); source locators in the bibliography (pp.\ 62--63)
& Appendices~\ref{legacy:section:constants}--\ref{legacy:section:sources}
\end{longtable}
\endgroup

\noindent\textbf{What is added after v1.}
Theorem~\ref{thm:upper} gives the uniform, log-free additive-confidence upper bound for
arbitrary positive finite VC dimension throughout the stated parameter range.
Its new base learner is Proposition~\ref{prop:base}. The shared-seed phase comparison
in Theorem~\ref{thm:phase} and the compatible weak and strong risk certificates in
Theorem~\ref{thm:risk} support that learner. These results and their proofs are in
Sections~\ref{sec:construction}--\ref{sec:confidence} and
Appendices~\ref{app:statistics}--\ref{app:repair}. Combining this extension with the
inherited lower bounds yields Theorem~\ref{thm:main}. The earlier absorption regions
describe the sample budgets of the September algorithms; the phase construction
attains the optimal rate throughout the parameter range.

\begin{center}
\begin{tabular}{>{\raggedright\arraybackslash}p{0.30\textwidth}>{\raggedright\arraybackslash}p{0.31\textwidth}>{\raggedright\arraybackslash}p{0.30\textwidth}}
\toprule
Module or purpose & Direct guarantee & Supporting analysis \\
\midrule
Earlier gate-and-round & \legacyref{4.1}{Theorem}, \legacyref{4.2}{Proposition} &
\legacyref{3.1}{Lemma}--\legacyref{3.4}{Lemma}, \legacyref{3.5}{Proposition} \\
Gate confidence repair & \legacyref{4.4}{Theorem} &
\legacyref{4.3}{Lemma}, \legacyref{4.2}{Proposition} \\
Proper VC-one selection & Theorem~\ref{thm:retained-vcone}; \legacyref{6.2}{Theorem} &
\legacyref{6.1}{Lemma}, ERM estimates \\
Full finite-atom class & Theorem~\ref{thm:retained-atoms}; \legacyref{6.3}{Theorem} &
\legacyref{3.2}{Lemma} \\
Realizable gate-and-round & Theorem~\ref{thm:retained-realizable}; \legacyref{6.4}{Corollary} &
Gate stability, consistent learning, Bernstein \\
SOFT-VC and repair & Theorem~\ref{thm:retained-soft}; \legacyref{7.4}{Theorem}--\legacyref{7.6}{Theorem} &
\legacyref{3.4}{Lemma}, \legacyref{6.1}{Lemma}, \legacyref{7.1}{Lemma}--\legacyref{7.3}{Lemma} \\
Optimality and comparisons & \legacyref{2.2}{Theorem}; Theorem~\ref{thm:lower} &
\legacyref{5.1}{Theorem}, \legacyref{5.2}{Lemma} \\
Interface limitations & Section~\ref{sec:retained} &
\legacyref{6.5}{Remark}, \legacyref{6.6}{Remark}, \legacyref{7.7}{Remark} \\
\bottomrule
\end{tabular}
\end{center}

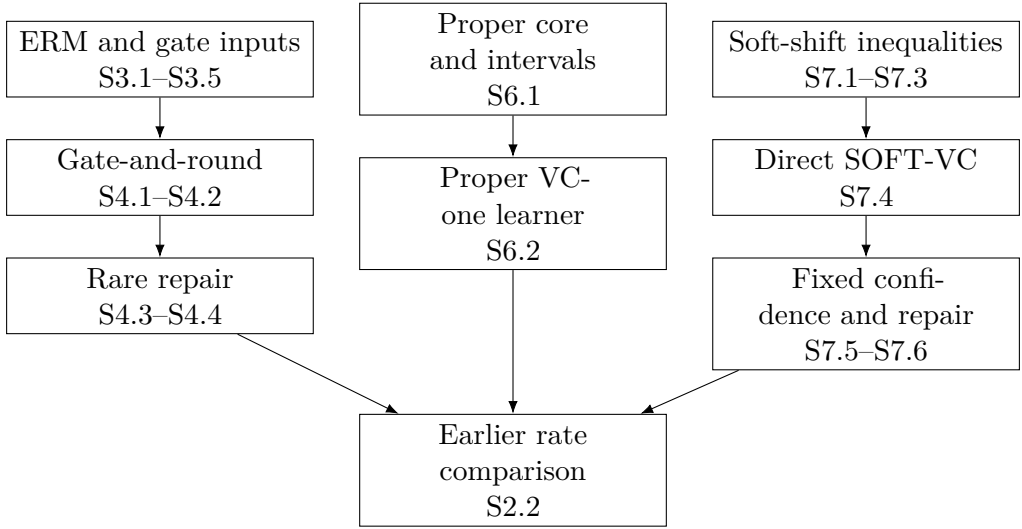
\begin{figure}[htbp]
\centering
\begin{tikzpicture}[>=Latex,every node/.style={draw=black,fill=white,text=black,
 align=center,text width=3.8cm,minimum height=1cm},node distance=0.55cm and 0.6cm]
\node (stat) {ERM and gate inputs\\S3.1--S3.5};
\node[right=of stat] (order) {Proper core and intervals\\S6.1};
\node[right=of order] (shift) {Soft-shift inequalities\\S7.1--S7.3};
\node[below=of stat] (gate) {Gate-and-round\\S4.1--S4.2};
\node[below=of order] (proper) {Proper VC-one learner\\S6.2};
\node[below=of shift] (soft) {Direct SOFT-VC\\S7.4};
\node[below=of gate] (repair) {Rare repair\\S4.3--S4.4};
\node[below=of soft] (softrepair) {Fixed confidence and repair\\S7.5--S7.6};
\node[below=1.9cm of proper] (summary) {Earlier rate comparison\\S2.2};
\draw[->] (stat)--(gate);
\draw[->] (order)--(proper);
\draw[->] (shift)--(soft);
\draw[->] (gate)--(repair);
\draw[->] (soft)--(softrepair);
\draw[->] (repair)--(summary);
\draw[->] (proper)--(summary);
\draw[->] (softrepair)--(summary);
\end{tikzpicture}
\caption{Selected dependencies among the retained constructions. The proper and soft
branches also use the ERM inputs. The comparison also uses S5.1--S5.2 and the finite-atom
case S6.3. Those arrows, and the separate realizable and counterexample branches,
are omitted here and are specified in the table and the individual proof introductions.}
\label{fig:retained-deps}
\end{figure}
\clearpage

\begingroup
\allowdisplaybreaks
\setlength{\parskip}{3pt}
\subsection{Signed-label notation and conventions}\label{legacy:section:2}

Throughout, \(X\) is a finite or countably infinite set, labels belong to \(\{-1,+1\}\), and \(H\subseteq\{-1,+1\}^{X}\) has finite VC dimension \(d\ge1\). All logarithms are natural. Unless explicitly stated otherwise,
\[
0<\alpha,\beta,\rho<\frac1{10}.
\]

Standard references for the classical tools used below (Hoeffding's inequality, the Sauer--Shelah lemma, symmetrization and chaining for VC classes, the Azuma--Hoeffding martingale inequality, and Bernstein's inequality) are \cite{Hoe63,Sau72,She72,VC71,Dud67,Azu67,BLM13}; they are used in explicit forms stated in the text.

For a distribution \(D\) on \(X\times\{-1,+1\}\), write
\[
w_x=D(X=x),\qquad
y_x=\mathbb E_D[Y\mid X=x],
\]
where \(y_x=0\) when \(w_x=0\). For a binary-valued function \(f\),
\[
e(f):=\operatorname{err}_D(f)
=\sum_{x\in X}w_x\frac{1-f(x)y_x}{2},
\qquad
\operatorname{OPT}:=\inf_{h\in H}e(h).
\]

\par\noindent\textbf{Role and dependencies.} This fixes the local signed-label notation and the strict-probability convention for all retained constructions.

\hypertarget{legacy:2.1}{}
\begin{legacydefinition}{Definition S2.1}{Pointwise replicability and agnostic learning}

A learner \(A\), using a sample \(S\sim D^m\) and a random string \(r\) independent of \(S\), is \(\rho\)-pointwise replicable if
\[
\forall D\ \forall x\in X:\quad
\Pr_{S,S'\sim D^m,r}
\bigl[A(S;r)(x)\ne A(S';r)(x)\bigr]<\rho,
\]
where the two executions use the same \(r\). It is an agnostic \((\alpha,\beta)\)-learner if
\[
\forall D:\quad
\Pr_{S,r}\bigl[e(A(S;r))>\operatorname{OPT}+\alpha\bigr]<\beta.
\]

The shared random string is unrestricted. Countably many independent uniform random variables can be generated by partitioning and interlacing the binary digits of a single uniform random variable on \([0,1]\); a fixed convention handles dyadic expansions. This realizes the shared seed for the constructions below. The general algorithms of \cite{HIY} use the same convention. For finite shared randomness over arbitrary domains, \cite{HIY}, Theorem 1.13, provides an alternative bound with weaker dependence on \(\alpha\) and \(\rho\).

\legacyref{2.1}{Definition} uses strict inequalities, as does \cite{HIY}, Definition 1.1. If replicability is instead required as \(\Pr[\cdot]\le\rho\) and failure as \(\Pr[\cdot]\le\beta\), every statement of this paper is unchanged: all upper bounds are proved with strict inequalities, and the lower-bound arguments of Appendix~\ref{legacy:section:5} use the replicability and accuracy hypotheses only through non-strict inequalities, so they apply verbatim under the non-strict convention.

Let \(m_H(\alpha,\beta,\rho)\) be the least fixed sample size of a learner satisfying both requirements. Let \(m^*(d,\alpha,\beta,\rho)\) be the supremum of \(m_H\) over the classes under consideration with VC dimension \(d\).

Define
\[
\kappa=\frac{1369}{250000\pi^2},
\qquad
\bar\kappa=\frac{3\kappa}{3+50\kappa}.
\]

\end{legacydefinition}

\subsection{Statistical inputs for the earlier constructions}\label{legacy:section:3}

\legacyheading{3.1}{Countable, bias-weighted rounding}

\par\noindent\textbf{Role and dependencies.} This is the rounding input to \legacyref{4.1}{Theorem} and the conditional SMALL-region estimate in \legacyref{7.7}{Remark}.

\hypertarget{legacy:3.1}{}
\begin{legacytheorem}{Lemma S3.1}{Countable Hoeffding inequality and biased rounding}

Let \((U_x)_{x\in I}\) be independent integrable random variables indexed by a countable set. Suppose that \(U_x\) has essential range of length at most \(c_x\), where
\[
\sum_x c_x<\infty.
\]
Then
\[
Z_0=\sum_x(U_x-\mathbb EU_x)
\]
converges absolutely almost surely. If \(V=\sum_xc_x^2>0\), then, for every \(t>0\),
\[
\Pr(Z_0>t)\le \exp\!\left(-\frac{2t^2}{V}\right).
\tag{S3.1}\label{legacy:eq:3.1}
\]
If \(V=0\), then \(Z_0=0\) almost surely. If additionally \(\sum_x\mathbb E|U_x|<\infty\), then \(Z_0=Z-\mathbb EZ\), where \(Z=\sum_xU_x\).

In particular, fix \(R\subseteq X\) and \(p_x\in[-1,1]\). Let \(v_x\) be independent uniforms on \([-1,1]\), and put
\[
q_x=
\begin{cases}
+1,&v_x\le p_x,\\
-1,&v_x>p_x.
\end{cases}
\]
For
\[
Z_R=\sum_{x\in R}w_x\frac{1-q_xy_x}{2},
\qquad
B_R=\sum_{x\in R}w_x^2y_x^2,
\]
one has
\[
\mathbb EZ_R
=\sum_{x\in R}w_x\frac{1-p_xy_x}{2},
\tag{S3.2}\label{legacy:eq:3.2}
\]
\[
\operatorname{Var}(Z_R)
=\frac14\sum_{x\in R}w_x^2y_x^2(1-p_x^2),
\tag{S3.3}\label{legacy:eq:3.3}
\]
and
\[
\Pr(Z_R>\mathbb EZ_R+t)
\le e^{-2t^2/B_R},
\tag{S3.4}\label{legacy:eq:3.4}
\]
with the zero-\(B_R\) convention above. The assertion also holds conditionally when \(R\) and \((p_x)\) are jointly random but independent of \((v_x)\).

\end{legacytheorem}
\begin{proof}

For a bounded random variable \(U\) with range length \(c\), let
\[
F(\lambda)=\log\mathbb Ee^{\lambda U}.
\]
Under exponential tilting,
\[
F''(\lambda)=\operatorname{Var}_{\lambda}(U)\le \frac{c^2}{4}.
\]
The variance bound follows by subtracting the midpoint of an interval containing \(U\): its squared distance from that midpoint is at most \(c^2/4\). Taylor's integral formula gives
\[
\mathbb Ee^{\lambda(U-\mathbb EU)}
\le e^{\lambda^2c^2/8}
\qquad(\lambda\in\mathbb R).
\tag{S3.5}\label{legacy:eq:3.5}
\]

Almost surely, \(|U_x-\mathbb EU_x|\le c_x\) simultaneously for all \(x\). Hence the centered series converges absolutely. For finite partial sums, independence and \eqref{legacy:eq:3.5} give
\[
\mathbb E\exp\!\left(\lambda\sum_{x\le n}(U_x-\mathbb EU_x)\right)
\le
\exp\!\left(\frac{\lambda^2}{8}\sum_{x\le n}c_x^2\right).
\]
The integrands are bounded by \(e^{|\lambda|\sum_xc_x}\). Dominated convergence therefore yields
\[
\mathbb Ee^{\lambda Z_0}\le e^{\lambda^2V/8}.
\]
For \(\lambda>0\), Markov's inequality gives
\[
\Pr(Z_0>t)\le e^{-\lambda t+\lambda^2V/8}.
\]
Choosing \(\lambda=4t/V\) proves \eqref{legacy:eq:3.1}. Absolute integrability of the uncentered series, when assumed, permits termwise expectation.

For the application,
\[
\mathbb E q_x=p_x,\qquad
\operatorname{Var}(q_x)=1-p_x^2.
\]
The \(x\)-summand has range length \(w_x|y_x|\), not \(w_x\). Its range lengths sum to at most \(1\). Equations \eqref{legacy:eq:3.2}--\eqref{legacy:eq:3.4} now follow from the first part. Conditioning on \(R,(p_x)\) gives the final assertion. 

If \(p_x=T^{-1}\sum_{t=1}^T h_t(x)\), then
\[
\mathbb EZ_R
=\frac1T\sum_{t=1}^T
\Pr_D[h_t(X)\ne Y,\ X\in R].
\tag{S3.6}\label{legacy:eq:3.6}
\]

\end{proof}

\legacyheading{3.2}{A self-normalized three-valued gate}

For a sample of size \(M\), define
\[
N_x=\#\{j:X_j=x\},\qquad
C_x=\sum_{j:X_j=x}Y_j,
\]
\[
s_x=
\begin{cases}
C_x/\sqrt{N_x},&N_x>0,\\
0,&N_x=0.
\end{cases}
\]
Given \(\tau\in(0,1)\), put \(a=\sqrt{M\tau}\). Independently for every point and sign, sample shared
\[
u_x^+,u_x^-\sim\operatorname{Unif}[1,2].
\]
Define
\[
\operatorname{TEST}(x)=
\begin{cases}
+1,&s_x\ge au_x^+,\\
-1,&s_x\le-au_x^-,\\
\mathrm{SMALL},&\text{otherwise}.
\end{cases}
\tag{S3.7}\label{legacy:eq:3.7}
\]
The two non-SMALL cases cannot overlap.

\par\noindent\textbf{Role and dependencies.} This certifies the gate used in \legacyref{4.1}{Theorem}, \legacyref{6.3}{Theorem}, and \legacyref{6.4}{Corollary}. It gives both per-query stability and the sign/coverage estimates used for accuracy.

\hypertarget{legacy:3.2}{}
\begin{legacytheorem}{Lemma S3.2}{Self-normalized gate}

The gate \eqref{legacy:eq:3.7} has the following properties.

\begin{enumerate}[leftmargin=1.6em]
\item For every fixed \(x\), including \(w_x=0\), two independent samples with shared thresholds satisfy
\[
\Pr(\operatorname{TEST}_1(x)\ne\operatorname{TEST}_2(x))
\le\frac{\sqrt2}{a}.
\tag{S3.8}\label{legacy:eq:3.8}
\]
\item Its probability of assigning any sign opposite to a nonzero conditional mean is at most
\[
\Pr\bigl(\exists x:\operatorname{TEST}(x)y_x<0,\
\operatorname{TEST}(x)\ne\mathrm{SMALL}\bigr)
\le Me^{-a^2/2}.
\tag{S3.9}\label{legacy:eq:3.9}
\]
\item Its probability of failing to correctly gate any dangerous atom is at most
\[
\Pr\bigl(\exists x:w_xy_x^2>9\tau,\
\operatorname{TEST}(x)\ne\operatorname{sign}(y_x)\bigr)
\le \frac4{9\tau}e^{-a^2/8}.
\tag{S3.10}\label{legacy:eq:3.10}
\]
\end{enumerate}
There is also a deterministic implication:
\[
|s_x-y_x\sqrt{Mw_x}|\le a/2,\quad
w_xy_x^2\ge\frac{25}{4}\tau
\quad\Longrightarrow\quad
\operatorname{TEST}(x)=\operatorname{sign}(y_x).
\tag{S3.11}\label{legacy:eq:3.11}
\]
For noiseless labels, \(N_x\ge25a^2/4\) implies \(|s_x|\ge5a/2\). The count threshold \(25a^2/4\) and the mass threshold \(25\tau/4\) are distinct quantities.

\end{legacytheorem}
\begin{proof}

Fix \(x\), and abbreviate \(w=w_x\), \(y=y_x\), \(N=N_x\), and \(\lambda=Mw\). Conditional on \(N\),
\[
\mathbb E[s_x\mid N]=y\sqrt N,\qquad
\operatorname{Var}(s_x\mid N)=(1-y^2)\mathbf1_{\{N>0\}}.
\]
For \(\lambda>0\),
\[
\operatorname{Var}(\sqrt N)
\le \mathbb E(\sqrt N-\sqrt\lambda)^2
\le \frac{\mathbb E(N-\lambda)^2}{\lambda}
=1-w.
\]
Consequently,
\[
\operatorname{Var}(s_x)
=(1-y^2)\Pr(N>0)+y^2\operatorname{Var}(\sqrt N)
\le1.
\tag{S3.12}\label{legacy:eq:3.12}
\]
For \(w=0\), \(N=s_x=0\), so the same conclusion holds.

Given two realized statistics \(s,s'\), a change of gate output requires a threshold crossing. The positive threshold is uniform on \([a,2a]\), the negative threshold on \([-2a,-a]\), and these supports are disjoint. The sum of the lengths crossed in these two supports is at most \(|s-s'|\). Therefore the conditional disagreement probability is at most \(|s-s'|/a\). Using \eqref{legacy:eq:3.12},
\[
\Pr(\operatorname{TEST}_1\ne\operatorname{TEST}_2)
\le\frac{\mathbb E|s-s'|}{a}
\le\frac{\sqrt{2\operatorname{Var}(s)}}a
\le\frac{\sqrt2}a.
\]

For sign safety, condition on all sample locations. At a point with \(y>0\), a negative output requires \(C_x/\sqrt{N_x}\le-a\). Hoeffding's inequality for the \(N_x\) labels gives probability at most \(e^{-a^2/2}\). The case \(y<0\) has the identical bound with signs reversed. An unobserved point cannot be gated. There are at most \(M\) observed points, proving \eqref{legacy:eq:3.9}. Points with \(y=0\) do not count as wrong signs.

For coverage, the binomial Chernoff bound implies
\[
\Pr(|\sqrt N-\sqrt\lambda|\ge t)\le2e^{-t^2}.
\tag{S3.13}\label{legacy:eq:3.13}
\]
For completeness, the relevant rate function is
\[
h(u)=u\log u-u+1,
\]
and
\[
h(u)-(\sqrt u-1)^2
=2\sqrt u\bigl(\sqrt u\log\sqrt u-\sqrt u+1\bigr)\ge0.
\]
Applying the binomial exponential-moment bound
\[
\mathbb Ee^{zN}\le e^{\lambda(e^z-1)}
\]
to both tails proves \eqref{legacy:eq:3.13}, including the endpoint \(N=0\).

Write the centered label term as zero when \(N=0\). Then
\[
s_x-y\sqrt\lambda
=\frac{C_x-Ny}{\sqrt N}\mathbf1_{\{N>0\}}
+y(\sqrt N-\sqrt\lambda).
\]
Conditional label Hoeffding and \eqref{legacy:eq:3.13} give
\[
\Pr(|s_x-y\sqrt\lambda|>t)
\le2e^{-t^2/8}+2e^{-t^2/4}
\le4e^{-t^2/8}.
\tag{S3.14}\label{legacy:eq:3.14}
\]
If \(wy^2>9\tau\), the absolute mean \( |y|\sqrt\lambda\) exceeds \(3a\). Failure to output the correct sign then requires a deviation greater than \(a\). There are at most \(1/(9\tau)\) such atoms because \(\sum_xw_xy_x^2\le1\). Union bounding \eqref{legacy:eq:3.14} proves \eqref{legacy:eq:3.10}.

Finally, the assumptions of \eqref{legacy:eq:3.11} give
\[
|y_x|\sqrt{Mw_x}\ge5a/2,
\]
so \(s_x\) has the correct sign and magnitude at least \(2a\). For noiseless labels, \(|s_x|=\sqrt{N_x}\), proving the count statement. 

\end{proof}

\legacyheading{3.3}{An explicit VC agnostic learner}

\par\noindent\textbf{Role and dependencies.} This supplies the explicit proper ERM budgets used by the ensemble, reference, and filtering modules. Its moment consequences are derived next in \legacyref{3.4}{Lemma}.

\hypertarget{legacy:3.3}{}
\begin{legacytheorem}{Lemma S3.3}{Explicit VC agnostic learning}

For \(d\ge1\), a deterministic, everywhere-defined proper empirical risk minimizer is an agnostic \((\varepsilon,\delta)\)-learner whenever
\[
n\ge
\left\lceil
\frac{2^{16}}{\varepsilon^2}
\left(d+\log\frac2\delta\right)
\right\rceil,
\qquad 0<\varepsilon,\delta<1.
\tag{S3.15}\label{legacy:eq:3.15}
\]

\end{legacytheorem}
\begin{proof}

We give the entropy, symmetrization, and concentration calculations explicitly.

\textbf{Growth bound.} On \(k\) points, the number of traces satisfies
\[
\Pi_H(k)\le\sum_{j=0}^{\min(d,k)}\binom kj.
\tag{S3.16}\label{legacy:eq:3.16}
\]
To prove this, separate traces according to their last coordinate. The projected traces contribute a family of VC dimension at most \(d\); the projected traces admitting both last-coordinate labels contribute a family of VC dimension at most \(d-1\). Induction in \(k,d\), with Pascal's identity, gives \eqref{legacy:eq:3.16}. For \(k\ge d\),
\[
\sum_{j=0}^d\binom kj\le(ek/d)^d.
\tag{S3.17}\label{legacy:eq:3.17}
\]
Indeed, with \(a=d/k\le1\),
\[
a^d\sum_{j=0}^d\binom kj
\le(1+a)^k\le e^d.
\]

(The simplification \((ek/d)^d\) requires \(k\ge d\); for \(k<d\) we use the trivial bound \(2^k\) instead. The qualification is needed: the full class on four points has \(\Pi_H(1)=2\), whereas \((e/4)^4<1\).)

\textbf{Packing bound.} In the packing argument, the functions are \(\{0,1\}\)-valued. We apply it to the loss class
\[
\mathcal L_H=\{(x,y)\mapsto\mathbf 1[h(x)\ne y]:h\in H\}.
\]
For this class, the \(L_1(Q)\) distance equals the disagreement probability, and its growth function is at most \(\Pi_H\), since fixing the labels only flips trace coordinates. Let a finite family of \(N\ge2\) such \(\{0,1\}\)-valued functions be pairwise separated by more than \(\eta\) in \(L_1(Q)\), where \(0<\eta\le1\). Sample
\[
k=\left\lceil\frac{2\log N}{\eta}\right\rceil
\]
points from \(Q\). A union bound shows that the probability of any two functions having identical traces is less than \(1/2\). Thus \(N\le\Pi_H(k)\).

If \(k<d\), then \(N\le2^d\). If \(k\ge d\), put \(z=N^{1/d}\). Equations \eqref{legacy:eq:3.17} and the definition of \(k\) give
\[
z\le\frac e\eta(2\log z+1).
\]
This implies
\[
N\le(5e/\eta)^{2d}.
\tag{S3.18}\label{legacy:eq:3.18}
\]
To verify the implication, set \(A=5e/\eta>1\). The function
\(z/(2\log z+1)\) is increasing for \(z>\sqrt e\), and
\[
\frac{A^2}{4\log A+1}>\frac A5
\]
because \(5A>4\log A+1\). Thus \(z>A^2\) contradicts the preceding inequality. When \(\eta>1\), a packing contains at most one element; formula \eqref{legacy:eq:3.18} is used only for \(0<\eta\le1\).

\textbf{Empirical chaining.} Fix an \(n\)-point labeled sample. Its loss vectors form a finite subset \(V\subseteq\{0,1\}^n\). Label-dependent coordinate flips preserve the growth bound. Equip \(V\) with the normalized Euclidean norm
\[
\|v\|_n^2=\frac1n\sum_{i=1}^nv_i^2.
\]
Choose nets \(V_k\) of radius \(2^{-k}\), obtained from maximal packings at squared-distance scale \(4^{-k}\), and let \(V_0\) be a singleton. By \eqref{legacy:eq:3.18},
\[
\log(2|V_k||V_{k-1}|)\le16dk,\qquad k\ge1.
\tag{S3.19}\label{legacy:eq:3.19}
\]
For example, the left side is at most
\[
\log2+4d\log(5e)+2d(2k-1)\log4
\le d(6k+10)\le16dk.
\]

For each \(v\in V\), choose \(\pi_k(v)\in V_k\) with \(\|v-\pi_k(v)\|_n\le2^{-k}\), and let
\[
\mathcal I_k=\{\pi_k(v)-\pi_{k-1}(v):v\in V\}
\]
be the set of chain increments. Then \(|\mathcal I_k|\le|V_k||V_{k-1}|\), and every member of \(\mathcal I_k\) has norm at most \(3\cdot2^{-k}\). For independent Rademacher signs \(\sigma_i\), every chain increment therefore has norm at most \(3\cdot2^{-k}\). The exponential-moment bound for a Rademacher sum and the maximum bound obtained by summing its exponential moments give
\[
\mathbb E_\sigma
\max_{\text{increments }a}
\left|\frac1n\sum_i\sigma_i a_i\right|
\le
3\cdot2^{-k}\sqrt{\frac{2\log(2|V_k||V_{k-1}|)}n}.
\]
The singleton term is at most \(1/\sqrt n\). Since
\[
\sum_{k\ge1}2^{-k}\sqrt k
\le\sum_{k\ge1}2^{-k}k=2,
\]
telescoping along the nets yields
\[
\mathbb E_\sigma\sup_{v\in V}
\left|\frac1n\sum_i\sigma_i v_i\right|
\le(1+24\sqrt2)\sqrt{\frac dn}
<35\sqrt{\frac dn}.
\tag{S3.20}\label{legacy:eq:3.20}
\]
The empirical vector set is finite, so sufficiently fine nets reproduce every vector; the telescoping argument has no limiting ambiguity.

\textbf{Uniform deviation.} Put
\[
Z=\sup_{h\in H}|P_n\ell_h-P\ell_h|.
\]
Introduce an independent ghost sample, apply conditional Jensen, randomize the two samples by coordinatewise swaps, and use the triangle inequality. These steps give
\[
\mathbb EZ
\le2\,\mathbb E_{S,\sigma}
\sup_h\left|\frac1n\sum_i\sigma_i\ell_h(Z_i)\right|
\le70\sqrt{\frac dn}.
\tag{S3.21}\label{legacy:eq:3.21}
\]
Changing one observation changes \(Z\) by at most \(1/n\). The Doob martingale exposing the sample therefore has conditionally centered increments with conditional range length at most \(1/n\). Applying \eqref{legacy:eq:3.5} conditionally and iterating gives
\[
\Pr(Z>\mathbb EZ+t)\le e^{-2nt^2}.
\tag{S3.22}\label{legacy:eq:3.22}
\]

Let \(q=\log(2/\delta)\). Under \eqref{legacy:eq:3.15}, with probability at least \(1-\delta\),
\[
Z\le70\sqrt{\frac dn}+\sqrt{\frac q{2n}}
\le\frac{\sqrt{4900+1/2}}{256}\varepsilon
=\frac{99}{256\sqrt2}\varepsilon
<\frac\varepsilon2.
\tag{S3.23}\label{legacy:eq:3.23}
\]
The constant \(2^{16}=256^2\) was chosen precisely to make this inequality hold.

An empirical minimum exists because empirical risks take values in the finite set \(\{0,1/n,\ldots,1\}\). Since the labeled sample space is countable, one can fix a minimizer for every sample; this defines a measurable, deterministic ERM. On the event in \eqref{legacy:eq:3.23}, for every \(h\in H\),
\[
e(\widehat h)\le P_n\ell_{\widehat h}+Z
\le P_n\ell_h+Z
\le e(h)+2Z.
\]
Taking the infimum over \(h\) proves the result even when OPT is not attained. 

\end{proof}

\legacyheading{3.4}{Square-exponential moments of ERM excess risk}

Set
\[
K_E=2^{24}.
\]
This constant equals \(2^{16}\cdot16^2\), the constant obtained from \legacyref{3.3}{Lemma} at accuracy \(t/16\).

\par\noindent\textbf{Role and dependencies.} This converts \legacyref{3.3}{Lemma} into the tail and moment bounds used in \legacyref{3.5}{Proposition}, \legacyref{6.2}{Theorem}, and \legacyref{7.4}{Theorem}.

\hypertarget{legacy:3.4}{}
\begin{legacytheorem}{Lemma S3.4}{ERM square-exponential moment}

Let \(\widehat h_n\) be a proper ERM based on \(n\ge1\) observations, and let
\[
Z_n=e(\widehat h_n)-\operatorname{OPT}\in[0,1].
\]
Then, for every \(t>0\),
\[
\Pr(Z_n>t)\le\min\{1,20e^{d-nt^2/K_E}\},
\tag{S3.24}\label{legacy:eq:3.24}
\]
\[
\mathbb E\exp\!\left(\frac{nZ_n^2}{2K_E}\right)\le21e^d,
\tag{S3.25}\label{legacy:eq:3.25}
\]
and
\[
\mathbb EZ_n\le2^{14}\sqrt{\frac dn}.
\tag{S3.26}\label{legacy:eq:3.26}
\]

\end{legacytheorem}
\begin{proof}

Only the case \(0<t\le1\) and
\(\delta_t:=20e^{d-nt^2/K_E}<1\) needs proof. Applying \legacyref{3.3}{Lemma} at accuracy \(t/16\) and confidence \(\delta_t\) is valid because its required sample size is
\[
\left\lceil
\frac{K_E}{t^2}\left(d+\log\frac2{\delta_t}\right)
\right\rceil
=
\left\lceil n-\frac{K_E\log10}{t^2}\right\rceil
\le n.
\]
Thus \(\Pr(Z_n>t)\le\Pr(Z_n>t/16)<\delta_t\).

For a nonnegative random variable, tail integration gives
\[
\begin{aligned}
\mathbb Ee^{nZ_n^2/(2K_E)}
&=1+\int_0^\infty e^s
\Pr\!\left(\frac{nZ_n^2}{2K_E}>s\right)\,ds\\
&\le1+20e^d\int_0^\infty e^{-s}\,ds
\le21e^d.
\end{aligned}
\]
Jensen's inequality then gives
\[
(\mathbb EZ_n)^2
\le\mathbb EZ_n^2
\le\frac{2K_E}{n}(d+\log21)
\le\frac{2^{28}d}{n},
\]
where \(d\ge1\) and \(2(1+\log21)<16\) were used. This proves \eqref{legacy:eq:3.26}. 

\end{proof}
\par\noindent\textbf{Role and dependencies.} This averages independent ERM excess risks using \legacyref{3.4}{Lemma}. It supports the product-moment block choice in \legacyref{4.2}{Proposition}.

\hypertarget{legacy:3.5}{}
\begin{legacytheorem}{Proposition S3.5}{Independent ERM blocks without a confidence logarithm per block}

Let \(0<\gamma<1\), and train \(T\) independent proper ERMs, each on
\[
n=
\left\lceil
\frac{8K_E}{\alpha^2}
\left(d+\log21+\frac{\log(1/\gamma)}T\right)
\right\rceil
\tag{S3.27}\label{legacy:eq:3.27}
\]
observations. Then
\[
\Pr\!\left(
\frac1T\sum_{t=1}^T e(h_t)>
\operatorname{OPT}+\frac\alpha2
\right)\le\gamma.
\tag{S3.28}\label{legacy:eq:3.28}
\]

\end{legacytheorem}
\begin{proof}

Let \(Z_t=e(h_t)-\operatorname{OPT}\). Independence and \eqref{legacy:eq:3.25} imply
\[
\mathbb E\exp\!\left(\frac n{2K_E}\sum_{t=1}^TZ_t^2\right)
\le(21e^d)^T.
\]
If \(T^{-1}\sum_tZ_t>\alpha/2\), Cauchy--Schwarz gives
\[
\sum_tZ_t^2>\frac{T\alpha^2}{4}.
\]
Markov's inequality therefore bounds the probability by
\[
\exp\!\left(T(d+\log21)-\frac{nT\alpha^2}{8K_E}\right)
\le\gamma.
\]
The factor \(8K_E=2^{27}\) comes from the moment exponent \(1/(2K_E)\) and the threshold \((\alpha/2)^2\). 

\end{proof}

\subsection{Gate-and-round and rare repair}\label{legacy:section:4}

\legacyheading{4.1}{The gate-and-round conversion}

The shared string is partitioned into:

\begin{itemize}[leftmargin=1.6em]
\item independent gate thresholds \((u_x^+,u_x^-)_{x\in X}\);
\item independent rounding uniforms \((v_x)_{x\in X}\);
\item independent base-learner seeds \(r_1,\ldots,r_T\), shared between the two executions.
\end{itemize}
All three parts are mutually independent. In the explicit constructions, the ERMs are deterministic and the base seeds may be omitted.

\par\noindent\textbf{Role and dependencies.} This certifies the gate-and-round module of Section~\ref{sec:soft-retained}. The gate estimate is \legacyref{3.2}{Lemma}; the final rounding estimate is \legacyref{3.1}{Lemma}.

\hypertarget{legacy:4.1}{}
\begin{legacytheorem}{Theorem S4.1}{Generic pointwise-replicable conversion}

Set
\[
L=\log\frac4\beta,\qquad
\tau=\frac{\alpha^2}{18L},\qquad
K=\log\frac{16}{\tau\beta}
=\log\frac{288L}{\alpha^2\beta},
\tag{S4.1}\label{legacy:eq:4.1}
\]
\[
M=\left\lceil\frac{64}{\tau}(\rho^{-2}+K)\right\rceil,
\qquad
T\ \text{any integer with}\ T\ge\lceil\rho^{-2}\rceil.
\tag{S4.2}\label{legacy:eq:4.2}
\]
Use \(M\) observations for the gate in \legacyref{3.2}{Lemma} and \(T\) independent training blocks to obtain \(h_1,\ldots,h_T\). Put
\[
p_x=\frac1T\sum_{t=1}^Th_t(x),
\]
and output
\[
f(x)=
\begin{cases}
\operatorname{TEST}(x),&\operatorname{TEST}(x)\ne\mathrm{SMALL},\\
+1,&\operatorname{TEST}(x)=\mathrm{SMALL},\ v_x\le p_x,\\
-1,&\operatorname{TEST}(x)=\mathrm{SMALL},\ v_x>p_x.
\end{cases}
\tag{S4.3}\label{legacy:eq:4.3}
\]
If
\[
\Pr\!\left(\frac1T\sum_te(h_t)>
\operatorname{OPT}+\frac\alpha2\right)\le\frac\beta4,
\tag{S4.4}\label{legacy:eq:4.4}
\]
then \(f\) is an agnostic \((\alpha,\beta)\)-learner and is \(\rho\)-pointwise replicable. Its respective bounds are
\[
\Pr(e(f)>\operatorname{OPT}+\alpha)
\le\frac{85}{144}\beta<\beta,
\tag{S4.5}\label{legacy:eq:4.5}
\]
\[
\Pr(f_1(x)\ne f_2(x))
\le\frac{5}{4\sqrt2}\rho<\rho
\qquad\text{for every fixed }x.
\tag{S4.6}\label{legacy:eq:4.6}
\]
The proof uses \(T\) only through \(1/\sqrt{2T}\le\rho/\sqrt2\); \legacyref{4.4}{Theorem} applies the statement with \(T_0=\lceil32/\rho^2\rceil\).

\end{legacytheorem}
\begin{proof}

Let \(z=a^2=M\tau\). By construction,
\[
z\ge64(\rho^{-2}+K)>6400.
\]
For \(z\ge6400\), \(\log z\le z/4\). Hence
\[
Me^{-z/2}
=\frac z\tau e^{-z/2}
\le\frac1\tau e^{-z/4}
\le\frac1\tau e^{-K}
=\frac\beta{16}.
\tag{S4.7}\label{legacy:eq:4.7}
\]
Also,
\[
\frac4{9\tau}e^{-z/8}
\le\frac4{9\tau}e^{-K}
=\frac\beta{36}.
\tag{S4.8}\label{legacy:eq:4.8}
\]

Let \(G\) be the event that every gated point has a Bayes-optimal sign and every point with \(w_xy_x^2>9\tau\) is correctly gated. \legacyref{3.2}{Lemma} and \eqref{legacy:eq:4.7}--\eqref{legacy:eq:4.8} imply
\[
\Pr(G^c)\le\frac{13}{144}\beta<\frac\beta8.
\tag{S4.9}\label{legacy:eq:4.9}
\]

Condition on the sigma-field \(\mathcal F\) generated by:

\begin{itemize}[leftmargin=1.6em]
\item all gate and base-training observations;
\item all gate thresholds;
\item all base-learner seeds;
\end{itemize}
but \textbf{not} the rounding uniforms. Let
\[
\Sigma=\{x:\operatorname{TEST}(x)\ne\mathrm{SMALL}\},
\qquad R=X\setminus\Sigma.
\]
On \(G\), replacing any base prediction by the gate's Bayes sign cannot increase risk. Thus
\[
\mathbb E_{\rm round}[e(f)\mid\mathcal F]
\le\frac1T\sum_te(h_t).
\tag{S4.10}\label{legacy:eq:4.10}
\]
Moreover,
\[
B_R=\sum_{x\in R}w_x^2y_x^2
\le\left(\sup_{x\in R}w_xy_x^2\right)\sum_{x\in R}w_x
\le9\tau=\frac{\alpha^2}{2L}.
\]
\legacyref{3.1}{Lemma}, with \(t=\alpha/2\), gives conditional rounding failure probability at most
\[
e^{-L}=\frac\beta4.
\tag{S4.11}\label{legacy:eq:4.11}
\]
Combining \eqref{legacy:eq:4.4}, \eqref{legacy:eq:4.9}, and \eqref{legacy:eq:4.11} gives
\[
\frac{13}{144}\beta+\frac\beta4+\frac\beta4
=\frac{85}{144}\beta.
\]

For replicability, fix an arbitrary \(x\). No accuracy event is used. Gate disagreement is at most \(\sqrt2/a\). If both gates are SMALL, shared-uniform rounding disagrees with conditional probability
\[
\frac12|p_x^{(1)}-p_x^{(2)}|.
\]
Conditioned on the shared base seeds, the two predictions within each block are independent and identically distributed. Therefore
\[
\mathbb E[h_t^{(1)}(x)-h_t^{(2)}(x)\mid r_t]=0,
\qquad
\mathbb E[(h_t^{(1)}(x)-h_t^{(2)}(x))^2\mid r_t]\le2.
\]
Independence across blocks gives
\[
\mathbb E(p_x^{(1)}-p_x^{(2)})^2\le\frac2T.
\]
Consequently,
\[
\Pr(f_1(x)\ne f_2(x))
\le\frac{\sqrt2}{a}+\frac1{\sqrt{2T}}
\le\left(\frac{\sqrt2}{8}+\frac1{\sqrt2}\right)\rho
=\frac5{4\sqrt2}\rho<\rho.
\]
This includes zero-mass points. 

\end{proof}
\par\noindent\textbf{Role and dependencies.} This supplies three concrete ensemble budgets for \legacyref{4.1}{Theorem}. It uses \legacyref{3.3}{Lemma} and \legacyref{3.5}{Proposition}, and contributes the direct gate route to \legacyref{2.2}{Theorem}.

\hypertarget{legacy:4.2}{}
\begin{legacytheorem}{Proposition S4.2}{Three block strategies and explicit costs}

The hypothesis of \legacyref{4.1}{Theorem} can be met by each of the following strategies.

For the sample-cost statements below, namely \eqref{legacy:eq:4.12}, \eqref{legacy:eq:4.15}, \eqref{legacy:eq:4.16}, and \eqref{legacy:eq:2.2}, take \(T=\lceil\rho^{-2}\rceil\); the accuracy guarantees hold for every admissible \(T\ge\lceil\rho^{-2}\rceil\) of \legacyref{4.1}{Theorem}.

U --- per-block union bound:
\[
n_U=
\left\lceil
\frac{2^{18}}{\alpha^2}
\left(d+\log\frac{8T}{\beta}\right)
\right\rceil.
\tag{S4.12}\label{legacy:eq:4.12}
\]

M --- Markov control of bad blocks:
\[
n_M=
\left\lceil
\frac{2^{20}}{\alpha^2}
\left(d+\log\frac{32}{\alpha\beta}\right)
\right\rceil.
\tag{S4.13}\label{legacy:eq:4.13}
\]

P --- square-exponential moments:
\[
n_P=
\left\lceil
\frac{2^{27}}{\alpha^2}
\left(d+\log21+\frac{\log(8/\beta)}T\right)
\right\rceil.
\tag{S4.14}\label{legacy:eq:4.14}
\]
For strategy P, the final accuracy failure probability is at most \(67\beta/144<\beta/2\). Its exact sample count is
\[
N_P=M+Tn_P,\qquad
M=\left\lceil
\frac{1152L}{\alpha^2}(\rho^{-2}+K)
\right\rceil,
\tag{S4.15}\label{legacy:eq:4.15}
\]
and
\[
N_P\le
\frac{2^{30}}{\alpha^2}
\left(\frac{d+L}{\rho^2}+LK\right).
\tag{S4.16}\label{legacy:eq:4.16}
\]
Choosing between U and M also gives \eqref{legacy:eq:2.2}.

\end{legacytheorem}
\begin{proof}

For U, use \legacyref{3.3}{Lemma} with accuracy \(\alpha/2\) and failure probability \(\beta/(4T)\). If all \(T\) blocks succeed, their average excess risk is at most \(\alpha/2\). The union bound proves \eqref{legacy:eq:4.4}.

For M, use accuracy \(\alpha/4\) and failure probability \(\alpha\beta/16\). Let \(F_{\rm bad}\) be the fraction of unsuccessful blocks. Then
\[
\mathbb EF_{\rm bad}<\frac{\alpha\beta}{16},
\qquad
\Pr(F_{\rm bad}>\alpha/4)\le\frac\beta4.
\]
Since every risk is at most \(1\) and OPT is nonnegative,
\[
\frac1T\sum_te(h_t)-\operatorname{OPT}
\le\frac\alpha4+F_{\rm bad}.
\]
This proves \eqref{legacy:eq:4.4}.

For P, apply \legacyref{3.5}{Proposition} with \(\gamma=\beta/8\), giving
\[
\Pr\!\left(\frac1T\sum_te(h_t)>
\operatorname{OPT}+\alpha/2\right)\le\frac\beta8.
\tag{S4.17}\label{legacy:eq:4.17}
\]

Here is the complete bridge from this average-risk statement to the final learner. Use the gate event \(G\) from \eqref{legacy:eq:4.9}, whose probability is at least \(1-13\beta/144>1-\beta/8\). Condition on the base samples, the gate sample, the base seeds, and the gate thresholds, but leave the rounding uniforms unconditioned. On \(G\), \eqref{legacy:eq:4.10} bounds the conditional rounded risk by the average base risk. On \(G\) and the complement of the event in \eqref{legacy:eq:4.17}, \legacyref{3.1}{Lemma} bounds the remaining rounding failure by \(\beta/4\). Thus the total failure probability is
\[
\frac{13}{144}\beta+\frac18\beta+\frac14\beta
=\frac{67}{144}\beta<\frac\beta2.
\tag{S4.18}\label{legacy:eq:4.18}
\]
No assertion about a fixed realization of the rounding randomness is used.

One can also certify \(G\) directly with the sharper counting observation that a possibly gated point must satisfy \(N_x\ge a^2\), so there are at most \(1/\tau\) such points. For a dangerous point with \(\lambda=Mw_x\), except with probability \(2e^{-a^2/8}\),
\[
N_x\ge\frac34\lambda,\qquad
C_x\operatorname{sign}(y_x)\ge\frac45N_x|y_x|.
\]
Indeed, the count failure is at most \(e^{-\lambda/32}\le e^{-a^2/8}\), and the conditional label failure is at most
\[
e^{-N_xy_x^2/50}
\le e^{-27a^2/200}
\le e^{-a^2/8}.
\]
These inequalities force
\[
|s_x|\ge\frac{6\sqrt3}{5}a>2a.
\]
This alternative certification is consistent with, but not needed to improve, the budget \eqref{legacy:eq:4.18}.

For the cost, \(64/\tau=1152L/\alpha^2\), proving \eqref{legacy:eq:4.15}. With \(T=\lceil\rho^{-2}\rceil\) and \(\rho<1/10\), since
\[
T\le2\rho^{-2},\qquad
d+\log21\le d+L,\qquad
\log(8/\beta)\le2L,
\]
the base cost is at most
\[
\frac{2^{29}}{\alpha^2}\frac{d+L}{\rho^2}+2\rho^{-2}.
\]
Adding the gate cost and the ceiling gives \eqref{legacy:eq:4.16}; the coefficient
\(2^{29}+1155\) is less than \(2^{30}\).

Finally, write
\[
B_\rho=d+\log(32/\beta)+\log(1/\rho),
\qquad
B_\alpha=d+\log(32/\beta)+\log(1/\alpha).
\]
Because \(T=\lceil\rho^{-2}\rceil\) and \(\rho<1/10\),
\[
T\le\frac{101}{100}\rho^{-2}.
\]
For U,
\[
d+\log(8T/\beta)\le2B_\rho,
\]
while the bracket for M equals \(B_\alpha\). Select U when \(B_\rho\le B_\alpha\), and M otherwise. With \(\mathcal B=\min(B_\rho,B_\alpha)\), the selected block size is at most
\[
2^{20}\mathcal B/\alpha^2+1.
\]
The gate's \(\rho^{-2}\) term is at most
\(1152\mathcal B/(\rho^2\alpha^2)\), since \(L\le\mathcal B\). Including both ceilings, the coefficient is bounded by
\[
\frac{101}{100}(2^{20}+1)+1153
=1060215.77<2^{21}.
\]
The remaining gate term is \(1152LK/\alpha^2\), proving \eqref{legacy:eq:2.2}. 

\end{proof}

\legacyheading{4.2}{Rare repair and additive confidence}

\par\noindent\textbf{Role and dependencies.} This is the earlier rare-repair interface behind Section~\ref{sec:confidence}. It combines a reliable stable base with a proper reference from \legacyref{3.3}{Lemma} and independent validation.

\hypertarget{legacy:4.3}{}
\begin{legacytheorem}{Lemma S4.3}{Rare repair}

Suppose \(B\) is \(\rho/2\)-pointwise replicable and
\[
\Pr(e(B)>\operatorname{OPT}+\alpha/4)<\rho/64.
\tag{S4.19}\label{legacy:eq:4.19}
\]
Set
\[
\eta=\min\{\beta/8,\rho/128\}.
\]
Independently of the sample used by \(B\), train a proper ERM \(g\) on
\[
n_{\rm ref}=
\left\lceil
\frac{2^{22}}{\alpha^2}
\left(d+\log\frac2\eta\right)
\right\rceil
\tag{S4.20}\label{legacy:eq:4.20}
\]
observations. Use a further independent validation sample of size
\[
n_{\rm val}=
\left\lceil
\frac{128}{\alpha^2}\log\frac4\eta
\right\rceil.
\tag{S4.21}\label{legacy:eq:4.21}
\]
Writing \(f=B(S)\), output
\[
F=
\begin{cases}
f,&\widehat e(f)\le\widehat e(g)+\alpha/2,\\
g,&\text{otherwise}.
\end{cases}
\tag{S4.22}\label{legacy:eq:4.22}
\]
Then \(F\) is an agnostic \((\alpha,\beta)\)-learner and
\[
\Pr(F_1(x)\ne F_2(x))
<\frac{35}{64}\rho<\rho
\qquad\text{for every fixed }x.
\tag{S4.23}\label{legacy:eq:4.23}
\]

\end{legacytheorem}
\begin{proof}

\legacyref{3.3}{Lemma}, at accuracy \(\alpha/8\), explains the constant
\(2^{22}=2^{16}\cdot8^2\) in \eqref{legacy:eq:4.20}, and gives
\[
\Pr(e(g)>\operatorname{OPT}+\alpha/8)<\eta.
\tag{S4.24}\label{legacy:eq:4.24}
\]
Conditionally on \(f,g\), validation Hoeffding and a union bound give
\[
\Pr\left(
|\widehat e(f)-e(f)|>\alpha/16
\ \text{or}\
|\widehat e(g)-e(g)|>\alpha/16
\right)
\le4e^{-n_{\rm val}\alpha^2/128}\le\eta.
\tag{S4.25}\label{legacy:eq:4.25}
\]

Define the branch indicator
\[
I=\mathbf 1\{\widehat e(f)>\widehat e(g)+\alpha/2\},\qquad W=\{I=1\}.
\]
Thus \(W\) denotes selection of the repair branch, not equality of the output function with \(g\); set \(F=f\) when \(I=0\) and \(F=g\) when \(I=1\). Properness gives
\(e(g)\ge\operatorname{OPT}\) on every realization, whether or not \eqref{legacy:eq:4.24} succeeds. If \(e(f)\le\operatorname{OPT}+\alpha/4\) and neither validation deviation in \eqref{legacy:eq:4.25} occurs, then
\[
\widehat e(f)-\widehat e(g)
\le\alpha/4+2(\alpha/16)=3\alpha/8<\alpha/2.
\]
Therefore
\[
W\subseteq
\{e(f)>\operatorname{OPT}+\alpha/4\}
\cup\{\text{a validation deviation in \eqref{legacy:eq:4.25}}\},
\]
and
\[
\Pr(W)<\rho/64+\eta\le3\rho/128.
\tag{S4.26}\label{legacy:eq:4.26}
\]

For two executions,
\[
\{F_1(x)\ne F_2(x)\}
\subseteq
\{f_1(x)\ne f_2(x)\}\cup W_1\cup W_2.
\]
This inclusion covers all four repair/no-repair combinations and requires no independence between \(W_1,W_2\). Equations \eqref{legacy:eq:4.19} and \eqref{legacy:eq:4.26} give
\[
\Pr(F_1(x)\ne F_2(x))
<\rho/2+6\rho/128
=\frac{35}{64}\rho.
\]

For accuracy, assume \eqref{legacy:eq:4.24} succeeds and neither validation deviation occurs. If \(F=g\), then \(e(F)\le\operatorname{OPT}+\alpha/8\). If \(F=f\), then
\[
e(f)\le e(g)+\alpha/2+2(\alpha/16)
\le\operatorname{OPT}+3\alpha/4.
\]
The failure probability is at most \(2\eta\le\beta/4<\beta\). 

Properness is used through the realization-wise inequality \(e(g)\ge\operatorname{OPT}\); this inequality is the property needed to make repairs rare.

\end{proof}
\par\noindent\textbf{Role and dependencies.} This applies \legacyref{4.3}{Lemma} to the gate learner of \legacyref{4.2}{Proposition}, giving the repaired gate route in \legacyref{2.2}{Theorem}.

\hypertarget{legacy:4.4}{}
\begin{legacytheorem}{Theorem S4.4}{Additive-confidence upper bound}

Let
\[
\ell=\log\frac{256}{\rho},\qquad
K_0=\log\frac{294912\ell}{\alpha^2\rho},
\]
\[
T_0=\left\lceil\frac{32}{\rho^2}\right\rceil,
\qquad
\tau_0=\frac{\alpha^2}{288\ell},
\]
\[
M_0=
\left\lceil
\frac{18432\ell}{\alpha^2}
\left(\frac4{\rho^2}+K_0\right)
\right\rceil,
\tag{S4.27}\label{legacy:eq:4.27}
\]
\[
n_0=
\left\lceil
\frac{2^{22}}{\alpha^2}
\left(d+\log\frac{512T_0}{\rho}\right)
\right\rceil.
\tag{S4.28}\label{legacy:eq:4.28}
\]
Run the conversion with target parameters
\[
(\text{accuracy},\text{failure},\text{replicability})
=(\alpha/4,\rho/64,\rho/2),
\]
using \(T_0\) blocks, each of size \(n_0\), and then apply \legacyref{4.3}{Lemma}. The resulting sample count
\[
N_R=M_0+T_0n_0+n_{\rm ref}+n_{\rm val}
\tag{S4.29}\label{legacy:eq:4.29}
\]
satisfies
\[
N_R\le
\frac{2^{30}}{\alpha^2}
\left[
\frac{d+\ell}{\rho^2}
+\ell K_0+\log\frac1\beta
\right].
\tag{S4.30}\label{legacy:eq:4.30}
\]

\end{legacytheorem}
\begin{proof}

Substituting the indicated target parameters into \eqref{legacy:eq:4.1} gives
\[
\log\frac4{\rho/64}=\ell,\qquad
\tau_0=\frac{(\alpha/4)^2}{18\ell},
\]
and
\[
\log\frac{16}{\tau_0(\rho/64)}
=\log\frac{16\cdot288\cdot64\,\ell}{\alpha^2\rho}
=K_0.
\]
The product \(64\cdot288=18432\) gives \eqref{legacy:eq:4.27}. Strategy U at accuracy \(\alpha/8\) and block failure \(\rho/(256T_0)\) gives \eqref{legacy:eq:4.28}.

The minimal block count supplied by \legacyref{4.1}{Theorem} would be
\(\lceil4/\rho^2\rceil\); the larger \(T_0\) is valid and only decreases rounding disagreement. Thus the base learner satisfies the hypotheses of \legacyref{4.3}{Lemma}.

For the explicit bound, \(T_0\le33/\rho^2\), and
\[
\log\frac{512T_0}{\rho}
\le\log16896+3\log\frac1\rho
\le3\ell.
\]
Hence
\[
T_0n_0
\le
\frac{99\cdot2^{22}}{\alpha^2\rho^2}(d+\ell)
+\frac{33}{\rho^2}.
\]
Also,
\[
\log\frac2\eta
=\max\left\{\log\frac{16}{\beta},\log\frac{256}{\rho}\right\}
\le\log\frac1\beta+\ell,
\]
and
\[
\log\frac4\eta
\le\log\frac1\beta+2\ell.
\]
Adding the gate, reference, validation, and ceiling terms proves \eqref{legacy:eq:4.30}. For example, the coefficient of \((d+\ell)/(\rho^2\alpha^2)\) is bounded by
\[
99\cdot2^{22}+73728+2^{22}+256+40
=419504424<2^{30};
\]
the coefficients of \(\ell K_0/\alpha^2\) and \(\log(1/\beta)/\alpha^2\) are smaller. 

\end{proof}

\subsection{Earlier lower bounds and constant optimization}\label{legacy:section:5}

\legacyheading{5.1}{A direct multi-coordinate lower bound}

We average acceptance-probability paths over shattered coordinates and integrate a
cubic potential. Replicability at intermediate biases controls these paths and yields
the dimension-dependent lower bound with the stated failure-probability range and
explicit constants. A related one-coordinate coin-testing bound appears in
\cite{ILPS}, Lemma 7.2.

\par\noindent\textbf{Role and dependencies.} This retains the dimension-dependent lower bound, its bias-family formula, and its explicit constants. It is the earlier source of the dimension term in Theorem~\ref{thm:lower} and of the comparisons in \legacyref{2.2}{Theorem}.

\hypertarget{legacy:5.1}{}
\begin{legacytheorem}{Theorem S5.1}{Replicability lower bound in every positive dimension}

For every class \(H\) of VC dimension \(d\ge1\), every agnostic \(\rho\)-pointwise replicable \((\alpha,\beta)\)-learner satisfies
\[
m\ge\kappa\frac d{\rho^2\alpha^2},
\qquad
\kappa=\frac{1369}{250000\pi^2}.
\tag{S5.1}\label{legacy:eq:5.1}
\]
The assertion allows improper outputs, unlimited shared randomness, and arbitrary dependence among predictions.

More generally, for \(0<\gamma<1\),
\[
m\ge
\frac d{9\rho^2\arcsin^2\gamma}
\left[
1-\frac94
\left(\beta+(1-\beta)\frac\alpha\gamma\right)
\right]_+^2.
\tag{S5.2}\label{legacy:eq:5.2}
\]

\end{legacytheorem}
\begin{proof}

Choose shattered points \(x_1,\ldots,x_d\). For \(v\in\{-1,+1\}^d\), let \(D_v\) be uniform on these points with conditional means \(\gamma v_i\). Shattering gives
\[
\operatorname{OPT}=\frac{1-\gamma}{2}.
\]
For every output \(f\),
\[
e(f)-\operatorname{OPT}
=\frac\gamma d\#\{i:f(x_i)\ne v_i\}.
\tag{S5.3}\label{legacy:eq:5.3}
\]
If \(\gamma>\alpha\), accuracy implies that the mean fraction of coordinate errors, averaged over \(v\), samples, and randomness, obeys
\[
\bar e
\le\frac\alpha\gamma
+\left(1-\frac\alpha\gamma\right)\Pr(\text{accuracy failure})
<
\beta+(1-\beta)\frac\alpha\gamma.
\tag{S5.4}\label{legacy:eq:5.4}
\]

Fix \(i,v_{-i},r\), and vary only the \(i\)-th conditional mean \(t\in[-\gamma,\gamma]\). Define
\[
g_r(t)=
\Pr_{S\sim D_t^m}[A(S;r)(x_i)=+1].
\]
Because the sample support is finite, \(g_r\) is a polynomial in \(t\). The likelihood score is
\[
W_t=\sum_{j=1}^m
\mathbf1_{\{X_j=x_i\}}\frac{Y_j-t}{1-t^2}.
\]
It has mean zero and variance
\[
\operatorname{Var}(W_t)=\frac m{d(1-t^2)}.
\]
Differentiating the finite likelihood sum gives
\[
g_r'(t)=
\operatorname{Cov}\bigl(\mathbf1_{\{A(S;r)(x_i)=+1\}},W_t\bigr),
\]
so
\[
|g_r'(t)|
\le
\sqrt{\frac m{d(1-t^2)}}\sqrt{g_r(t)(1-g_r(t))}.
\tag{S5.5}\label{legacy:eq:5.5}
\]

For every intermediate distribution \(D_t\), pointwise replicability gives
\[
2\mathbb E_r[g_r(t)(1-g_r(t))]<\rho.
\tag{S5.6}\label{legacy:eq:5.6}
\]
This is where replicability for all distributions, rather than only the two endpoint distributions, is used.

Let
\[
\Psi(u)=3u^2-2u^3.
\]
It satisfies \(\Psi(1-u)=1-\Psi(u)\) and
\[
\frac98u-\Psi(u)=2u(u-\tfrac34)^2\ge0.
\]
Therefore
\[
\Psi(g_r(\gamma))-\Psi(g_r(-\gamma))
\ge
1-\frac98\bigl(1-g_r(\gamma)+g_r(-\gamma)\bigr).
\]
Averaging over \(i,v_{-i},r\), the right side is at least
\[
1-\frac94\bar e.
\tag{S5.7}\label{legacy:eq:5.7}
\]

On the other hand, using \eqref{legacy:eq:5.5} and
\(\sqrt{u(1-u)}\le1/2\),
\[
\left|\frac{d}{dt}\Psi(g_r(t))\right|
\le
3\sqrt{\frac m{d(1-t^2)}}g_r(t)(1-g_r(t)).
\]
Averaging \eqref{legacy:eq:5.6} and integrating gives
\[
\mathbb E_{i,v_{-i},r}
\bigl[\Psi(g_r(\gamma))-\Psi(g_r(-\gamma))\bigr]
\le
3\rho\,\arcsin(\gamma)\sqrt{\frac md}.
\tag{S5.8}\label{legacy:eq:5.8}
\]
All differentiations and integrations are justified by the finite-support likelihood representation and the uniform bound on the derivative over the compact interval.

Combining \eqref{legacy:eq:5.4}, \eqref{legacy:eq:5.7}, and \eqref{legacy:eq:5.8} proves \eqref{legacy:eq:5.2} when \(\gamma>\alpha\). When \(\gamma\le\alpha\), its positive-part bracket is zero, so the assertion is automatic.

For the stated constant, take \(\gamma=5\alpha<1/2\). Equation \eqref{legacy:eq:5.4} gives \(\bar e<7/25\), and thus \eqref{legacy:eq:5.7} is greater than
\[
1-\frac94\frac7{25}=\frac{37}{100}.
\]
The function \(\arcsin z/z\) is increasing on \((0,1)\). Therefore, for \(z\le1/2\),
\[
\arcsin z\le\frac\pi3 z.
\]
The right side of \eqref{legacy:eq:5.8} is at most
\[
5\pi\alpha\rho\sqrt{\frac md}.
\]
Squaring
\[
\frac{37}{100}
\le5\pi\alpha\rho\sqrt{\frac md}
\]
gives \eqref{legacy:eq:5.1}. 

\end{proof}
\subsubsection*{Constant optimization within the bias family}

The choice \(5\alpha\) gives the displayed constant. We can also optimize over the bias family. Put \(\gamma=c\alpha\), \(81/31<c<10\). Uniformly over \(\alpha,\beta<1/10\), \eqref{legacy:eq:5.2} gives the coefficient
\[
\kappa(c)=
\frac{\left(\frac{31}{40}-\frac{81}{40c}\right)^2}
{900\,\arcsin^2(c/10)}.
\tag{S5.9}\label{legacy:eq:5.9}
\]
Its unique maximizer \(c_*\) is characterized by
\[
81\sqrt{100-c_*^2}\arcsin(c_*/10)
=c_*(31c_*-81).
\tag{S5.10}\label{legacy:eq:5.10}
\]
To prove uniqueness, set \(c=10\sin t\). The stationary equation becomes
\[
\frac{\sin t+t\cos t}{\sin^2t}=\frac{310}{81}.
\]
The derivative of its left side is
\[
-\frac{t(1+\cos^2t)}{\sin^3t}<0.
\]
Consider the left side first on the enlarged interval \(t\in(0,\pi/2)\): it is strictly decreasing there, with limits \(+\infty\) and \(1\). At the lower endpoint \(t_0=\arcsin(81/310)\) of the actual optimization domain,
\[
\frac{\sin t_0+t_0\cos t_0}{\sin^2t_0}=\frac{310}{81}+\frac{t_0\cos t_0}{\sin^2t_0}>\frac{310}{81},
\]
so the unique solution of the stationary equation lies in the actual optimization domain. The coefficient vanishes at the left endpoint and its derivative is negative near \(c=10\). Thus this stationary point is the unique maximizer.

Numerically, from the displayed formulas,
\[
c_*\approx4.98426087836,\qquad
\kappa(c_*)\approx0.000554848925731.
\]
The simpler choice \(c=5\) gives the exact \(\kappa\) in \legacyref{5.1}{Theorem}.

\legacyheading{5.2}{Confidence lower bound}

\par\noindent\textbf{Role and dependencies.} This gives the separate confidence term. Combined with \legacyref{5.1}{Theorem}, it yields the lower half of \legacyref{2.2}{Theorem} and Theorem~\ref{thm:lower}.

\hypertarget{legacy:5.2}{}
\begin{legacytheorem}{Lemma S5.2}{Additive confidence lower bound}

Every agnostic \((\alpha,\beta)\)-learner for a class of positive VC dimension satisfies
\[
m\ge\frac3{50}\frac{\log(1/\beta)}{\alpha^2}.
\tag{S5.11}\label{legacy:eq:5.11}
\]
Consequently, every learner in \legacyref{5.1}{Theorem} satisfies
\[
m\ge
\frac{\bar\kappa}{\alpha^2}
\left(\frac d{\rho^2}+\log\frac1\beta\right).
\tag{S5.12}\label{legacy:eq:5.12}
\]

\end{legacytheorem}
\begin{proof}

Use one shattered point and conditional means \(+2\alpha\) and \(-2\alpha\). A wrong output sign incurs excess risk \(2\alpha>\alpha\). Hence the output bit distinguishes the two distributions with error less than \(\beta\) under each.

The KL divergence of one label under the positive distribution from one label under the negative distribution is
\[
2\alpha\log\frac{1+2\alpha}{1-2\alpha}
\le\frac{8\alpha^2}{1-4\alpha^2}
\le\frac{25}{3}\alpha^2.
\tag{S5.13}\label{legacy:eq:5.13}
\]
Here
\[
\log\frac{1+z}{1-z}
=2\int_0^z\frac{du}{1-u^2}
\le\frac{2z}{1-z^2}.
\]

The independent shared string adds zero KL divergence, and \(m\) observations multiply the one-observation divergence by \(m\). The log-sum inequality, applied to the two output events, gives data processing to the output bit. For Bernoulli output probabilities \(p>1-\beta\) and \(q<\beta\), their divergence is at least
\[
(1-2\beta)\log\frac{1-\beta}{\beta}
>\frac12\log\frac1\beta.
\tag{S5.14}\label{legacy:eq:5.14}
\]
For example, with \(b=\log(1/\beta)>\log10\),
\[
(1-2\beta)\log\frac{1-\beta}{\beta}
>\frac45(b-1/9)>\frac b2.
\]
Combining \eqref{legacy:eq:5.13}--\eqref{legacy:eq:5.14} proves \eqref{legacy:eq:5.11}.

Finally, if \(m\ge\kappa A\) and \(m\ge(3/50)B\), then
\[
A+B\le m\left(\frac1\kappa+\frac{50}{3}\right).
\]
Take \(A=d/(\rho^2\alpha^2)\), \(B=\log(1/\beta)/\alpha^2\) to obtain \eqref{legacy:eq:5.12}. 

\end{proof}

\subsection{Proper learning, realizability, and interface barriers}\label{legacy:section:6}

\legacyheading{6.1}{Ordered coupling for VC-one classes}

\par\noindent\textbf{Role and dependencies.} This supplies the proper core and ordered selection mechanism in Section~\ref{sec:proper-retained}. Its interval coupling is the structural input to \legacyref{6.2}{Theorem}; its core is also used by SOFT-VC.

\hypertarget{legacy:6.1}{}
\begin{legacytheorem}{Lemma S6.1}{Countable core, laminar order, and quantile coupling}

Every class \(H\) on a countable domain has a countable proper subclass \(H_c\subseteq H\) such that
\[
\inf_{h\in H_c}e(h)=\inf_{h\in H}e(h)
\quad\text{for every distribution }D.
\tag{S6.1}\label{legacy:eq:6.1}
\]

If \(\operatorname{VC}(H)=1\), fix \(h_0\in H_c\) and define
\[
A_x=\{h\in H_c:h(x)\ne h_0(x)\}.
\]
There is a linear order \(\prec\) on \(H_c\) in which every \(A_x\) is an interval. For any probability measure \(\mu\) on \(H_c\), there is a proper quantile sampler \(h_\mu(U)\), \(U\sim\operatorname{Unif}(0,1)\), with law \(\mu\). For any two such measures,
\[
\Pr_U(h_\mu(U)(x)\ne h_\nu(U)(x))
\le
|\mu(B_x^0)-\nu(B_x^0)|
+
|\mu(B_x^1)-\nu(B_x^1)|,
\tag{S6.2}\label{legacy:eq:6.2}
\]
where \(B_x^0\) is the initial segment preceding \(A_x\), and \(B_x^1=B_x^0\cup A_x\).

\end{legacytheorem}
\begin{proof}

For every finite \(F\subseteq X\) and every \(H\)-realizable labeling of \(F\), fix one representative in \(H\). Their union, with an additional fixed \(h_0\) if needed, is countable and is denoted \(H_c\). Given \(h\in H\) and \(\varepsilon>0\), choose finite \(F\) with \(D_X(F)>1-\varepsilon\). Its representative agrees with \(h\) on \(F\), so their risks differ by at most \(\varepsilon\). Taking infima and then \(\varepsilon\downarrow0\) proves \eqref{legacy:eq:6.1}.

For VC dimension one, the family \((A_x)\) is laminar. Otherwise two sets would have a nonempty intersection and two nonempty differences. Together with \(h_0\), which belongs to neither, these witnesses realize all four patterns on two points.

View \(H_c\) as a hypergraph on the vertex set \(X\) whose hyperedges are the sets \(\{x:h(x)\ne h_0(x)\}\), \(h\in H_c\); the family \((A_x)_{x\in X}\) is then its dual hypergraph. Bollob\'as, Pritchard, Rothvo\ss{}, and Scott \cite{BPRS12} observe (Section 4, the sentence following Definition 16 there) that a hypergraph has VC-dimension one if and only if its dual hypergraph is cross-free (their Definition 16); since \(h_0\) lies in no \(A_x\), the cross-free condition reduces to the laminar one (their Definition 17), which is the argument just given. The same section records the directed-tree representation of cross-free families due to Edmonds and Giles. Ben-David \cite{BD15} gives the primal form of the structure: \(\operatorname{VC}(H)\le1\) if and only if there are a tree ordering of the domain (Definition 3 there) and a representation under which every hypothesis is an initial segment (Theorem 4 there). What is used below is the countable linear order on \(H_c\) in which every \(A_x\) is an interval, and the shared quantile coupling built on it.

Enumerate \(H_c\). We construct compatible orders on successive finite prefixes. Suppose the old prefix is ordered and all old traces of \(A_x\) are intervals. Insert a new hypothesis \(h\). The sets \(A_x\) containing \(h\) form a chain. If their nonempty old traces exist, choose the smallest such trace \(C\), and insert \(h\) immediately before the first element of \(C\). Every trace of a set containing \(h\) contains \(C\), so adjoining \(h\) preserves its interval property.

If a set \(B\) does not contain \(h\) and its old trace meets \(C\), laminarity implies \(B\subseteq A\), where \(A\) is a set containing \(h\) with old trace \(C\). Hence the old trace of \(B\) is contained in \(C\), and insertion before \(C\) does not split it. A trace disjoint from \(C\) is not split either. If no nonempty old trace of a set containing \(h\) exists, insert \(h\) at the beginning. This completes the induction.

The limit order is linear. Any failure of interval convexity would have a three-element witness, already present in one finite prefix, so no failure occurs.

For a probability measure \(\mu\), set
\[
a_\mu(h)=\sum_{g\prec h}\mu(g),\qquad
b_\mu(h)=a_\mu(h)+\mu(h).
\]
The half-open intervals \([a_\mu(h),b_\mu(h))\) with \(\mu(h)>0\) are disjoint and have total Lebesgue measure one. Define \(h_\mu(U)=h\) on these intervals. \textbf{On the entire uncovered null set, define \(h_\mu(U)=h_0\), simultaneously for all coordinates.} This is a single proper hypothesis, not a coordinatewise default.

Because \(A_x\) is an interval, up to a null set,
\[
h_\mu(U)\in A_x
\quad\Longleftrightarrow\quad
U\in[\mu(B_x^0),\mu(B_x^1)).
\]
The symmetric difference of two intervals has measure at most the sum of the absolute differences of their endpoints. This proves \eqref{legacy:eq:6.2}. Empty \(A_x\) is handled by \(B_x^0=B_x^1=\varnothing\). 

\end{proof}
\par\noindent\textbf{Role and dependencies.} This proves the upper bound in Theorem~\ref{thm:retained-vcone}. It combines \legacyref{6.1}{Lemma} with the ERM estimates of \legacyref{3.3}{Lemma} and \legacyref{3.4}{Lemma}.

\hypertarget{legacy:6.2}{}
\begin{legacytheorem}{Theorem S6.2}{Exact proper learning for every VC-one class}

For every VC-one class on a countable domain, there is a proper \(\rho\)-pointwise replicable agnostic \((\alpha,\beta)\)-learner using
\[
O\!\left(
\frac1{\rho^2\alpha^2}
+\frac{\log(1/\beta)}{\alpha^2}
\right)
\tag{S6.3}\label{legacy:eq:6.3}
\]
observations. Its pointwise disagreement probability is at most \(\rho/10\).

\paragraph{Construction}

Use the countable ordered core in \legacyref{6.1}{Lemma} and deterministic ERMs in that core. Train
\[
T=\left\lceil\frac{2^{12}}{\rho^2}\right\rceil
\]
independent base hypotheses, each on
\[
n_b=
\left\lceil
\frac{2^{22}}{\alpha^2}(1+\log32)
\right\rceil
\tag{S6.4}\label{legacy:eq:6.4}
\]
observations. Train an independent reference \(g\) on
\[
s=\max\left\{
\left\lceil\frac{2^{50}}{\rho^2\alpha^2}\right\rceil,\
\left\lceil
\frac{2^{22}}{\alpha^2}\left(1+\log\frac8\beta\right)
\right\rceil
\right\}
\tag{S6.5}\label{legacy:eq:6.5}
\]
observations. Use an independent validation sample of size
\[
v=\max\left\{
\left\lceil\frac{2^{22}}{\rho^2\alpha^2}\right\rceil,\
\left\lceil
\frac{128}{\alpha^2}\log\frac{8(T+1)}{\beta}
\right\rceil
\right\}.
\tag{S6.6}\label{legacy:eq:6.6}
\]
The shared randomness consists of independent
\[
\Theta\sim\operatorname{Unif}[\alpha/4,\alpha/2],
\qquad U\sim\operatorname{Unif}(0,1).
\]
Accept \(h_j\) when
\[
\widehat e(h_j)-\widehat e(g)\le\Theta.
\tag{S6.7}\label{legacy:eq:6.7}
\]
If none is accepted, output \(g\). Otherwise, use \(U\) to quantile-sample the accepted multiset in the fixed order.

\end{legacytheorem}
\begin{proof}[Proof of accuracy]

The second term in \eqref{legacy:eq:6.5} and \legacyref{3.3}{Lemma} imply
\[
\Pr(e(g)>\operatorname{OPT}+\alpha/8)<\beta/4.
\]
Conditionally on all trained hypotheses, the second term in \eqref{legacy:eq:6.6} and a union bound over \(T+1\) models imply that every validation error is within \(\alpha/16\) of its true value except with probability at most \(\beta/4\).

On these two events, any accepted model satisfies
\[
e(h_j)\le e(g)+\Theta+2(\alpha/16)
\le\operatorname{OPT}+3\alpha/4.
\]
The empty-set fallback \(g\) has still smaller risk. Thus the accuracy failure probability is less than \(\beta/2<\beta\). Every output is in \(H_c\subseteq H\), including the quantile sampler's null-set fallback.

\end{proof}
\begin{proof}[Proof of replicability]

Fix \(D\) and a point \(x\). Let \(\mu\) be the law of one base ERM and define
\[
Z(h)=e(h)-\operatorname{OPT},
\]
\[
\nu_\theta(h)=\mu(h)\mathbf1_{\{Z(h)\le\theta\}},
\qquad
p_\theta=\nu_\theta(H_c),
\qquad
\lambda_\theta=\nu_\theta/p_\theta.
\tag{S6.8}\label{legacy:eq:6.8}
\]
By \eqref{legacy:eq:6.4} and \legacyref{3.3}{Lemma},
\[
\Pr_{h\sim\mu}(Z(h)>\alpha/8)<1/16.
\]
Since \(\theta\ge\alpha/4\), it follows that
\[
p_\theta\ge15/16
\quad\text{for every }\theta\in[\alpha/4,\alpha/2].
\tag{S6.9}\label{legacy:eq:6.9}
\]

For one execution, define the empirical unnormalized accepted measure
\[
\widehat\nu_\theta(B)
=\frac1T\sum_{j=1}^T
\mathbf1_{\{h_j\in B\}}
\mathbf1_{\{\widehat e(h_j)-\widehat e(g)\le\theta\}}.
\]
Let \(\widetilde\nu_\theta\) be defined by replacing the empirical test with \(Z(h_j)\le\theta\).

For any two real numbers \(a,b\), the probability that a uniform \(\Theta\) on an interval of length \(\alpha/4\) lies between them is at most \(4|a-b|/\alpha\). Therefore
\[
\begin{aligned}
&\mathbb E_{\Theta,\mathrm{samples}}
\left|\widehat\nu_\Theta(B)-\widetilde\nu_\Theta(B)\right|\\
&\quad\le
\frac4\alpha
\left(
\frac1{\sqrt v}+\mathbb E[e(g)-\operatorname{OPT}]
\right)
=:\varepsilon_{\rm mis}.
\end{aligned}
\tag{S6.10}\label{legacy:eq:6.10}
\]
To justify the \(1/\sqrt v\), each validation risk, conditionally on its model, has mean absolute deviation at most \(1/(2\sqrt v)\). The two deviations and the nonnegative reference excess risk bound the displacement between the two thresholds.

\legacyref{3.4}{Lemma}, with \(d=1\), and \eqref{legacy:eq:6.5}--\eqref{legacy:eq:6.6} give
\[
\frac1{\sqrt v}\le\frac{\rho\alpha}{2^{11}},
\qquad
\mathbb E[e(g)-\operatorname{OPT}]
\le\frac{2^{14}}{\sqrt s}
\le\frac{\rho\alpha}{2^{11}}.
\]
Hence
\[
\varepsilon_{\rm mis}\le\frac{\rho}{256}.
\tag{S6.11}\label{legacy:eq:6.11}
\]

For each fixed \(\theta\), \(\widetilde\nu_\theta(B)\) is an average of independent Bernoulli variables, so
\[
\mathbb E|\widetilde\nu_\theta(B)-\nu_\theta(B)|
\le\frac1{2\sqrt T}.
\]
Integrating this inequality and combining it with \eqref{legacy:eq:6.10} gives the repaired statement
\[
\boxed{
\mathbb E_{\Theta,\mathrm{samples}}
|\widehat\nu_\Theta(B)-\nu_\Theta(B)|
\le
\varepsilon_{\rm mis}+\frac1{2\sqrt T}.
}
\tag{S6.12}\label{legacy:eq:6.12}
\]
Equation \eqref{legacy:eq:6.12} averages over \(\Theta\) and the samples. The fixed-\(\theta\) estimate above controls the sampling term.

Let \(\widehat p=\widehat\nu_\theta(H_c)\), and let \(\widehat\lambda_\theta\) be the normalized accepted measure, or \(\delta_g\) if \(\widehat p=0\). If \(\widehat p>0\), then for \(a=\nu_\theta(B)\), \(\widehat a=\widehat\nu_\theta(B)\), \(p=p_\theta\),
\[
\left|\frac{\widehat a}{\widehat p}-\frac ap\right|
\le
\frac{|\widehat a-a|+|\widehat p-p|}{p}.
\tag{S6.13}\label{legacy:eq:6.13}
\]
This follows by writing the difference as
\[
\frac{\widehat a-a}{p}
+\frac{\widehat a}{\widehat p}\frac{p-\widehat p}{p}.
\]
If \(\widehat p=0\), the left side, interpreted using \(\delta_g\), is at most \(1\), while the right side is \((a+p)/p\ge1\); thus the bound also covers the empty accepted set.

Using \eqref{legacy:eq:6.9}, \eqref{legacy:eq:6.12}, and \eqref{legacy:eq:6.13}, the expected discrepancy between one execution and \(\lambda_\Theta\), on any fixed set \(B\), is at most
\[
\frac{16}{15}\left(2\varepsilon_{\rm mis}+\frac1{\sqrt T}\right).
\]
For two executions, the triangle inequality doubles this bound. Apply \legacyref{6.1}{Lemma} to the two initial segments \(B_x^0,B_x^1\). The final disagreement is at most
\[
\frac{16}{15}
\left(8\varepsilon_{\rm mis}+\frac4{\sqrt T}\right)
\le
\frac{16}{15}\left(\frac{\rho}{32}+\frac{\rho}{16}\right)
=\frac{\rho}{10}<\rho.
\tag{S6.14}\label{legacy:eq:6.14}
\]
This calculation is unconditional: neither accuracy of \(g\) nor any validation good event is assumed.

Finally, \eqref{legacy:eq:6.4}--\eqref{legacy:eq:6.6} give total sample size \eqref{legacy:eq:6.3}. In the validation count, \(\log(T+1)=O(\log(1/\rho))=O(\rho^{-2})\), so it is absorbed into the first term. Combining with \legacyref{5.1}{Theorem} and \legacyref{5.2}{Lemma} proves the exact order \eqref{legacy:eq:2.11}. 

\end{proof}
\subsubsection*{Structural limitation of the interval method}

On \([5]\), take the class of indicators of exactly two-element subsets and choose \(h_0=\mathbf1_{\{1,2\}}\). The sets \(A_3,A_4,A_5\) have pairwise nonempty intersections, witnessed respectively by \(\{3,4\},\{3,5\},\{4,5\}\), but their triple intersection is empty. Three intervals in a linear order cannot have this pattern: pairwise intersection of three intervals implies common intersection.

This class has VC dimension exactly two: it shatters \(\{1,2\}\), while three positive labels cannot be realized on any triple. Thus the interval-order construction does not extend to all VC-two classes.

\legacyheading{6.2}{\texorpdfstring{The full class on \(d\) atoms}{The full class on d atoms}}

\par\noindent\textbf{Role and dependencies.} This proves the direct finite-atom module and Theorem~\ref{thm:retained-atoms}, using \legacyref{3.2}{Lemma}. The matching lower bound is \legacyref{5.1}{Theorem} together with \legacyref{5.2}{Lemma}.

\hypertarget{legacy:6.3}{}
\begin{legacytheorem}{Theorem S6.3}{Exact learning on \(d\) atoms}

Let \(X=[d]\) and \(H=\{-1,+1\}^{[d]}\). Set
\[
\tau=\frac{\alpha^2}{64d},
\]
\[
M=
\left\lceil
\frac{512d/\rho^2+8(d\log2+\log(4/\beta))}{\alpha^2}
\right\rceil.
\tag{S6.15}\label{legacy:eq:6.15}
\]
Use the gate of \legacyref{3.2}{Lemma} and output \(+1\) at every SMALL point. This is a proper \(\rho\)-pointwise replicable agnostic \((\alpha,\beta)\)-learner.

\end{legacytheorem}
\begin{proof}

Because \(M\tau\ge8/\rho^2\), \legacyref{3.2}{Lemma} gives pointwise disagreement at most
\[
\frac{\sqrt2}{\sqrt{M\tau}}\le\frac\rho2<\rho.
\]

Let \(q\) be the output, and compare it with an empirical majority labeling. A gated point agrees with empirical majority. A SMALL point can disagree only when \(C_x<0\), in which case
\[
|C_x|<2a\sqrt{N_x}.
\]
At \(N_x=0\), its empirical excess contribution is exactly zero. Consequently,
\[
\widehat e(q)-\widehat{\operatorname{OPT}}
\le\frac{2a}{M}\sum_x\sqrt{N_x}
\le2\sqrt{d\tau}=\frac\alpha4.
\tag{S6.16}\label{legacy:eq:6.16}
\]
The second inequality uses \(\sum_xN_x=M\).

For every fixed \(h\), Hoeffding bounds
\(\Pr(|\widehat e(h)-e(h)|>\alpha/4)\) by
\(2e^{-M\alpha^2/8}\). A union bound over \(2^d\) hypotheses gives
\[
\Pr\left(\sup_h|\widehat e(h)-e(h)|>\alpha/4\right)
\le2^{d+1}e^{-M\alpha^2/8}\le\frac\beta2.
\]
On its complement, \eqref{legacy:eq:6.16} implies
\[
e(q)\le\operatorname{OPT}+3\alpha/4.
\]
The output belongs to the full class, so it is proper. The lower bounds already proved give \eqref{legacy:eq:2.12}. 

\end{proof}

\legacyheading{6.3}{A realizable corollary}

\par\noindent\textbf{Role and dependencies.} This proves Theorem~\ref{thm:retained-realizable}. It keeps the gate of \legacyref{3.2}{Lemma} and the stability argument of \legacyref{4.1}{Theorem}, while deriving a consistent-learning and Bernstein accuracy bound.

\hypertarget{legacy:6.4}{}
\begin{legacytheorem}{Corollary S6.4}{Realizable upper bound with centered Bernstein rounding}

Suppose \(\operatorname{OPT}=0\). The gate-and-round construction is accurate with
\[
T=\lceil\rho^{-2}\rceil,\qquad
L=\log(4/\beta),
\]
\[
\tau_r=\frac{\alpha}{48L},
\qquad
K_r=\log\frac{768L}{\alpha\beta},
\]
\[
M_r=
\left\lceil
\frac{3072L}{\alpha}(\rho^{-2}+K_r)
\right\rceil,
\tag{S6.17}\label{legacy:eq:6.17}
\]
and independent, everywhere-defined consistent ERM blocks of size
\[
n_r=
\left\lceil
\frac{128}{\alpha}
\left(d\log\frac{128}{\alpha}+\log\frac{32T}{\beta}\right)
\right\rceil.
\tag{S6.18}\label{legacy:eq:6.18}
\]
It is pointwise replicable for every distribution, and on distributions with OPT zero its accuracy failure is at most \(19\beta/36\). Its sample count is
\[
O\!\left(
\frac{d\log(1/\alpha)+\log(1/(\beta\rho))}
{\rho^2\alpha}
+
\frac{\log(1/\beta)}{\alpha}
\log\frac{\log(1/\beta)}{\alpha\beta}
\right).
\tag{S6.19}\label{legacy:eq:6.19}
\]
\end{legacytheorem}
\begin{proof}

First, OPT zero implies noiseless labels at every positive-mass atom. Otherwise the Bayes risk itself would be positive. Moreover, every finite set of positive-mass points admits a hypothesis agreeing with all their noiseless labels: if not, every hypothesis would incur risk at least the least mass in that finite set. Thus a consistent hypothesis exists on every sample almost surely, even when no single hypothesis attains OPT on the entire domain. Define the learner by ERM on every sample, including samples outside this almost-sure event.

We record the required consistent-learning bound. If \(n\varepsilon\ge8\), double sampling gives
\[
\Pr(\exists h:P_n\ell_h=0,\ P\ell_h>\varepsilon)
\le2\Pi_H(2n)e^{-\varepsilon n/4}.
\tag{S6.20}\label{legacy:eq:6.20}
\]
Indeed, for a witness of true risk \(p>\varepsilon\), an independent ghost sample has at least \(n\varepsilon/2\) errors with probability at least \(1/2\), by Chebyshev:
\[
\Pr(\operatorname{Bin}(n,p)<np/2)\le4/(np)\le1/2.
\]
Condition on the pooled \(2n\) observations and randomize their split. For a trace with \(k\ge n\varepsilon/2\) errors, the probability that all errors fall in the second half is at most
\[
2^{-k}\le e^{-k/2}\le e^{-n\varepsilon/4}.
\]
There are at most \(\Pi_H(2n)\) traces. This proves \eqref{legacy:eq:6.20}.

Take \(\varepsilon=\alpha/2\). To check \eqref{legacy:eq:6.18}, put
\[
A=d\log(128/\alpha)+q,\qquad q=\log(32T/\beta).
\]
Then \(128A/\alpha\le n_r\le129A/\alpha\), and the Sauer bound gives
\[
d\log(2en_r/d)\le3A.
\]
For example,
\[
\log(2en_r/d)
\le\log(258e/\alpha)+\log(A/d)
\le2\log(128/\alpha)+A/d.
\]
Meanwhile \(\alpha n_r/8\ge16A\). Hence \eqref{legacy:eq:6.20} is less than \(\beta/(4T)\). With probability at least \(1-\beta/4\), every base error is at most \(\alpha/2\).

Noiseless labels make every non-SMALL gate sign correct. The dangerous-atom failure is at most \(\beta/36\), by \legacyref{3.2}{Lemma} with the parameters in \eqref{legacy:eq:6.17}. On its complement, every remaining atom satisfies
\[
w_x\le9\tau_r.
\]
Condition on all non-rounding randomness. Let
\[
J_x=w_x\mathbf1_{\{q_x\text{ is wrong}\}},\qquad
\mu_R=\sum_{x\in R}\mathbb E[J_x\mid\mathcal F].
\]
On the base-good event, \(\mu_R\le\alpha/2\). Put
\[
X_x=J_x-\mathbb E[J_x\mid\mathcal F],
\qquad b_0=9\tau_r.
\]
Then
\[
\mathbb E[X_x\mid\mathcal F]=0,\quad
|X_x|\le b_0,\quad
\sigma^2:=\sum_x\mathbb E[X_x^2\mid\mathcal F]
\le b_0\mu_R\le b_0\alpha/2.
\]

For \(0\le\lambda<3/b_0\), the exponential series and
\(k!\ge2\cdot3^{k-2}\), \(k\ge2\), imply
\[
\mathbb E[e^{\lambda X_x}\mid\mathcal F]
\le
\exp\!\left(
\frac{\lambda^2\mathbb E[X_x^2\mid\mathcal F]}
{2(1-\lambda b_0/3)}
\right).
\]
Independence gives the corresponding bound for finite sums. Since
\(\sum_x|X_x|\le\sum_xw_x\le1\), dominated convergence extends it to the countable sum. For \(\sigma^2>0\), choose
\(\lambda=t/(\sigma^2+b_0t/3)<3/b_0\); the zero-variance case is deterministic. Substitution gives
\[
\Pr\left(\sum_xX_x>t\mid\mathcal F\right)
\le
\exp\!\left(
-\frac{t^2}{2(\sigma^2+b_0t/3)}
\right).
\]
At \(t=\alpha/2\), this is at most
\[
\exp\!\left(-\frac{3\alpha}{16b_0}\right)
=
\exp\!\left(-\frac{\alpha}{48\tau_r}\right)
=\frac\beta4.
\]
The total failure is therefore
\[
\frac\beta4+\frac\beta{36}+\frac\beta4
=\frac{19}{36}\beta.
\]
Replicability uses only the unconditional argument of \legacyref{4.1}{Theorem}, which does not require realizability. Expanding \eqref{legacy:eq:6.17}--\eqref{legacy:eq:6.18} proves \eqref{legacy:eq:6.19}. 

\end{proof}

\legacyheading{6.4}{Remarks on the earlier paper and algorithm-specific barriers}

\par\noindent\textbf{Role and dependencies.} This checks the scope of two literal statements in version 2 of \cite{HIY}. The explicit examples explain the conventions used in \legacyref{3.1}{Lemma} and why the realizable upper bound is not a universal matching lower bound.

\hypertarget{legacy:6.5}{}
\begin{legacydefinition}{Remark S6.5}{Two statement/proof qualifications in the earlier paper}

\textbf{(a) The literal dimension-one realizable lower-bound statement.}
Let \(X=\{x\}\) and \(H=\{+1,-1\}\). Take
\[
m=\left\lceil\frac2{\rho^2}\right\rceil,
\qquad
\widehat\mu=\frac1m\sum_{j=1}^mY_j,
\qquad
v\sim\operatorname{Unif}(-1,1),
\]
and output \(+1\) iff \(v\le\widehat\mu\). For every noisy distribution with mean \(\mu\),
\[
\Pr(f_1(x)\ne f_2(x))
=\frac12\mathbb E|\widehat\mu_1-\widehat\mu_2|
\le\sqrt{\frac{1-\mu^2}{2m}}
\le\frac\rho2.
\]
On either realizable distribution, \(\widehat\mu=\pm1\), so the learner has zero error. Its sample count is independent of \(\alpha\).

Thus the literal realizable branch of \cite{HIY}, Theorem 1.5, asserted for every class of VC dimension \(d\), cannot hold at \(d=1\) with an unavoidable \(\alpha^{-1}\) factor. According to the literal text of the proof of \cite{HIY}, Theorem 3.10 (version 2), that argument assumes that \(H\) has VC dimension \(d+1\) with \(d\) odd and uses \(d+1\) shattered elements; the statement of Theorem 1.5 carries no such restriction.

A rational instance is \(\rho=1/20\), \(m=800\), with disagreement at most \(1/40\). In contrast, ``copy the label of one observation'' is not a valid counterexample: under a fair noisy label its disagreement is \(1/2\).

This observation does not affect the agnostic lower bound, which \legacyref{5.1}{Theorem} proves also for \(d=1\).

\textbf{(b) The moment-generating-function step.}
According to the literal text of the proof of \cite{HIY}, Theorem 3.6 (version 2, Appendix B.2), Hoeffding's exponential-moment bound is applied there in a pointwise form. Such a pointwise bound is false. For a fair \(Z\in\{-1,+1\}\), the range length is \(2\), but on \(Z=1\),
\[
e^{Z}=e>e^{1/2}.
\]
The valid inequality is \(\mathbb Ee^{\lambda Z}\le e^{\lambda^2/2}\), with expectation. The literal two-sided statement of \cite{HIY} Theorem 3.6, \(\Pr(|S-\mu|>t)<e^{-\Omega(t^2/\nu)}\) for every \(t>0\) without a prefactor, also fails as written: take \(X_1\in\{0,1/20\}\) with probability \(1/2\) each, all other summands zero, \(b_1=\nu=1/20\) and \(t=1/80\); then \(\mu=1/40\) and \(\Pr(|S-\mu|>t)=1\), while the right side is below \(1\) for any positive constant in the exponent. The correct two-sided form is \(\Pr(|S-\mu|>t)\le\min\{1,2e^{-2t^2/\sum_ib_i^2}\}\). \legacyref{3.1}{Lemma} proves the needed centered, countable one-sided version directly, rather than invoking that proof; no result of this paper depends on the two-sided statement.

\end{legacydefinition}
\par\noindent\textbf{Role and dependencies.} These bounds concern the specified sign-identification, rounding, and gating procedures. They use the score argument of \legacyref{5.1}{Theorem} and the gate of \legacyref{3.2}{Lemma} to explain the alternatives in Section~\ref{sec:retained}.

\hypertarget{legacy:6.6}{}
\begin{legacydefinition}{Remark S6.6}{Three algorithm-specific barriers}

\textbf{(a) An unnecessary atomwise sign-identification interface.}

We quantify atomwise sign identification using the explicit \(3/4\)-success normalization below. Algorithm 1 and Proposition 3.3 of \cite{HIY} describe a related interface up to constant factors; the calculation here specifies the distribution family and constants directly.

Fix an atom \(x_\star\), a mass \(w\in(0,1]\), and a distribution \(Q\), independent of \(t\), supported outside \(x_\star\). Consider only the family
\[
D_t(x_\star,+1)=\frac{w(1+t)}2,\qquad D_t(x_\star,-1)=\frac{w(1-t)}2,\qquad D_t\!\restriction_{\{X\ne x_\star\}}=(1-w)Q,\qquad t\in[-\alpha,\alpha].
\]
The interface and its random-seed law do not depend on \(t\). Suppose the interface must identify the sign of the conditional mean at \(x_\star\) at the endpoints \(t=\pm\alpha\), with success at least \(3/4\) at each endpoint, and must be \(\rho\)-pointwise replicable throughout this family. Since only the labels at \(x_\star\) depend on \(t\), the likelihood-score variance is
\[
\frac{mw}{1-t^2}.
\]
For the potential \(\Psi\) used in \legacyref{5.1}{Theorem}, endpoint success gives expected potential difference at least
\[
1-\frac94\cdot\frac14=\frac7{16}.
\]
The derivative calculation gives an upper bound
\(3\rho\arcsin(\alpha)\sqrt{mw}\). Therefore
\[
m\ge\frac{49}{2304\rho^2w\,\arcsin^2\alpha}.
\tag{S6.21}\label{legacy:eq:6.21}
\]
For \(w=\alpha^2\), this is \(\Omega(\rho^{-2}\alpha^{-4})\).

This interface can be unnecessary for learning. Take two atoms \(x_0,x_1\), with
\[
w_{x_1}=\alpha^2,\quad y_{x_1}=t\in[-\alpha,\alpha],
\qquad
w_{x_0}=1-\alpha^2,\quad y_{x_0}=1,
\]
and let \(H\) contain only the two constant functions. The constant \(+1\) is optimal in \(H\) throughout the family and requires no sign identification.

For a rational endpoint instance, take \(\alpha=1/20\), \(w_{x_1}=1/400\), \(t=\pm1/20\). The two label masses at \(x_1\) are \(21/16000\) and \(19/16000\), interchanged between endpoints; \(D(x_0,+1)=399/400\).

\textbf{(b) The independent-rounding tail scale is tight.}

Take \(N\) uniform atoms, all true labels \(+1\), and average two complementary hypotheses equally. Then \(p_x=0\), independent rounding has independent fair signs, and
\[
e(q)=\frac1N\operatorname{Bin}(N,1/2),\qquad B_R=\frac1N.
\]
For
\[
0<\alpha<1/10,\qquad N\alpha^2\ge1,
\]
one has
\[
\Pr(e(q)>1/2+\alpha)\ge\frac34e^{-20N\alpha^2}.
\tag{S6.22}\label{legacy:eq:6.22}
\]

To prove this, let \(P=\operatorname{Bern}(1/2)^{\otimes N}\) and
\(Q=\operatorname{Bern}(1/2+2\alpha)^{\otimes N}\). Under \(Q\), the event
\[
E=\left\{\left|\frac1N\sum_iZ_i-(1/2+2\alpha)\right|<\alpha\right\}
\]
has probability at least \(3/4\) by Chebyshev, and implies the desired strict upper tail. On \(E\), putting \(u=4\alpha<2/5\),
\[
\frac1N\log\frac{dQ}{dP}
\le\frac12\log(1-u^2)
+\frac{3u}{4}\log\frac{1+u}{1-u}.
\]
The power series has first term \(u^2\); every remaining coefficient is at most \(1/2\). Hence
\[
\frac1N\log\frac{dQ}{dP}
\le u^2+\frac{u^4}{2(1-u^2)}
\le\left(16+\frac{32}{21}\right)\alpha^2
<20\alpha^2.
\]
Changing measure proves \eqref{legacy:eq:6.22}.

The condition \(N\alpha^2\ge1\) cannot be dropped. At \(N=2\), \(\alpha=1/20\), the true tail probability is \(1/4\), whereas the condition-free proposed right side would be \((3/4)e^{-1/10}>1/4\).

This is a barrier for independent rounding. For even \(N\), a correlated uniform choice of exactly \(N/2\) negative signs has the same individual fair marginals and deterministic risk \(1/2\).

\textbf{(c) A logarithmic accuracy cost intrinsic to the retained gate implementation.}

Let \(q\) be an integer, \(X=[q^2]\), and
\[
D(i,+1)=\frac3{4q^2},\qquad
D(i,-1)=\frac1{4q^2}.
\]
Let \(H\) be the two constant functions, and take
\[
\alpha=1/q,\qquad \beta=\rho=1/20.
\]
Use the gate scale from \legacyref{4.1}{Theorem}, so
\[
\tau=\frac1{18q^2\log80}.
\]
Any implementation that forces output \(-1\) whenever TEST is negative fails with probability at least \(31/32\) whenever
\[
q^2\le M\le\frac{q^2\log q}{32}.
\tag{S6.23}\label{legacy:eq:6.23}
\]
The interval is empty when \(\log q<32\).

Here is a fixed-sample occupancy proof. Put
\[
\lambda=M/q^2,\qquad k=\lceil\lambda\rceil,
\]
and let \(Z\) count atoms receiving exactly \(k\) observations, all negatively labeled. Since
\[
k\ge\lambda>4a^2,
\]
every such atom is forced negative regardless of the shared thresholds. For one atom, with \(u=q^{-2}\),
\[
p=\binom Mk(4q^2)^{-k}(1-u)^{M-k}
\ge8^{-k}e^{-2\lambda}
\ge e^{-8\lambda}
\ge q^{-1/4}.
\]
We used \(k\le2\lambda\), \(\binom Mk\ge(M/k)^k\), and
\(\log(1-u)\ge-2u\). Thus \(\mu:=\mathbb EZ\ge q^{7/4}\).

For two distinct atoms, direct multinomial probabilities give
\[
\frac{p_2}{p^2}
\le
\frac{(1-2u)^{M-2k}}{(1-u)^{2M-2k}}
\le e^{4ku}\le e^{8\lambda/q^2}.
\]
Since \(8\lambda/q^2\le1\),
\[
\operatorname{Var}(Z)
\le\mu+\frac{16\lambda}{q^2}\mu^2.
\]
The nonempty parameter interval in \eqref{legacy:eq:6.23} implies \(q\ge e^{32}>32\). In particular \(\mu\ge4q\), and Chebyshev gives
\[
\Pr(Z\le2q)
\le\frac4\mu+\frac{64\lambda}{q^2}
\le4q^{-7/4}+\frac{2\log q}{q^2}
<\frac1{32}.
\]
Every negative prediction adds \(1/(2q^2)\) to the risk above the optimum \(1/4\). Thus \(Z>2q\) implies excess risk greater than \(1/q=\alpha\).

All coordinates of this family are rational. For example, \(q=2^{64}\), \(M=q^2\) satisfies \eqref{legacy:eq:6.23}. The conclusion is restricted to \(M\ge\alpha^{-2}\) and to the stated forced-negative implementation. \legacyref{6.2}{Theorem} proves that it is not a lower bound for general VC-one learning.

\end{legacydefinition}

\legacyheading{6.5}{Connections between the constructions}

The following comparisons explain how the gate, lower bound, proper selection, and
repair steps relate to existing constructions.

\begin{enumerate}[leftmargin=1.6em]
\item \textbf{Difference from \cite{HIY} Algorithm 1.}
The gate is self-normalized:
\[
s_x=C_x/\sqrt{N_x}.
\]
It detects atoms dangerous for rounding concentration through the quantity
\[
w_xy_x^2,
\]
rather than treating mass \(w_x\) alone as the danger criterion. This matches the actual range-square contribution \(w_x^2y_x^2\) in \legacyref{3.1}{Lemma}. The self-normalized gate performs this screening directly.

\item \textbf{From coin testing to the dimension lower bound.}
The coin-testing bound of \cite{ILPS}, Lemma 7.2, uses \(\tau<1/4\), \(\rho<1/16\), and failure parameter \(1/16\), together with replicability at intermediate biases. \legacyref{5.1}{Theorem} gives a multi-coordinate calculation for all positive dimensions, with the stated parameter range, cubic potential, and explicit constant.

\item \textbf{Rare repair: \cite{AAC} and \cite{BGH23}.}
\cite{AAC} uses early stopping followed by fallback to a nonreplicable algorithm in a worst-case sample-size conversion. The finite-class learner of \cite{BGH23} uses risk filtering and a shared random order. In \legacyref{4.3}{Lemma}, a proper-reference risk comparison controls the replacement event at probability \(O(\rho)\), preserving pointwise replicability while improving accuracy confidence.

\item \textbf{Shared truncation: \cite{BGH23}, \S{}5.2.2.}
That reduction uses shared truncation. The proof of \legacyref{5.1}{Theorem} instead evaluates the score directly on the shattered-coordinate family.

\item \textbf{VC-one selection: \cite{BGH23} and \cite{KKMV23}.}
The risk filter is related to the finite-class filter of \cite{BGH23}; shared quantile selection provides the coupling, in the terminology discussed by \cite{KKMV23}. The VC-one structure is described by \cite{BPRS12}, Section 4, Definitions 16 and 17 and the cross-free dual-hypergraph characterization, and by \cite{BD15}, Definition 3 and Theorem 4. \legacyref{6.1}{Lemma} converts this structure into a countable interval order. Shared quantile selection then gives proper pointwise replicable agnostic learning at the exact rate \eqref{legacy:eq:2.11}.
\end{enumerate}
\textbf{6. SOFT-VC and the clipped-shift estimate.}
SOFT-VC combines an average of proper ERMs with a signed empirical histogram, a clipped linear shift, and independent coordinatewise rounding.

\legacyref{7.1}{Lemma} derives the clipped-shift estimate from Hoeffding's exponential-moment bound by completing the square and checking the clipping endpoints. Its countable product form controls the risk of the rounded predictor and supports the SOFT-VC guarantee in \legacyref{7.4}{Theorem}.

\textbf{7. Rare repair and success-probability boosting.}
For comparison, Banihashem et al.~\cite{BBEGH}, Theorem 5 (formally Theorem 12 and Remark 5), boost confidence by running a replicable algorithm at failure probability of order \(\rho\), testing its output, and using a nonreplicable fallback. Their sample count has the form \(n_{\rm replic}(\rho/4,\alpha/2,\rho/4)+n_{\rm tester}(\alpha/2,\min\{\rho,\beta\}/4)+n_{\rm nonreplic}(\alpha,\beta/2)\). Their tester accepts legal solutions at the stronger accuracy and rejects illegal ones at the weaker accuracy. Our agnostic pointwise repair in \legacyref{4.3}{Lemma} and \legacyref{7.6}{Theorem} tests accuracy by comparing the candidate with a proper reference, without access to \(\operatorname{OPT}\). Its analysis bounds the switch probability and disagreement at each fixed coordinate. These guarantees were established in the September version~\cite{Zhang2026}. With the base learner of Proposition~\ref{prop:base}, this repair attains the optimal rate for all three parameters.

Full bibliographic entries are given in the References.

\subsection{SOFT-VC, explicit confidence costs, and counterexamples}\label{legacy:section:7}

For a fractional prediction \(p:X\to[-1,1]\), define
\[
\mathcal E_D(p)
=\sum_{x\in X}w_x\frac{1-p_xy_x}{2},
\qquad
\zeta_x=\frac{w_xy_x}{2},
\qquad
\Pi(t)=\max\{-1,\min\{1,t\}\}.
\tag{S7.1}\label{legacy:eq:7.1}
\]
The fractional risk is used only in the analysis. Every final output below is binary-valued.

We use the countable proper core constructed in \legacyref{6.1}{Lemma}. That construction preserves every finite trace, as well as the infimum of the population risk. Fix an enumeration of this core and select the first empirical risk minimizer in that enumeration. Thus all ERMs below are proper, deterministic, and defined on every sample. Their square-exponential moments satisfy \legacyref{3.4}{Lemma}, with
\[
K_E=2^{24}.
\]
This notation is distinct from the gate logarithm \(K\) in Appendix~\ref{legacy:section:4}.

\legacyheading{7.1}{A one-coordinate clipped-shift bound}

\par\noindent\textbf{Role and dependencies.} This is the local risk inequality behind the soft-shift module. Its product form is \legacyref{7.2}{Corollary}, which is used in \legacyref{7.4}{Theorem}.

\hypertarget{legacy:7.1}{}
\begin{legacytheorem}{Lemma S7.1}{One-coordinate exponential moment for a clipped shift}

Let
\[
p\in[-1,1],\qquad z,\widehat z\in\mathbb R,\qquad \lambda>0,
\]
and set
\[
q=\Pi(p+2\lambda\widehat z).
\]
If \(F\in\{-1,+1\}\) has mean \(q\), then
\[
\boxed{
\log\mathbb E\exp\!\bigl(\lambda z(p-F)\bigr)
\le
\frac23\lambda^2(\widehat z-z)^2.
}
\tag{S7.2}\label{legacy:eq:7.2}
\]
The estimate holds with the numerical constant \(2/3\).

\end{legacytheorem}
\begin{proof}

Put \(\Delta=\widehat z-z\).

If \(z=0\), the logarithmic exponential moment is \(0\), and the assertion follows. Henceforth suppose \(z\ne0\).

\textbf{Interior case: \(-1<q<1\).} Here
\(q=p+2\lambda\widehat z\). The random variable \(F\) has range length \(2\). The centered Hoeffding exponential-moment bound in \legacyref{3.1}{Lemma} gives
\[
\begin{aligned}
\log\mathbb E e^{\lambda z(p-F)}
&=\lambda z(p-q)
+\log\mathbb E e^{\lambda z(q-F)}\\
&\le -2\lambda^2z\widehat z+\frac{\lambda^2z^2}{2}\\
&=\lambda^2\left(-\frac32z^2-2z\Delta\right)\\
&=-\frac32\lambda^2
\left(z+\frac23\Delta\right)^2
+\frac23\lambda^2\Delta^2\\
&\le\frac23\lambda^2\Delta^2.
\end{aligned}
\]

\textbf{Correct endpoint: \(q\in\{-1,+1\}\) and \(qz>0\).} A binary random variable with mean \(q\in\{-1,+1\}\) equals \(q\) almost surely. Since \(p\in[-1,1]\),
\[
z(p-q)\le0.
\]
Thus the logarithmic exponential moment is nonpositive.

\textbf{Incorrect endpoint: \(q\in\{-1,+1\}\) and \(qz<0\).} If \(q=1\) and \(z<0\), clipping implies
\[
2\lambda\widehat z\ge1-p\ge0.
\]
If \(q=-1\) and \(z>0\), clipping implies
\[
2\lambda\widehat z\le-1-p\le0.
\]
In both cases,
\[
z\widehat z\le0,\qquad
|p-q|\le2\lambda|\widehat z|.
\]
Since \(F=q\) almost surely,
\[
\begin{aligned}
\log\mathbb E e^{\lambda z(p-F)}
&=\lambda |z|\,|p-q|\\
&\le2\lambda^2|z|\,|\widehat z|\\
&\le\frac{\lambda^2}{2}(|z|+|\widehat z|)^2\\
&=\frac{\lambda^2}{2}(\widehat z-z)^2\\
&\le\frac23\lambda^2(\widehat z-z)^2.
\end{aligned}
\]
The third inequality is equivalent to
\((|z|-|\widehat z|)^2\ge0\).

These cases include equality at either clipping boundary and exhaust all possibilities. 

\end{proof}
\legacyheading{7.2}{Countably many coordinates}

\par\noindent\textbf{Role and dependencies.} This applies \legacyref{7.1}{Lemma} to countably many coordinates and gives the conditional tail bound needed by \legacyref{7.4}{Theorem}.

\hypertarget{legacy:7.2}{}
\begin{legacytheorem}{Corollary S7.2}{Conditional countable-product bound}

Fix a distribution \(D\) on a countable domain and a number \(\lambda>0\). Let \(\mathcal F\) be a sigma-field with respect to which \(p_x\in[-1,1]\) and \(\widehat\zeta_x\in\mathbb R\) are measurable for every \(x\). Let \((U_x)_{x\in X}\) be independent uniforms on \([-1,1]\), jointly independent of \(\mathcal F\). Set
\[
q_x=\Pi(p_x+2\lambda\widehat\zeta_x),
\qquad
F_x=
\begin{cases}
+1,&U_x\le q_x,\\
-1,&U_x>q_x.
\end{cases}
\]

Then
\[
\boxed{
\mathbb E\!\left[
e^{\lambda(e(F)-\mathcal E_D(p))}
\mid\mathcal F
\right]
\le
\exp\!\left(
\frac23\lambda^2
\|\widehat\zeta-\zeta\|_2^2
\right).
}
\tag{S7.3}\label{legacy:eq:7.3}
\]
If the squared norm is infinite, the right side is interpreted as \(+\infty\). In particular, whenever that norm is finite, for every \(t>0\),
\[
\Pr\!\left(
e(F)>\mathcal E_D(p)+t
\mid\mathcal F
\right)
\le
\exp\!\left(
-\lambda t+
\frac23\lambda^2\|\widehat\zeta-\zeta\|_2^2
\right).
\tag{S7.4}\label{legacy:eq:7.4}
\]

\end{legacytheorem}
\begin{proof}

The risk identity is
\[
e(F)-\mathcal E_D(p)
=\sum_x\zeta_x(p_x-F_x).
\]
Conditioned on \(\mathcal F\), the variables \(F_x\) are independent and have respective means \(q_x\). For every finite set \(A\subseteq X\), \legacyref{7.1}{Lemma} therefore gives
\[
\begin{aligned}
\mathbb E\!\left[
\exp\!\left(\lambda\sum_{x\in A}\zeta_x(p_x-F_x)\right)
\mid\mathcal F
\right]
&=\prod_{x\in A}
\mathbb E\!\left[e^{\lambda\zeta_x(p_x-F_x)}
\mid\mathcal F\right]\\
&\le
\exp\!\left(
\frac23\lambda^2
\sum_{x\in A}(\widehat\zeta_x-\zeta_x)^2
\right).
\end{aligned}
\]

Moreover,
\[
\sum_x|\zeta_x|\le\frac12,
\qquad
\sum_x|\zeta_x(p_x-F_x)|\le1.
\]
Thus the risk-difference series converges absolutely, and every exponential of a finite partial sum is bounded by \(e^\lambda\). Conditional dominated convergence on the left, and convergence of the nonnegative squared-error series on the right, prove \eqref{legacy:eq:7.3}.

Applying conditional Markov inequality to \eqref{legacy:eq:7.3} proves \eqref{legacy:eq:7.4}. The independence used here concerns the remaining rounding variables; it does not require independence among the histogram coordinates or among the fractional predictions. 

\end{proof}
\legacyheading{7.3}{Concentration of a signed empirical histogram}

Given \(M\ge1\) independent observations, define
\[
\widehat\zeta_x
=\frac1{2M}\sum_{i=1}^{M}
Y_i\mathbf1_{\{X_i=x\}}.
\tag{S7.5}\label{legacy:eq:7.5}
\]
In particular, \(\widehat\zeta_x=0\) at every unobserved point.

\par\noindent\textbf{Role and dependencies.} This supplies the signed-histogram accuracy estimate used together with \legacyref{7.2}{Corollary} in \legacyref{7.4}{Theorem}.

\hypertarget{legacy:7.3}{}
\begin{legacytheorem}{Lemma S7.3}{Dimension-free \(\ell_2\) concentration of the signed histogram}

For every distribution \(D\) on a countable domain, every integer \(M\ge1\), and every \(L>0\),
\[
\Pr\!\left(
\|\widehat\zeta-\zeta\|_2
>
\frac{1+\sqrt{2L}}{2\sqrt M}
\right)
\le e^{-L}.
\tag{S7.6}\label{legacy:eq:7.6}
\]

\end{legacytheorem}
\begin{proof}

In the separable Hilbert space \(\ell_2(X)\), put
\[
V_i=\frac{Y_i}{2}e_{X_i}.
\]
Then
\[
\|V_i\|_2=\frac12,\qquad
\mathbb EV_i=\zeta,\qquad
\widehat\zeta=\frac1M\sum_iV_i.
\]
Independence and centering make the cross terms vanish, giving
\[
\begin{aligned}
\mathbb E\|\widehat\zeta-\zeta\|_2^2
&=\frac1M\mathbb E\|V_1-\zeta\|_2^2\\
&=\frac1M\left(\frac14-\|\zeta\|_2^2\right)
\le\frac1{4M}.
\end{aligned}
\]
These identities can also be obtained by summing the coordinate variances; Tonelli's theorem applies to the squared coordinates. Cauchy--Schwarz yields
\[
\mathbb E\|\widehat\zeta-\zeta\|_2
\le\frac1{2\sqrt M}.
\]

Let \(G=\|\widehat\zeta-\zeta\|_2\). Replacing one observation changes its associated \(V_i\) by a vector of norm at most \(1\). The triangle inequality therefore changes \(G\) by at most \(1/M\).

For completeness, expose the observations successively and let \(D_i\) be the increments of the Doob martingale for \(G\). Conditional on the preceding observations, \(D_i\) has mean zero and conditional range length at most \(1/M\). Applying the Hoeffding exponential-moment bound conditionally and iterating gives
\[
\mathbb E e^{\theta(G-\mathbb EG)}
\le e^{\theta^2/(8M)}.
\]
For \(\theta>0\), exponential Markov inequality gives
\[
\Pr(G-\mathbb EG>t)
\le e^{-\theta t+\theta^2/(8M)}.
\]
The choice \(\theta=4Mt\) yields \(e^{-2Mt^2}\). Taking
\(t=\sqrt{L/(2M)}\) proves \eqref{legacy:eq:7.6}. 

\end{proof}
\legacyheading{7.4}{The SOFT-VC construction}

\par\noindent\textbf{Role and dependencies.} This proves the direct part of Theorem~\ref{thm:retained-soft}. Its inputs are the ERM moments of \legacyref{3.4}{Lemma}, the proper core of \legacyref{6.1}{Lemma}, and \legacyref{7.1}{Lemma}--\legacyref{7.3}{Lemma}.

\hypertarget{legacy:7.4}{}
\begin{legacytheorem}{Theorem S7.4}{SOFT-VC}

Let \(X\) be countable, let \(\operatorname{VC}(H)=d\ge1\), and let
\(\alpha,\beta,\rho\in(0,1/10)\). Set
\[
L=\log\frac4\beta,\qquad
T=\left\lceil\frac{32}{\rho^2}\right\rceil,
\qquad
n=\left\lceil
\frac{32K_E}{\alpha^2}
\left(d+\log21+\frac LT\right)
\right\rceil,
\tag{S7.7}\label{legacy:eq:7.7}
\]
\[
\lambda=\frac{4L}{\alpha},
\qquad
M=\left\lceil\frac{32\lambda^2}{\rho^2}\right\rceil
=\left\lceil\frac{512L^2}{\rho^2\alpha^2}\right\rceil.
\tag{S7.8}\label{legacy:eq:7.8}
\]

Train deterministic proper ERMs \(h_1,\ldots,h_T\) on independent blocks of size \(n\). Use an independent block of size \(M\) to construct \eqref{legacy:eq:7.5}. Define
\[
p_x=\frac1T\sum_{j=1}^{T}h_j(x),
\qquad
q_x=\Pi(p_x+2\lambda\widehat\zeta_x),
\]
\[
F(x)=
\begin{cases}
+1,&U_x\le q_x,\\
-1,&U_x>q_x,
\end{cases}
\tag{S7.9}\label{legacy:eq:7.9}
\]
where the \(U_x\) are independent uniforms on \([-1,1]\), independent of all data and shared between executions.

This learner uses
\[
\boxed{
Tn+M
\le
\frac{C_{\mathrm{soft}}}{\alpha^2}
\left[
\frac{d+L^2}{\rho^2}+L
\right],
\qquad C_{\mathrm{soft}}=2^{37},
}
\tag{S7.10}\label{legacy:eq:7.10}
\]
and satisfies
\[
\Pr(e(F)>\operatorname{OPT}+\alpha)<\frac{3\beta}{4},
\qquad
\forall x\in X:\quad
\Pr(F_1(x)\ne F_2(x))\le\frac\rho4.
\tag{S7.11}\label{legacy:eq:7.11}
\]
In particular, it is an agnostic \((\alpha,\beta)\)-learner and is \(\rho\)-pointwise replicable. The final predictor may be improper.

\end{legacytheorem}
\begin{proof}[Proof of accuracy]

Put
\[
Z_j=e(h_j)-\operatorname{OPT}.
\]
Properness implies \(Z_j\ge0\) on every realization. \legacyref{3.4}{Lemma} and independence give
\[
\mathbb E\exp\!\left(
\frac n{2K_E}\sum_{j=1}^{T}Z_j^2
\right)
\le (21e^d)^T.
\]
If \(T^{-1}\sum_jZ_j>\alpha/4\), then Cauchy--Schwarz gives
\[
\sum_jZ_j^2>\frac{T\alpha^2}{16}.
\]
Consequently,
\[
\begin{aligned}
\Pr\!\left(\frac1T\sum_jZ_j>\frac\alpha4\right)
&\le
\exp\!\left(
T(d+\log21)-\frac{nT\alpha^2}{32K_E}
\right)\\
&\le e^{-L}.
\end{aligned}
\]
Since fractional risk is linear,
\[
\mathcal E_D(p)=\frac1T\sum_je(h_j).
\]
Thus, except with probability \(e^{-L}\),
\[
\mathcal E_D(p)\le\operatorname{OPT}+\frac\alpha4.
\tag{S7.12}\label{legacy:eq:7.12}
\]

\legacyref{7.3}{Lemma} supplies, except with probability \(e^{-L}\),
\[
\|\widehat\zeta-\zeta\|_2^2
\le\frac{(1+\sqrt{2L})^2}{4M}
\le\frac{6L}{4M}.
\tag{S7.13}\label{legacy:eq:7.13}
\]
Here \(L>1\), and
\[
(1+\sqrt{2L})^2
\le(3+2\sqrt2)L<6L.
\]

Let \(\mathcal F\) contain all training and histogram observations, but not the rounding variables. On the event in \eqref{legacy:eq:7.13}, \legacyref{7.2}{Corollary} gives
\[
\begin{aligned}
\Pr\!\left(
e(F)>\mathcal E_D(p)+\frac\alpha2
\mid\mathcal F
\right)
&\le
\exp\!\left(
-\frac{\lambda\alpha}{2}
+\frac23\lambda^2\frac{6L}{4M}
\right)\\
&\le
\exp\!\left(-2L+\frac{\rho^2L}{32}\right)
<e^{-L}.
\end{aligned}
\tag{S7.14}\label{legacy:eq:7.14}
\]
The last inequality uses \(\rho^2/32<1\).

When neither \eqref{legacy:eq:7.12}, \eqref{legacy:eq:7.13}, nor the conditional rounding conclusion fails, the final risk is at most
\[
\operatorname{OPT}+\frac\alpha4+\frac\alpha2
=\operatorname{OPT}+\frac{3\alpha}{4}.
\]
A union bound, with the last probability integrated over \(\mathcal F\), therefore gives
\[
\Pr(e(F)>\operatorname{OPT}+\alpha)
<3e^{-L}
=\frac{3\beta}{4}<\beta.
\]

\end{proof}
\begin{proof}[Proof of pointwise replicability]

Fix an arbitrary \(x\in X\). This argument is unconditional and uses no accuracy event.

Given the two data sets, shared-uniform rounding gives
\[
\Pr_U(F_1(x)\ne F_2(x)\mid\text{data})
=\frac{|q_x^{(1)}-q_x^{(2)}|}{2}.
\]
The clipping map is \(1\)-Lipschitz. Hence
\[
\Pr(F_1(x)\ne F_2(x))
\le
\frac12\mathbb E|p_x^{(1)}-p_x^{(2)}|
+\lambda\mathbb E|
\widehat\zeta_x^{(1)}-\widehat\zeta_x^{(2)}|.
\tag{S7.15}\label{legacy:eq:7.15}
\]

Because the same deterministic ERM rule is used on independent blocks,
\[
\mathbb E(p_x^{(1)}-p_x^{(2)})^2
=\frac{2\operatorname{Var}(h_1(x))}{T}
\le\frac2T.
\]
For the histogram,
\[
\operatorname{Var}(\widehat\zeta_x)
=\frac1M\left(\frac{w_x}{4}-\zeta_x^2\right)
\le\frac{w_x}{4M}.
\]
Cauchy--Schwarz applied to the two independent executions gives
\[
\mathbb E|
\widehat\zeta_x^{(1)}-\widehat\zeta_x^{(2)}|
\le\sqrt{\frac{w_x}{2M}}.
\]
Substitution into \eqref{legacy:eq:7.15} yields
\[
\begin{aligned}
\Pr(F_1(x)\ne F_2(x))
&\le
\frac1{\sqrt{2T}}+\lambda\sqrt{\frac{w_x}{2M}}\\
&\le\frac\rho8+\frac{\rho\sqrt{w_x}}8\\
&\le\frac\rho4<\rho.
\end{aligned}
\tag{S7.16}\label{legacy:eq:7.16}
\]
When \(w_x=0\), the histogram term is identically zero. The base-average bound still applies at that fixed zero-mass point.

\end{proof}
\begin{proof}[Proof of sample count]

The ceilings give
\[
Tn+M
\le
\frac{32K_E}{\alpha^2}
\bigl[T(d+\log21)+L\bigr]
+T+\frac{512L^2}{\rho^2\alpha^2}+1.
\tag{S7.17}\label{legacy:eq:7.17}
\]
Since \(T\le33/\rho^2\), \(\log21<4\), \(d\ge1\), and \(\alpha<1\),
\[
T(d+\log21)\le\frac{165d}{\rho^2},
\qquad
T+1\le\frac{34}{\rho^2}
\le\frac{34d}{\rho^2\alpha^2}.
\]
It follows that
\[
Tn+M
\le\frac1{\alpha^2}
\left[
\frac{(5280K_E+34)d+512L^2}{\rho^2}
+32K_E L
\right].
\]
The three coefficients satisfy
\[
5280K_E+34=88583700514<2^{37},
\qquad
512<2^{37},
\qquad
32K_E=2^{29}<2^{37}.
\]
This proves \eqref{legacy:eq:7.10}. No logarithm of \(1/\alpha\) or \(1/\rho\) is hidden in this calculation. 

\end{proof}
\legacyheading{7.5}{Exact order at fixed confidence}

\par\noindent\textbf{Role and dependencies.} This combines \legacyref{7.4}{Theorem} with \legacyref{5.1}{Theorem} to obtain the fixed-confidence statement in Theorem~\ref{thm:retained-soft}.

\hypertarget{legacy:7.5}{}
\begin{legacytheorem}{Theorem S7.5}{Fixed-confidence exact order}

For every fixed \(\beta_0\in(0,1/10)\), every countable domain, every class \(H\) of VC dimension \(d\ge1\), and every \(\alpha,\rho\in(0,1/10)\),
\[
\boxed{
\kappa\frac d{\rho^2\alpha^2}
\le
m_H(\alpha,\beta_0,\rho)
\le
C_{\beta_0}\frac d{\rho^2\alpha^2}.
}
\tag{S7.18}\label{legacy:eq:7.18}
\]
The same inequalities hold for \(m^*(d,\alpha,\beta_0,\rho)\). One valid choice is
\[
L_0=\log\frac4{\beta_0},
\qquad
C_{\beta_0}=2^{37}(1+L_0^2+L_0).
\tag{S7.19}\label{legacy:eq:7.19}
\]
Thus the constants are independent of \(H,X,d,\alpha,\rho\).

\end{legacytheorem}
\begin{proof}

The lower bound is \legacyref{5.1}{Theorem}.

For the upper bound, \legacyref{7.4}{Theorem} gives
\[
m_H(\alpha,\beta_0,\rho)
\le\frac{2^{37}}{\alpha^2}
\left(\frac{d+L_0^2}{\rho^2}+L_0\right).
\]
Since \(d\ge1\) and \(\rho^2<1\),
\[
\frac{d+L_0^2}{\rho^2}+L_0
\le
(1+L_0^2+L_0)\frac d{\rho^2}.
\]
Taking the supremum over classes proves the assertion for \(m^*\). 

In particular, at fixed confidence there is no general necessity for an additional \(\log(1/\alpha)\) factor as \(\alpha\to0\), nor for an additional \(\log(1/\rho)\) factor as \(\rho\to0\). This conclusion concerns the optimal sample complexity over the permitted learning rules.

\end{proof}
\legacyheading{7.6}{SOFT-VC with rare repair}

\par\noindent\textbf{Role and dependencies.} This proves the repaired and selected soft-shift bound in Theorem~\ref{thm:retained-soft}. It applies the repair principle of \legacyref{4.3}{Lemma} to \legacyref{7.4}{Theorem} and rechecks all block constants.

\hypertarget{legacy:7.6}{}
\begin{legacytheorem}{Theorem S7.6}{Full-parameter bounds from SOFT-VC and repair}

Under the hypotheses of \legacyref{7.4}{Theorem}, run
\[
B=\operatorname{SOFT\text{-}VC}
\left(\frac\alpha4,\frac{\rho}{128},\rho\right).
\tag{S7.20}\label{legacy:eq:7.20}
\]
Independently train a proper reference ERM \(g\), and use a further independent validation block, with
\[
\eta=\min\{\beta/8,\rho/128\},
\]
\[
n_{\rm ref}
=
\left\lceil
\frac{2^{22}}{\alpha^2}
\left(d+\log\frac2\eta\right)
\right\rceil,
\qquad
n_{\rm val}
=
\left\lceil
\frac{128}{\alpha^2}\log\frac4\eta
\right\rceil.
\tag{S7.21}\label{legacy:eq:7.21}
\]
Output
\[
F=
\begin{cases}
B,&\widehat e(B)\le\widehat e(g)+\alpha/2,\\
g,&\widehat e(B)>\widehat e(g)+\alpha/2.
\end{cases}
\tag{S7.22}\label{legacy:eq:7.22}
\]

The repaired learner is an agnostic \((\alpha,\beta)\)-learner, has pointwise disagreement strictly below \(9\rho/32\), and uses at most
\[
\boxed{
N_{\rm rep}
\le
\frac{C_3}{\alpha^2}
\left[
\frac{d+J^2}{\rho^2}+L
\right],
\qquad
J=\log\frac{512}{\rho},
\quad C_3=2^{39}.
}
\tag{S7.23}\label{legacy:eq:7.23}
\]

Choosing, using only the known parameters, between this learner and direct SOFT-VC gives
\[
\boxed{
\frac{\kappa/2}{\alpha^2}
\left(\frac d{\rho^2}+L\right)
\le
m_H(\alpha,\beta,\rho)
\le
\frac{C_3}{\alpha^2}
\left[
\frac d{\rho^2}+L+
\frac{\min\{L^2,J^2\}}{\rho^2}
\right].
}
\tag{S7.24}\label{legacy:eq:7.24}
\]
The same bounds hold for \(m^*\). In particular, for every fixed implicit constant in
\[
\min\{L^2,J^2\}\lesssim d+\rho^2L,
\tag{S7.25}\label{legacy:eq:7.25}
\]
these bounds establish the order
\[
\Theta\!\left(
\frac d{\rho^2\alpha^2}+\frac{L}{\alpha^2}
\right),
\]
with the comparison constants allowed to depend on that implicit constant. Condition \eqref{legacy:eq:7.25} is sufficient for this route's bound to attain the optimal order.

\end{legacytheorem}
\begin{proof}[Proof of the repair interface]

The inner confidence logarithm is
\[
\log\frac4{\rho/128}=J.
\]
\legacyref{7.4}{Theorem} supplies
\[
\Pr\!\left(e(B)>\operatorname{OPT}+\frac\alpha4\right)
<\frac{3\rho}{512}<\frac{\rho}{128},
\]
and, for every fixed \(x\),
\[
\Pr(B_1(x)\ne B_2(x))\le\frac\rho4.
\]

\legacyref{3.3}{Lemma}, with accuracy \(\alpha/8\) and confidence \(\eta\), gives
\[
\Pr\!\left(e(g)>\operatorname{OPT}+\frac\alpha8\right)<\eta.
\]
Condition on the already generated functions \(B,g\), including all randomness used to produce \(B\). The validation observations are still independent. Hoeffding's inequality and a union bound give
\[
\begin{aligned}
\Pr(V_{\rm bad}\mid B,g)
&\le4e^{-n_{\rm val}\alpha^2/128}\\
&\le\eta,
\end{aligned}
\]
where \(V_{\rm bad}\) is the event that either validation risk differs from its true risk by more than \(\alpha/16\).

Define the switching event by the decision branch,
\[
W=\{\widehat e(B)>\widehat e(g)+\alpha/2\}.
\]
On this event the algorithm selects \(g\) through the repair branch.

Properness gives \(e(g)\ge\operatorname{OPT}\) on every realization. If
\(e(B)\le\operatorname{OPT}+\alpha/4\) and validation is accurate, then
\[
\widehat e(B)-\widehat e(g)
\le\frac\alpha4+2\frac\alpha{16}
=\frac{3\alpha}{8}<\frac\alpha2.
\]
Therefore
\[
W\subseteq
\left\{e(B)>\operatorname{OPT}+\frac\alpha4\right\}
\cup V_{\rm bad},
\]
so
\[
\Pr(W)<\frac{\rho}{128}+\eta
\le\frac{\rho}{64}.
\]

For two executions and every fixed \(x\),
\[
\{F_1(x)\ne F_2(x)\}
\subseteq
\{B_1(x)\ne B_2(x)\}\cup W_1\cup W_2.
\]
A union bound gives
\[
\Pr(F_1(x)\ne F_2(x))
<
\frac\rho4+2\frac{\rho}{64}
=\frac{9\rho}{32}<\rho.
\]

For accuracy, suppose that the reference and validation conclusions both hold. If the repair branch is taken, then
\[
e(F)=e(g)\le\operatorname{OPT}+\alpha/8.
\]
If it is not taken, then
\[
\begin{aligned}
e(F)=e(B)
&\le\widehat e(B)+\alpha/16\\
&\le\widehat e(g)+\alpha/2+\alpha/16\\
&\le e(g)+\alpha/2+\alpha/8\\
&\le\operatorname{OPT}+3\alpha/4.
\end{aligned}
\]
Thus the accuracy failure probability is at most
\[
2\eta\le\beta/4<\beta.
\]

\end{proof}
\begin{proof}[Proof of sample count]

The inner construction uses
\[
T=\left\lceil\frac{32}{\rho^2}\right\rceil,
\qquad
\lambda_B=\frac{16J}{\alpha}.
\]
Expanding its actual sample count, rather than treating a fixed-confidence constant as uniform in \(\rho\), gives
\[
N_B
\le
\frac{512K_E}{\alpha^2}
\bigl[T(d+\log21)+J\bigr]
+T+\frac{8192J^2}{\rho^2\alpha^2}+1.
\]
Since \(T\le33/\rho^2\),
\[
N_B
\le
\frac1{\alpha^2}
\left[
\frac{16896K_E(d+\log21)+8192J^2+34}{\rho^2}
+512K_EJ
\right].
\]

The definition of \(\eta\) gives
\[
\log\frac2\eta
=\max\{L+\log4,\ J-\log2\}\le L+J,
\]
\[
\log\frac4\eta
=\max\{L+\log8,\ J\}\le L+J.
\]
Consequently,
\[
n_{\rm ref}+n_{\rm val}
\le
\frac1{\alpha^2}
\left[
2^{22}d+(2^{22}+128)(L+J)
\right]+2.
\]

Now \(J^2>\log21\), \(J>1\), \(d\ge1\), and \(\rho^2<1\). Absorbing \(J\) into \(J^2/\rho^2\), and the two reference/validation ceilings into \(2d/(\rho^2\alpha^2)\), yields
\[
N_{\rm rep}
\le
\frac1{\alpha^2}
\left[
\frac{A_d d+A_sJ^2}{\rho^2}+A_LL
\right],
\]
where
\[
\begin{aligned}
A_d&=16896K_E+2^{22}+36
=283472035876,\\
A_s&=17408K_E+8192+2^{22}+128
=292061978752,\\
A_L&=2^{22}+128
=4194432.
\end{aligned}
\]
Each coefficient is less than
\[
2^{39}=549755813888.
\]
This proves \eqref{legacy:eq:7.23}.

The direct learner has the bound \eqref{legacy:eq:7.10}, with
\(C_{\mathrm{soft}}\le C_3\). Selecting the smaller of the two explicit sample counts therefore gives the upper bound in \eqref{legacy:eq:7.24}. This selection is deterministic in the known parameters and does not inspect the training data.

For the lower bound, \legacyref{5.1}{Theorem} gives
\[
m_H\ge\frac{\kappa}{\alpha^2}\frac d{\rho^2}.
\]
Also \(L\le2b\), so \legacyref{5.2}{Lemma} gives
\[
m_H\ge\frac{3}{50}\frac b{\alpha^2}
\ge\frac{3}{100}\frac L{\alpha^2}
\ge\frac{\kappa L}{\alpha^2}.
\]
Here \(\kappa<3/100\). Averaging these two lower bounds proves the lower side of \eqref{legacy:eq:7.24}. Taking suprema preserves all the bounds.

Finally, \eqref{legacy:eq:7.25} absorbs the only additional numerator in the upper side of \eqref{legacy:eq:7.24}. 

\end{proof}
\legacyheading{7.7}{Gate estimates and agnostic counterexamples}

\par\noindent\textbf{Role and dependencies.} These estimates and counterexamples complete the explanation in Section~\ref{sec:retained}. They refine the gate analysis of \legacyref{3.2}{Lemma} and separate agnostic excess risk from Bayes excess risk.

\hypertarget{legacy:7.7}{}
\begin{legacydefinition}{Remark S7.7}{Variance-proxy control, a forced-gate obstruction, and an agnostic counterexample}

The following observations analyze concentration on the SMALL region, the cost of forced incorrect gate outputs, and the relation between agnostic excess risk and Bayes excess risk. The obstructions concern the gating and rounding rules specified in each construction.

\paragraph{(a) Expected SMALL-region variance proxy and its conditional rounding interface}

Consider the gate in \legacyref{3.2}{Lemma}, with integer \(M\ge1\), \(\tau\in(0,1)\), and
\[
a^2=M\tau\ge1.
\]
Define
\[
B_{\rm small}
=\sum_{x:\operatorname{TEST}(x)=\mathrm{SMALL}}w_x^2y_x^2.
\]
Then
\[
\boxed{\mathbb E B_{\rm small}\le16\tau.}
\tag{S7.26}\label{legacy:eq:7.26}
\]

To prove this, retain the conventions
\[
s_x=0\quad\text{when }N_x=0,\qquad
y_x=0\quad\text{when }w_x=0.
\]
Zero-mass points contribute zero to \(B_{\rm small}\), so fix \(x\) with \(w=w_x>0\). Write \(y=y_x\), \(N=N_x\), \(s=s_x\), and
\[
v=|y|\sqrt{Mw}.
\]
The conditional and unconditional expectations are, respectively,
\[
\mathbb E[s\mid N]=y\sqrt N,
\qquad
\mathbb Es=y\,\mathbb E\sqrt N.
\]
Furthermore,
\[
\begin{aligned}
\mathbb E|\sqrt N-\sqrt{Mw}|
&\le\frac{\mathbb E|N-Mw|}{\sqrt{Mw}}\\
&\le\frac{\sqrt{\operatorname{Var}(N)}}{\sqrt{Mw}}
=\sqrt{1-w}\le1.
\end{aligned}
\]
It follows that
\[
|\mathbb Es|\ge v-1.
\]

If \(v\ge4a\), then SMALL implies \(|s|\le2a\), and
\[
|\mathbb Es|-2a
\ge v-1-2a
\ge v-3a
\ge v/4.
\]
\legacyref{3.2}{Lemma} gives \(\operatorname{Var}(s)\le1\). Chebyshev's inequality therefore gives
\[
\Pr(\operatorname{TEST}(x)=\mathrm{SMALL})\le16/v^2.
\]
Hence
\[
w^2y^2\Pr(\mathrm{SMALL})
=\frac{wv^2}{M}\Pr(\mathrm{SMALL})
\le\frac{16w}{M}
\le\frac{16a^2w}{M}.
\]
If \(v<4a\), the probability bound \(1\) instead gives
\[
w^2y^2\Pr(\mathrm{SMALL})
\le\frac{wv^2}{M}
\le\frac{16a^2w}{M}.
\]
Summing by Tonelli's theorem proves \eqref{legacy:eq:7.26}. No independence across gate coordinates is needed for this summation.

The rounding interface requires a separate independence condition. Suppose the final rounding on the SMALL region uses independent coordinate seeds, jointly independent of the gate data, gate randomness, and base learners. Let \(\mathcal G\) fix all randomness except these final rounding seeds. Let
\[
R=\{x:\operatorname{TEST}(x)=\mathrm{SMALL}\},
\]
and let \(p_x\in[-1,1]\) be any \(\mathcal G\)-measurable rounding means. Write
\[
Z_R=\sum_{x\in R}w_x\frac{1-Q_xy_x}{2},
\]
where the \(Q_x\) are rounded using the independent final seeds.

For every \(\delta_1\in(0,1)\), Markov's inequality and \eqref{legacy:eq:7.26} give
\[
\Pr(B_{\rm small}>16\tau/\delta_1)\le\delta_1.
\]
On the complementary event, the conditional version of \legacyref{3.1}{Lemma} gives, for every \(t>0\),
\[
\boxed{
\Pr\!\left(
Z_R-\mathbb E[Z_R\mid\mathcal G]>t
\right)
\le
\delta_1+
\exp\!\left(-\frac{\delta_1t^2}{8\tau}\right).
}
\tag{S7.27}\label{legacy:eq:7.27}
\]
If \(B_{\rm small}=0\), its conditional fluctuation is zero, as required by the zero-variance convention.

For example, at fixed confidence \(\beta_0\), take
\[
\delta_1=\beta_0/8,\qquad t=\alpha/2,\qquad
\tau\le
\frac{\beta_0\alpha^2}{256\log(8/\beta_0)}.
\]
The right side of \eqref{legacy:eq:7.27} is at most \(\beta_0/4\). Thus this particular fluctuation can be controlled with \(\tau=c_{\beta_0}\alpha^2\), without a \(\log(1/\alpha)\) factor.

This statement controls only the SMALL-region rounding fluctuation. It neither certifies the signs forced by the gate nor proves accuracy of the complete hard-gate learner. In particular, the earlier numerical choice \(18\) in the gate scale cannot simply be retained without resetting the desired failure budget.

\paragraph{(b) A witnessed obstruction for a specified forced-negative gate family}

Let
\[
X=\mathbb N\cup\{c\},
\]
and let
\[
H=\{h_{s,t}:s,t\in\{-1,+1\}\},
\]
where \(h_{s,t}(i)=s\) for \(i\in\mathbb N\) and
\(h_{s,t}(c)=t\). This fixed class has VC dimension two: it shatters \(\{1,c\}\), and its four members cannot shatter three points.

Fix an integer \(q\ge2^{4096}\), and set
\[
\alpha=1/q,\qquad \beta=\rho=1/20.
\]
The distribution is supported on \(1,\ldots,q^2\), with
\[
D_q(i,+1)=\frac3{4q^2},\qquad
D_q(i,-1)=\frac1{4q^2}.
\tag{S7.28}\label{legacy:eq:7.28}
\]
Every positive-mass atom has conditional mean \(1/2\). Both the Bayes risk and \(\operatorname{OPT}\) equal \(1/4\).

Consider implementations with the following explicit contract:

\begin{enumerate}[leftmargin=1.6em]
\item TEST uses a designated block of a predetermined integer \(M\) iid observations and the gate of \legacyref{3.2}{Lemma}.
\item Its parameters are
\[
\tau=\frac1{18q^2\log80},\qquad a^2=M\tau,
\]
with \(a\ge1\).
\item Whenever TEST returns \(-1\), the final prediction at that point must be \(-1\).
\item All other computations and outputs may be arbitrary.
\end{enumerate}
For every such implementation, if
\[
q^2\le M\le\frac{q^2\log q}{24},
\tag{S7.29}\label{legacy:eq:7.29}
\]
the accuracy failure probability is at least \(31/32\).

Here is a proof with a specified binomial counting variable. Put
\[
\nu=M/q^2,\qquad k=\lceil\nu\rceil.
\]
Temporarily replace the block size by \(\operatorname{Poi}(M)\), while keeping the nominal threshold \(a^2=M\tau\) fixed.

The marked Poisson counts \(N_i^+,N_i^-\) are independent, with respective means \(3\nu/4,\nu/4\). This factorization follows directly by multiplying the Poisson probability of the total count by the conditional multinomial probability:
\[
e^{-M}\prod_{i,\sigma}
\frac{(M D_q(i,\sigma))^{n_i^\sigma}}{n_i^\sigma!}
=
\prod_{i,\sigma}
e^{-M D_q(i,\sigma)}
\frac{(M D_q(i,\sigma))^{n_i^\sigma}}{n_i^\sigma!}.
\]

Define the witness indicators and their count by
\[
I_i=\mathbf1_{\{N_i^+=0,\ N_i^-=k\}},
\qquad
B_0=\sum_{i=1}^{q^2}I_i.
\tag{S7.30}\label{legacy:eq:7.30}
\]
Then, in the Poissonized experiment,
\[
B_0\sim\operatorname{Bin}(q^2,p),
\qquad
p=e^{-\nu}\frac{(\nu/4)^k}{k!}.
\]
This binomial assertion concerns \(B_0\), not the total number of negative gate outputs.

At each witness,
\[
s_i=-\sqrt k<-2a,
\]
because
\[
4a^2=\frac{4\nu}{18\log80}<\nu\le k.
\]
Hence the witness forces a negative output for every allowed threshold realization.

Since \(\nu\ge1\),
\[
k\le2\nu,\qquad k!\le k^k.
\]
Therefore
\[
p\ge e^{-\nu}\left(\frac{\nu}{4k}\right)^k
\ge e^{-\nu}8^{-k}
\ge e^{-6\nu}
\ge q^{-1/4},
\]
where \(1+2\log8<6\) and the upper bound in \eqref{legacy:eq:7.29} were used. Thus
\[
\mu:=\mathbb EB_0\ge q^{7/4}.
\]

For completeness, the needed binomial lower-tail bound follows from exponential Markov inequality. For \(t>0\),
\[
\mathbb Ee^{-tB_0}
=(1-p+pe^{-t})^{q^2}
\le e^{\mu(e^{-t}-1)}.
\]
Taking \(t=1/2\) and using \(e^{-1/2}\le5/8\) gives
\[
\Pr(B_0\le\mu/2)
\le e^{\mu/4+\mu(e^{-1/2}-1)}
\le e^{-\mu/8}.
\]
Consequently,
\[
\Pr(B_0<q^{7/4}/2)\le e^{-q^{7/4}/8}.
\]

To condition back to exactly \(M\) observations, note that
\[
\log M!
\le\int_1^M\log t\,dt+\log M
=M\log M-M+1+\log M.
\]
It follows that
\[
\Pr(\operatorname{Poi}(M)=M)\ge\frac1{eM}>\frac1{3M}.
\]
Thus, in the original fixed-sample experiment,
\[
\Pr(B_0<q^{7/4}/2)
\le3M e^{-q^{7/4}/8}<\frac1{32}.
\tag{S7.31}\label{legacy:eq:7.31}
\]
The final inequality has an elementary margin: from \eqref{legacy:eq:7.29},
\[
\log(96M)\le\log(4q^2\log q)\le4\log q,
\]
whereas \(q^{7/4}/8\ge8q>4\log q\) for \(q\ge256\).

Every witnessed negative output increases the risk above the Bayes risk by \(1/(2q^2)\). On the complementary event in \eqref{legacy:eq:7.31},
\[
e(F)-\operatorname{OPT}
\ge\frac{B_0}{2q^2}
\ge\frac14q^{-1/4}
>\frac1q=\alpha.
\]
Predictions at other atoms cannot improve on the Bayes risk and cannot cancel this loss.

Finally, the imposed condition \(a\ge1\) implies
\[
M\ge18\log80\,q^2>q^2.
\]
There is therefore no allowed \(M<q^2\) escape interval when this obstruction is used to infer the implementation-specific requirement
\[
M=\Omega(q^2\log q)
=\Omega\!\left(\frac{\log(1/\alpha)}{\alpha^2}\right).
\]
No independence assertion is made after conditioning on the total sample count.

\paragraph{(c) The same class admits a proper two-atom reduction}

The preceding obstruction is not a lower bound for learning the class \(H\). Define
\[
\phi(i)=1\quad(i\in\mathbb N),\qquad
\phi(c)=2.
\]
Every member of \(H\) is the pullback of a function on the two-point full class.

Here is an explicit proper learner for every distribution on this compressed domain. Let
\[
\mu_j=\mathbb E[Y\mathbf1_{\{\phi(X)=j\}}],
\qquad
\widehat\mu_j=\frac1N\sum_{i=1}^N
Y_i\mathbf1_{\{\phi(X_i)=j\}},
\]
where
\[
N=\left\lceil
\frac{32}{\alpha^2}
\left(\rho^{-2}+\log\frac8\beta\right)
\right\rceil.
\tag{S7.32}\label{legacy:eq:7.32}
\]
Use independent shared
\[
u_j\sim\operatorname{Unif}[-\alpha/4,\alpha/4],
\]
and output \(f_j=+1\) exactly when \(\widehat\mu_j\ge u_j\).

For each fixed \(j\), the threshold interval has length \(\alpha/2\), so
\[
\Pr(f_j\ne f'_j)
\le\frac2\alpha\mathbb E|\widehat\mu_j-\widehat\mu'_j|
\le\frac{2\sqrt2}{\alpha\sqrt N}
\le\frac\rho2<\rho.
\]
The second inequality uses the variance bound \(1/N\) for an average of variables in \([-1,1]\).

Hoeffding's inequality and a union bound over the two coordinates give
\[
\Pr\!\left(\exists j:
|\widehat\mu_j-\mu_j|>\alpha/4\right)
\le4e^{-N\alpha^2/32}
\le\beta/2.
\]
On the complementary event, a prediction of the wrong sign of \(\mu_j\) requires \(|\mu_j|\le\alpha/2\). Since the full two-point class attains the coordinatewise optimum,
\[
e(f)-\operatorname{OPT}
=\sum_{j:\,f_j\mu_j<0}|\mu_j|
\le\alpha.
\]
The two-point output belongs to the full class, and its pullback belongs to \(H\). This proves the proper full-parameter order
\[
O\!\left(
\frac1{\rho^2\alpha^2}
+\frac{\log(1/\beta)}{\alpha^2}
\right)
\]
for this class, independently of the number of original atoms or their mass scales.

\paragraph{(d) A rational three-atom VC-two counterexample to a Bayes-excess inference}

Let \(X=\{a,b,c\}\), and let
\[
H=\{h_{s,t}:s,t\in\{-1,+1\}\},
\qquad
h_{s,t}(a)=h_{s,t}(b)=s,\quad h_{s,t}(c)=t.
\]
The class shatters \(\{a,c\}\), but cannot shatter three points, so its VC dimension is exactly two.

Use the rational distribution
\[
\begin{array}{c|cc|cc}
x&D(x,+1)&D(x,-1)&w_x&y_x\\ \hline
a&3/10&1/10&2/5&1/2\\
b&1/10&3/10&2/5&-1/2\\
c&1/10&1/10&1/5&0
\end{array}
\]
For every \(h_{s,t}\),
\[
e(h_{s,t})
=(1/5-s/10)+(1/5+s/10)+1/10
=\frac12.
\]
Thus every proper hypothesis is exactly optimal in \(H\), whereas the Bayes risk is \(3/10\).

The uniform mixture of the four proper hypotheses has
\(p_a=p_b=p_c=0\). Its unshifted, independent coordinatewise rounding satisfies
\[
e(F)=\frac12-\frac1{10}F_a+\frac1{10}F_b,
\]
and therefore
\[
e(F)=
\begin{cases}
3/10,&\text{with probability }1/4,\\
1/2,&\text{with probability }1/2,\\
7/10,&\text{with probability }1/4.
\end{cases}
\tag{S7.33}\label{legacy:eq:7.33}
\]
Consequently, for every \(\alpha\in(0,1/10)\),
\[
\Pr(e(F)>\operatorname{OPT}+\alpha)=\frac14,
\qquad
\operatorname{Var}(e(F))=\frac1{50}.
\]

This proves that small agnostic excess risk---even zero agnostic excess risk for every hypothesis in a mixture---does not imply small Bayes excess risk or high-confidence accuracy after independent marginal rounding.

This example concerns independent rounding of the specified mixture. SOFT-VC instead shifts the marginals using the signed histogram before rounding.

A related scale calculation must retain its accuracy denominator:
\[
\tau=\frac{\alpha^2}{18L}
\quad\Longrightarrow\quad
\frac{\alpha}{12\tau}
=\frac{3L}{2\alpha},
\]
not \(3L/2\). The same rational distribution can be placed in an arbitrarily small-accuracy fixed-confidence regime by choosing
\(\alpha=1/q\), integer \(q\ge2^{4096}\), and
\(\rho=\beta=1/20\).

\end{legacydefinition}

\subsection{Consolidated bounds for the earlier algorithms}\label{legacy:section:summary}

\par\noindent\textbf{Role and dependencies.} This consolidates the sample budgets proved in the preceding subsections. It compares the earlier algorithms in Section~\ref{sec:soft-retained}; the new all-parameter optimum is Theorem~\ref{thm:main}.

\hypertarget{legacy:2.2}{}
\begin{legacytheorem}{Theorem S2.2}{Explicit bounds for the earlier agnostic algorithms}

For every countable domain \(X\), every class \(H\subseteq\{-1,+1\}^{X}\) of VC dimension \(d\ge1\), and every \(\alpha,\beta,\rho\in(0,1/10)\), the following statements hold.

\textbf{(a) Full-parameter lower and upper bounds.} Put
\[
b=\log\frac1\beta,\qquad
L=\log\frac4\beta,\qquad
\ell=\log\frac{256}{\rho},\qquad
J=\log\frac{512}{\rho},
\]
\[
K=\log\frac{288L}{\alpha^2\beta},
\qquad
K_0=\log\frac{294912\ell}{\alpha^2\rho},
\]
and
\[
Q_R=\frac{d+\ell}{\rho^2}+\ell K_0+b,\qquad
Q_P=\frac{d+L}{\rho^2}+LK,
\]
\[
s_{\mathrm S}=\min\{L^2,J^2\},
\qquad
Q_{\mathrm S}=\frac{d+s_{\mathrm S}}{\rho^2}+L.
\]

The first two routes give
\[
\boxed{
\frac{\bar\kappa}{\alpha^2}
\left(\frac d{\rho^2}+\log\frac1\beta\right)
\le m_H(\alpha,\beta,\rho)
\le
\frac{2^{30}}{\alpha^2}\min\{Q_R,Q_P\}.
}
\tag{S2.1}\label{legacy:eq:2.1}
\]

A third route, using SOFT-VC and rare repair, gives
\[
\boxed{
m_H(\alpha,\beta,\rho)
\le
\frac{C_3}{\alpha^2}
\left[
\frac{d+\min\{L^2,J^2\}}{\rho^2}+L
\right],
\qquad C_3=2^{39}.
}
\tag{S2.1S}\label{legacy:eq:2.1S}
\]

Consequently, the three-route bound is
\[
\boxed{
\begin{aligned}
\frac{\bar\kappa}{\alpha^2}
\left(\frac d{\rho^2}+b\right)
&\le m_H(\alpha,\beta,\rho)\\
&\le
\frac1{\alpha^2}
\min\{2^{30}Q_R,\ 2^{30}Q_P,\ C_3Q_{\mathrm S}\}\\
&\le
\frac{C_3}{\alpha^2}\min\{Q_R,Q_P,Q_{\mathrm S}\}.
\end{aligned}
}
\tag{S2.1U}\label{legacy:eq:2.1U}
\]
The same inequalities hold for \(m^*(d,\alpha,\beta,\rho)\).

A further explicit upper bound is
\[
\boxed{
m_H(\alpha,\beta,\rho)
\le
\frac{2^{21}\mathcal B}{\rho^2\alpha^2}
+\frac{1152L}{\alpha^2}
\log\frac{288L}{\alpha^2\beta},
}
\tag{S2.2}\label{legacy:eq:2.2}
\]
where
\[
\mathcal B
=d+\log\frac{32}{\beta}
+\min\left\{\log\frac1\rho,\log\frac1\alpha\right\}.
\]

In asymptotic form, the first two routes are
\[
O\!\left(
\frac{d+\log(1/\rho)}{\rho^2\alpha^2}
+\frac{\log(1/\rho)\log(1/(\alpha\rho))}{\alpha^2}
+\frac{\log(1/\beta)}{\alpha^2}
\right)
\tag{S2.3}\label{legacy:eq:2.3}
\]
and
\[
O\!\left(
\frac{d+\log(1/\beta)}{\rho^2\alpha^2}
+\frac{\log(1/\beta)}{\alpha^2}
\log\frac{\log(1/\beta)}{\alpha^2\beta}
\right).
\tag{S2.4}\label{legacy:eq:2.4}
\]
The third route is
\[
O\!\left(
\frac{d+\min\{\log^2(1/\beta),\log^2(1/\rho)\}}
{\rho^2\alpha^2}
+\frac{\log(1/\beta)}{\alpha^2}
\right).
\tag{S2.4S}\label{legacy:eq:2.4S}
\]
In particular, \eqref{legacy:eq:2.5} holds:
\[
m^*(d,\alpha,\beta,\rho)
=
\widetilde\Theta_{\alpha,\rho}\!\left(
\frac d{\rho^2\alpha^2}
+\frac{\log(1/\beta)}{\alpha^2}
\right)
\tag{S2.5}\label{legacy:eq:2.5}
\]
where the suppressed factors may depend polylogarithmically only on \(1/\alpha\) and \(1/\rho\), and the confidence term is additive.

\textbf{(b) Exact absorption regions and fixed confidence.} Let
\[
B_*=\frac d{\rho^2}+b,
\]
\[
E_R=\frac{\ell}{\rho^2}+\ell K_0,\qquad
E_P=\frac L{\rho^2}+LK,\qquad
E_{\mathrm S}=\frac{s_{\mathrm S}}{\rho^2}.
\]
For every fixed \(A<\infty\), throughout the region
\[
\min\{E_R,E_P,E_{\mathrm S}\}\le A B_*,
\tag{S2.6}\label{legacy:eq:2.6}
\]
one has
\[
m^*(d,\alpha,\beta,\rho)
=
\Theta_A\!\left(
\frac d{\rho^2\alpha^2}
+\frac{\log(1/\beta)}{\alpha^2}
\right).
\tag{S2.7}\label{legacy:eq:2.7}
\]
Condition \eqref{legacy:eq:2.6} is the constant-factor absorption test for the three displayed general-purpose bounds, up to changing \(A\). The phase construction of Theorem~\ref{thm:main} attains the optimal order throughout the parameter range.

Two sufficient regions are
\[
d\ge\ell,\qquad
\log\frac1\rho\log\frac1{\alpha\rho}
\lesssim\frac d{\rho^2},
\tag{S2.8}\label{legacy:eq:2.8}
\]
or
\[
d\ge\log\frac1\beta,\qquad
\log\frac1\beta
\log\frac{\log(1/\beta)}{\alpha^2\beta}
\lesssim\frac d{\rho^2}.
\tag{S2.9}\label{legacy:eq:2.9}
\]
The third route supplies the additional sufficient region
\[
\min\{L^2,J^2\}\lesssim d+\rho^2b.
\tag{S2.9S}\label{legacy:eq:2.9S}
\]

For every fixed \(\beta\in(0,1/10)\),
\[
\boxed{
m^*(d,\alpha,\beta,\rho)
=\Theta_\beta\!\left(\frac d{\rho^2\alpha^2}\right)
\quad
\text{for all }d\ge1,\quad \alpha,\rho\in(0,1/10).
}
\tag{S2.10}\label{legacy:eq:2.10}
\]
There is no restriction on \(\log(1/\alpha)\) or \(\log(1/\rho)\); in particular the subregion \(\log(1/\alpha)\lesssim d/\rho^2\) is covered. The same fixed-confidence order holds for each individual class \(H\), with constants uniform over \(H\) and \(X\).

\textbf{(c) Two structural cases with no logarithmic gap.} For every VC-one class on a countable domain, there is a proper learner attaining
\[
m_H(\alpha,\beta,\rho)
=
\Theta\!\left(
\frac1{\rho^2\alpha^2}
+\frac{\log(1/\beta)}{\alpha^2}
\right).
\tag{S2.11}\label{legacy:eq:2.11}
\]
For \(X=[d]\) and \(H=\{-1,+1\}^{[d]}\), there is a proper learner attaining
\[
m_H(\alpha,\beta,\rho)
=
\Theta\!\left(
\frac d{\rho^2\alpha^2}
+\frac{\log(1/\beta)}{\alpha^2}
\right).
\tag{S2.12}\label{legacy:eq:2.12}
\]

\textbf{(d) Budget comparison for the three earlier routes.} Equation \eqref{legacy:eq:2.10} gives the exact fixed-confidence order. We now compare the three earlier bounds as \(\beta\) varies. Theorem~\ref{thm:upper} supplies the uniform optimum.

For the earlier general-purpose bounds, consider the low-dimensional regime
\[
2\le d\ll
\min\{\log(1/\rho),\log(1/\beta)\},
\]
outside the absorption region \eqref{legacy:eq:2.6}. For example, when
\[
d=2,\qquad \alpha=\frac1{20},\qquad \beta=\rho\longrightarrow0,
\]
the additional terms in the bounds for \(R\) and \(P\) each have order
\[
\frac{\log(1/\rho)}{\rho^2\alpha^2},
\]
whereas the additional term for \(\mathrm S\) has order
\[
\frac{\log^2(1/\rho)}{\rho^2\alpha^2}.
\]

The additional terms in the three bounds are
\[
\frac1{\alpha^2}
\left(\frac{\ell}{\rho^2}+\ell K_0\right),\qquad
\frac1{\alpha^2}
\left(\frac L{\rho^2}+LK\right),\qquad
\frac{s_{\mathrm S}}{\rho^2\alpha^2}.
\]
For any nonnegative additional numerator \(s\), including
\(s=\ell,L,s_{\mathrm S}\), its \(\rho^{-2}\)-amplified term is absorbable into the lower-bound scale with a uniform constant exactly when
\[
s\lesssim d+\rho^2\log(1/\beta).
\tag{S2.13}\label{legacy:eq:2.13}
\]
For routes \(R\) and \(P\), the additional terms \(\ell K_0\) and \(LK\) must also be absorbed.

The dependence on \(\alpha\) yields a second comparison. For the rational parameter family
\[
\rho=\beta=\frac1q,\qquad
\alpha=2^{-q^3},\qquad
d=\lceil\log q\rceil,\qquad q\in\mathbb N,\ q\ge16,
\]
the dimension is comparable to both logarithms, while the smallest displayed upper-bound numerator divided by \(B_*\) diverges. The proof below verifies this. Attainment of the additive scale \eqref{legacy:eq:2.7} by these three bounds is characterized by \eqref{legacy:eq:2.6}; part (c) gives separate structural guarantees.

Relative to the target \(B_*/\alpha^2\) in \eqref{legacy:eq:2.7}, uniform constant-factor comparison of these bounds amounts to boundedness of \(\min\{E_R,E_P,E_{\mathrm S}\}/B_*\). An unbounded ratio is equivalent to failure of \eqref{legacy:eq:2.6} for every fixed \(A\) along a subsequence. Part (e) compares the weaker target of \cite{HIY}.

\textbf{(e) Answer to the additional-term question.} Hopkins, Impagliazzo, and Ye \cite{HIY} ask whether the additive \(\rho^{-2}\alpha^{-O(1)}\) contribution is necessary and whether a \(\rho\)-pointwise replicable learner can use \(O(m(\alpha,\beta)\rho^{-2})\) samples. Here \(m(\alpha,\beta)\) denotes the agnostic sample complexity without replicability, which is of order \((d+\log(1/\beta))/\alpha^2\). The first question is answered in the negative for every countable domain and every class of finite VC dimension, with explicit constants: the additional \(\rho^{-2}\alpha^{-4}\) term in the earlier general construction is unnecessary, and confidence amplification can be achieved at an additive sample cost rather than a cost multiplied by \(\rho^{-2}\). The second question concerns the weaker target \(O(m(\alpha,\beta)\rho^{-2})=O((d+b)/(\rho^2\alpha^2))\), in which the confidence term may be multiplied by \(\rho^{-2}\); it must be distinguished from the additive target \(\Theta(d/(\rho^2\alpha^2)+b/\alpha^2)\) of parts (b)--(d). Put \(B_H=(d+b)/\rho^2\). For the three displayed routes, the constant-factor test for the target of \cite{HIY} is
\[
\min\{Q_R,Q_P,Q_{\mathrm S}\}\le A'B_H
\tag{S2.14}\label{legacy:eq:2.14}
\]
for a fixed \(A'<\infty\). On this region, \eqref{legacy:eq:2.14} and \eqref{legacy:eq:2.1U} give \(m_H\le C_3A'B_H/\alpha^2=O_{A'}(m(\alpha,\beta)\rho^{-2})\). The structural cases of part (c) attain this target by \eqref{legacy:eq:2.11}--\eqref{legacy:eq:2.12}.

The region certified for the target of \cite{HIY} strictly contains \eqref{legacy:eq:2.6}. Indeed, \eqref{legacy:eq:2.6} with constant \(A\) implies \eqref{legacy:eq:2.14} with \(A'=A+2\). The family \(d=2\), \(\alpha=1/20\), \(\beta=\rho\to0\) satisfies \eqref{legacy:eq:2.14} with \(A'=5\) through route \(R\), while violating \eqref{legacy:eq:2.6} for every fixed \(A\). Along the rational family in part (d), the ratio \(\min\{Q_R,Q_P,Q_{\mathrm S}\}/B_H\) is unbounded. These comparisons characterize the earlier bounds. Theorem~\ref{thm:main} attains the stronger additive target throughout the parameter range. For every fixed confidence \(\beta_0\), the order is exactly \(\Theta_{\beta_0}(d/(\rho^2\alpha^2))\), with no logarithmic factor.

\end{legacytheorem}

\subsubsection*{Completion of the proof of \legacyref{2.2}{Theorem}}

\legacyref{5.1}{Theorem} and \legacyref{5.2}{Lemma} prove the lower bound in \eqref{legacy:eq:2.1}. \legacyref{4.2}{Proposition} proves the \(Q_P\) upper bound and \eqref{legacy:eq:2.2}. \legacyref{4.4}{Theorem} proves the \(Q_R\) upper bound. \legacyref{7.6}{Theorem} proves the \(Q_{\mathrm S}\) upper bound. All these statements are uniform over the classes in \legacyref{2.2}{Theorem}, so they remain valid after taking the supremum over those classes.

On the parameter range in question,
\[
L=\Theta(b),\qquad
\ell=\Theta(\log(1/\rho)),\qquad
J=\Theta(\log(1/\rho)),
\]
\[
K_0=\Theta(\log(1/(\alpha\rho))),
\qquad
K=\Theta\!\left(
\log\frac{\log(1/\beta)}{\alpha^2\beta}
\right).
\]
For example, \(L\le2b\), and both \(\ell\) and \(J\) are at most
\(4\log(1/\rho)\). Expanding the three numerators gives \eqref{legacy:eq:2.3}, \eqref{legacy:eq:2.4}, and \eqref{legacy:eq:2.4S}. The first route already proves \eqref{legacy:eq:2.5}.

For the absorption test, observe that
\[
Q_R=B_*+E_R,
\qquad
Q_P=B_*+E_P-b,
\qquad
Q_{\mathrm S}=B_*+E_{\mathrm S}+\log4.
\]
Moreover,
\[
b\le\rho^2E_P,\qquad \rho^2<1/100,
\]
and \(\log4\le B_*\). Thus
\[
B_*+\frac{99}{100}E_P
\le Q_P\le B_*+E_P,
\]
and
\[
B_*+E_{\mathrm S}
\le Q_{\mathrm S}\le2B_*+E_{\mathrm S}.
\]
These comparisons establish \eqref{legacy:eq:2.6}--\eqref{legacy:eq:2.7}, including the stated constant-factor interpretation. The logarithmic comparisons above establish the sufficient regions \eqref{legacy:eq:2.8}--\eqref{legacy:eq:2.9}. Equation \eqref{legacy:eq:2.9S} is the soft-route absorption condition.

\legacyref{7.5}{Theorem} proves \eqref{legacy:eq:2.10}. \legacyref{6.2}{Theorem} and \legacyref{6.3}{Theorem} prove the proper structural assertions in part (c). Multiplying
\[
\frac{s}{\rho^2}\lesssim\frac d{\rho^2}+b
\]
by \(\rho^2\) proves \eqref{legacy:eq:2.13}. Substituting
\(d=2,\alpha=1/20,\beta=\rho\) gives the three overheads stated in part (d).

We now quantify the overhead of the three bounds in part (d) of \legacyref{2.2}{Theorem}.

Let \(q\ge16\) be an integer and put
\[
r=\log q,\qquad
\rho=\beta=q^{-1},\qquad
\alpha=2^{-q^3},\qquad
d=\lceil r\rceil.
\]
All three probability/accuracy parameters are rational and lie in the required range. Since \(r>1\),
\[
r\le d\le2r,
\qquad
B_*=dq^2+r\le3rq^2.
\]
Also
\[
\ell\ge r,\quad L\ge r,\quad
K_0\ge2q^3\log2,\quad
K\ge2q^3\log2.
\]
It follows that
\[
\frac{Q_R}{B_*}\ge\frac{2\log2}{3}q,
\qquad
\frac{Q_P}{B_*}\ge\frac{2\log2}{3}q.
\]
Since \(L<J\), the soft numerator uses \(s_{\mathrm S}=L^2\). Therefore
\[
\frac{Q_{\mathrm S}}{B_*}
\ge\frac{L^2q^2}{3rq^2}
\ge\frac r3.
\]
Consequently,
\[
\frac{\min\{Q_R,Q_P,Q_{\mathrm S}\}}{B_*}
\longrightarrow\infty,
\]
although
\[
\frac d{\min\{\log(1/\rho),\log(1/\beta)\}}
=\frac{\lceil r\rceil}{r}\longrightarrow1.
\]
Thus the overhead of these three bounds extends beyond \(d\ll\min\{\log(1/\rho),\log(1/\beta)\}\) when \(\alpha\) is unrestricted. The exact absorption test describes their full parameter dependence.

Part (e) follows from the established upper bounds and \legacyref{7.5}{Theorem}, together with the following three facts about \eqref{legacy:eq:2.14}. First, \(B_*=d/\rho^2+b\le(d+b)/\rho^2=B_H\). If \eqref{legacy:eq:2.6} holds with constant \(A\), then for a route attaining the minimum in \eqref{legacy:eq:2.6} the comparisons above give \(Q\le(2+A)B_*\); hence \(\min\{Q_R,Q_P,Q_{\mathrm S}\}\le(A+2)B_*\le(A+2)B_H\), which is \eqref{legacy:eq:2.14} with \(A'=A+2\). Second, let \(d=2\), \(\alpha=1/20\), and \(\beta=\rho\), so that \(b=\log(1/\rho)>\log10\). Then \(\ell\le4b\) and \(K_0=\log(294912\cdot400\,\ell)+b\le19+\log(4b)+b\le19+2b\), so
\[
\ell K_0+b\le4b(19+2b)+b=b(77+8b)\le b\,e^{2b}\le B_H,
\]
because \(77+8b\le e^{2b}\) for \(b\ge\log10\). Also \((2+\ell)/\rho^2\le(2+4b)/\rho^2\le4B_H\). Hence \(Q_R\le5B_H\), and \eqref{legacy:eq:2.14} holds with \(A'=5\) along this family. On the same family, \(E_R/B_*\ge(\ell/\rho^2)/(2/\rho^2+b)\ge b/(2+b\rho^2)\to\infty\), and likewise \(E_P/B_*\ge L/(2+b\rho^2)\to\infty\) and \(E_{\mathrm S}/B_*\ge L^2/(2+b\rho^2)\to\infty\), so \eqref{legacy:eq:2.6} fails for every fixed \(A\). Third, on the rational family above, \(B_H=(d+r)q^2\le3rq^2\le3B_*\), so \(\min\{Q_R,Q_P,Q_{\mathrm S}\}/B_H\ge\min\{Q_R,Q_P,Q_{\mathrm S}\}/(3B_*)\to\infty\).

\subsection{Notation, constants, and logarithmic factors}\label{legacy:section:constants}

\legacyheading{constants.1}{Notation}

\begin{center}\small
\begin{longtable}{>{\raggedright\arraybackslash}p{0.34\textwidth}>{\raggedright\arraybackslash}p{0.60\textwidth}}
\toprule
\textbf{Symbol} & \textbf{Meaning} \\
\midrule
\endhead
\(w_x,y_x\) & Atom mass and conditional label mean \\
\(e(f),\operatorname{OPT}\) & True risk and infimal risk in \(H\) \\
\(N_x,C_x,s_x\) & Hit count, label sum, \(C_x/\sqrt{N_x}\), with zero-hit value \(0\) \\
\(M,\tau,a\) & Gate sample count, danger scale, \(a=\sqrt{M\tau}\) \\
\(u_x^\pm\) & Shared independent gate uniforms on \([1,2]\) \\
\(v_x\) & Shared independent rounding uniforms on \([-1,1]\) \\
\(T,h_t,p_x\) & Block count, base outputs, average prediction \\
\(B_R\) & \(\sum_{x\in R}w_x^2y_x^2\), the rounding range-square sum \\
\(K_E\) & ERM square-moment constant \(2^{24}\), distinct from gate logarithms \\
\(\eta\) & \(\min\{\beta/8,\rho/128\}\) in rare repair \\
\(\Theta,U\) & Shared filter threshold and quantile uniform in the VC-one learner \\
\(H_c,\prec\) & Countable proper core and its laminar interval order \\
\(m_H,m^*\) & Class-specific and worst-case sample complexities \\
\bottomrule
\end{longtable}
\end{center}

The rare-repair algorithm uses no additional random threshold. The variables \(\Theta,U\) are used in the ordered VC-one construction.

\legacyheading{constants.2}{Closed-form verification of the constants}

The table evaluates the explicit formulas used in the sample-budget proofs.

\begin{center}\small
\begin{longtable}{>{\raggedright\arraybackslash}p{0.34\textwidth}>{\raggedright\arraybackslash}p{0.60\textwidth}}
\toprule
\textbf{Constant} & \textbf{Verification} \\
\midrule
\endhead
Hoeffding \(1/8\) & Tilted variance \(\le c^2/4\), integrated twice \\
Hoeffding tail \(2\) & Optimize \(-\lambda t+\lambda^2V/8\) at \(\lambda=4t/V\) \\
Gate \(\sqrt2/a\) & \(\operatorname{Var}(s_x)\le1\), two-sample difference variance \(\le2\) \\
Dangerous threshold \(9\tau\) & Mean magnitude \(>3a\), all thresholds have magnitude \(\le2a\) \\
Deterministic \(25\tau/4\) & Mean magnitude \(\ge5a/2\), deviation \(\le a/2\) \\
VC packing \((5e/\eta)^{2d}\) & Collision sample of size \(\lceil2\log N/\eta\rceil\), \(0<\eta\le1\) \\
Chaining \(35\) & \(1+24\sqrt2<35\) \\
Symmetrization \(70\) & Twice the chaining bound \\
VC \(2^{16}\) & \(99/(256\sqrt2)<1/2\) \\
\(K_E=2^{24}\) & \(2^{16}\cdot16^2\) \\
Moment \(21e^d\) & \(1+20e^d\le21e^d\) \\
Mean \(2^{14}\sqrt{d/n}\) & \(2K_E(d+\log21)\le2^{28}d\) \\
P block \(2^{27}\) & \(8K_E\) \\
Gate \(1152\) & \(64\cdot18\) \\
Gate failure \(13/144\) & \(1/16+1/36\) \\
General accuracy \(85/144\) & \(1/16+1/36+1/4+1/4\) \\
P accuracy \(67/144\) & \(1/16+1/36+1/8+1/4\) \\
General disagreement \(5/(4\sqrt2)\) & \(\sqrt2/8+1/\sqrt2\) \\
Repair disagreement \(35/64\) & \(1/2+2(1/64+1/128)\) \\
Route-one \(\tau_0\) denominator \(288\) & \(18\cdot4^2\) \\
Route-one log constant \(294912\) & \(16\cdot288\cdot64\) \\
Route-one gate \(18432\) & \(64\cdot288\) \\
Explicit route coefficient \(2^{30}\) & Dominates \(419504424\) and all other coefficients \\
Lower potential \(37/100\) & \(1-(9/4)(7/25)\) \\
\(\kappa\) & \((37/(500\pi))^2\) \\
\(\bar\kappa\) & Harmonic combination of \(\kappa\) and \(3/50\) \\
VC-one \(\varepsilon_{\rm mis}\) & \((4/\alpha)(2\rho\alpha/2^{11})=\rho/256\) \\
VC-one disagreement & \((16/15)(8\rho/256+4\rho/64)=\rho/10\) \\
Atom-class gate term \(512d\) & Ensures \(M\tau\ge8/\rho^2\) \\
Atom-class empirical excess & \(2\sqrt{d\tau}=\alpha/4\) \\
Realizable \(3072\) & \(64\cdot48\) \\
Realizable \(768\) & \(16\cdot48\) \\
Realizable failure \(19/36\) & \(1/4+1/36+1/4\) \\
Independent-rounding \(20\) & \(16+32/21<20\), with \(N\alpha^2\ge1\) \\
\bottomrule
\end{longtable}
\end{center}

Selected decimal values:
\[
\begin{array}{c|c}
\kappa&0.0005548348016254416\\
\bar\kappa&0.0005497511174224673\\
85/144&0.5902777777777778\\
35/64&0.546875\\
5/(4\sqrt2)&0.8838834764831844\\
99/(256\sqrt2)&0.2734514505369852
\end{array}
\]

\legacyheading{constants.3}{Origin of each logarithmic factor}

\begin{center}\small
\begin{longtable}{>{\raggedright\arraybackslash}p{0.34\textwidth}>{\raggedright\arraybackslash}p{0.60\textwidth}}
\toprule
\textbf{Factor} & \textbf{Source and treatment} \\
\midrule
\endhead
\(\log(1/\beta)\) in \(L\) & Rounding tail budget \\
\(\log(1/\tau)\) in \(K\) & Simultaneous safety/coverage over atoms \\
\(\log T\) in U & Union bound over base blocks \\
\(\log(1/\alpha)\) in M & Bad-block probability \(\alpha\beta/16\) \\
\(\log(1/\gamma)/T\) in P & Product of square-exponential moments; total cost is additive \\
\(\ell=\log(256/\rho)\) & Running the base construction at failure probability \(\rho/64\) \\
Additive \(\log(1/\beta)\) after repair & Proper reference and validation, neither multiplied by \(\rho^{-2}\) \\
VC-one validation \(\log(T+1)\) & Finite union bound for accuracy; absorbed by \(\rho^{-2}\) \\
Realizable \(d\log(1/\alpha)\) & One-sided consistent-learning growth bound \\
No agnostic \(d\log(1/\alpha)\) in the ERM block term & Removed by empirical chaining \\
\bottomrule
\end{longtable}
\end{center}

\legacyheading{constants.4}{Additional notation for SOFT-VC}

\begin{center}\small
\begin{longtable}{>{\raggedright\arraybackslash}p{0.34\textwidth}>{\raggedright\arraybackslash}p{0.60\textwidth}}
\toprule
\textbf{Symbol} & \textbf{Meaning} \\
\midrule
\endhead
\(\zeta_x\) & \(w_xy_x/2\), the signed population histogram \\
\(\widehat\zeta_x\) & Signed empirical histogram, zero at unobserved points \\
\(\mathcal E_D(p)\) & Linear risk of a fractional prediction \\
\(\Pi\) & Clipping map onto \([-1,1]\) \\
\(\lambda\) & Soft-shift strength; not a gate count mean \\
\(U_x\) & Independent shared rounding uniforms on \([-1,1]\) \\
\(J\) & \(\log(512/\rho)\), the repaired soft learner's internal confidence logarithm \\
\(s_{\mathrm S}\) & \(\min\{L^2,J^2\}\) \\
\(Q_{\mathrm S},E_{\mathrm S}\) & Third-route numerator and its additional overhead \\
\(C_{\mathrm{soft}}\) & \(2^{37}\), direct SOFT-VC coefficient \\
\(C_3\) & \(2^{39}\), coefficient valid after repair and selection \\
\(\delta_1\) & Markov budget for the random SMALL-region variance proxy \\
\(B_0\) & Count of the specified Poissonized witness events, not all negative gate outputs \\
\(W\) & Event of selecting the repair branch \\
\bottomrule
\end{longtable}
\end{center}

\legacyheading{constants.5}{Closed-form verification of the constants}

All values below are recomputed from the formulas in the text; none is a simulation result.

\begin{center}\small
\begin{longtable}{>{\raggedright\arraybackslash}p{0.34\textwidth}>{\raggedright\arraybackslash}p{0.60\textwidth}}
\toprule
\textbf{Constant or coefficient} & \textbf{Derivation} \\
\midrule
\endhead
Clipped-shift \(2/3\) & Complete the square in \(-3z^2/2-2z\Delta\) \\
Incorrect-endpoint \(1/2\) & \(2ab\le(a+b)^2/2\) \\
Histogram second moment \(1/(4M)\) & \(\|V_i\|_2=1/2\) and independence \\
Histogram change \(1/M\) & \(\|V_i-V_i'\|_2\le1\) \\
Histogram tail \(e^{-2Mt^2}\) & Doob increments have range length \(1/M\) \\
Average-ERM coefficient \(32K_E\) & Moment exponent \(n/(2K_E)\), squared threshold \(\alpha^2/16\) \\
\(\lambda=4L/\alpha\) & Gives \(-\lambda\alpha/2=-2L\) \\
Direct histogram coefficient \(512\) & \(32\cdot4^2\) \\
\((1+\sqrt{2L})^2\le6L\) & \(3+2\sqrt2<6\), \(L\ge1\) \\
Direct failure \(3\beta/4\) & Two budgets \(e^{-L}\), one budget strictly below \(e^{-L}\) \\
Direct disagreement \(\rho/4\) & \(\rho/8+\rho/8\) \\
Direct \(d\)-coefficient & \(5280K_E+34=88583700514<2^{37}\) \\
\(C_{\mathrm{soft}}=2^{37}\) & Dominates the coefficients of \(d/\rho^2,L^2/\rho^2,L\) \\
\(C_{\beta_0}\) & \(2^{37}(1+L_0^2+L_0)\) \\
Repaired \(\lambda_B\) & \(4J/(\alpha/4)=16J/\alpha\) \\
Repaired histogram coefficient \(8192\) & \(32\cdot16^2\) \\
Switch budget & \(<\rho/128+\eta\le\rho/64\) \\
Repaired disagreement \(9\rho/32\) & \(1/4+2(1/64)=9/32\) \\
Repaired accuracy & \(2\eta\le\beta/4\) \\
Repaired \(A_d\) & \(16896K_E+2^{22}+36=283472035876\) \\
Repaired \(A_s\) & \(17408K_E+8192+2^{22}+128=292061978752\) \\
Repaired \(A_L\) & \(2^{22}+128=4194432\) \\
\(C_3=2^{39}\) & Dominates \(A_d,A_s,A_L\) \\
Full-parameter lower \(c=\kappa/2\) & Separate lower bounds for \(d/\rho^2\) and \(L\), then average \\
Expected SMALL proxy \(16\tau\) & Split at \(v=4a\), use \(\operatorname{Var}(s)\le1\) \\
SMALL interface exponent \(1/8\) & \(2t^2/(16\tau/\delta_1)=\delta_1t^2/(8\tau)\) \\
Fixed-confidence SMALL denominator \(256\) & \(\delta_1=\beta_0/8\), \(t=\alpha/2\) \\
Witness probability exponent \(6\nu\) & \(1+2\log8<6\) \\
Witness lower-tail \(1/8\) & Exponential Markov at \(t=1/2\) \\
De-Poissonization factor \(3M\) & \(\Pr(\operatorname{Poi}(M)=M)>1/(3M)\) \\
Two-atom sample coefficient \(32\) & Hoeffding for \([-1,1]\) at deviation \(\alpha/4\) \\
Three-atom variance \(1/50\) & \(2(1/4)(1/5)^2\) \\
\bottomrule
\end{longtable}
\end{center}

\textbf{Algorithm constants.} The constant \(2^{37}\) belongs to the direct algorithm of \legacyref{7.4}{Theorem}. Expanding the inner accuracy and all sample blocks for the repaired algorithm gives the constant \(2^{39}\).

\legacyheading{constants.6}{Logarithmic factors of SOFT-VC}

\begin{center}\small
\begin{longtable}{>{\raggedright\arraybackslash}p{0.34\textwidth}>{\raggedright\arraybackslash}p{0.60\textwidth}}
\toprule
\textbf{Factor} & \textbf{Source and treatment} \\
\midrule
\endhead
\(L\) in the average-ERM term & Product of square-exponential moments; appears as \(L/T\) per block and costs \(O(L/\alpha^2)\) in total \\
No \(\log T\) in SOFT-VC & No union bound over the \(T\) ERMs \\
\(L\) in \(\lambda\) & Suppresses the final rounding-seed upper tail \\
\(L^2/\rho^2\) & Current pointwise stability calculation requires \(M\ge32\lambda^2/\rho^2\) \\
No atom-count logarithm in the histogram & Concentration is for one Hilbert norm, not a union over coordinates \\
No \(\log(1/\alpha)\) in the soft route & No simultaneous atomwise sign-safety requirement \\
\(J=\log(512/\rho)\) & Internal confidence \(\rho/128\) in the repaired soft learner \\
\(J^2/\rho^2\) & The same shift-sensitivity cost after substituting the internal confidence \\
Additive \(L\) after repair & Reference ERM and validation; not multiplied by \(\rho^{-2}\) \\
\(\min\{L^2,J^2\}\) & Deterministic parameter-based selection of direct or repaired soft learning \\
\(\log(8/\beta_0)\) in \legacyref{7.7}{Remark}(a) & A fixed rounding-tail budget after the Markov split; absorbed into \(c_{\beta_0}\) \\
\(\log q\) in the forced-gate obstruction & Implementation-specific accumulation of witnessed incorrect forced outputs \\
Gate factors \(\ell K_0\), \(LK\) & Sample-budget terms of the retained gate routes \\
\bottomrule
\end{longtable}
\end{center}

\subsection{Complete pseudocode for the earlier constructions}\label{legacy:section:algorithms}

The pseudocode calls only objects proved in the text; all sample blocks are mutually independent.

\begin{small}\begin{verbatim}
PR-VC(alpha, beta, rho, strategy):
    L   := log(4/beta)
    tau := alpha^2/(18 L)
    K   := log(16/(tau beta))
    M   := ceil(64 (rho^(-2)+K)/tau)
    T   := ceil(rho^(-2))

    Choose block size n by (S4.12), (S4.13), or (S4.14).
    Split samples into one gate block of size M and T blocks of size n.
    Train independent proper ERMs h[1],...,h[T].

    Shared randomness:
        u_x^+,u_x^- independently Uniform[1,2], for all x
        v_x independently Uniform[-1,1], for all x
        all gate and rounding variables mutually independent

    For any query x:
        N_x := number of gate observations at x
        C_x := sum of their labels
        s_x := 0 if N_x=0, otherwise C_x/sqrt(N_x)
        a   := sqrt(M tau)

        if s_x >= a u_x^+: return +1
        if s_x <= -a u_x^-: return -1
        p_x := average_t h[t](x)
        return +1 iff v_x <= p_x
\end{verbatim}\end{small}

\begin{small}\begin{verbatim}
ADDITIVE-CONFIDENCE(alpha, beta, rho):
    Train B with parameters (alpha/4, rho/64, rho/2),
        using (S4.27)-(S4.28).
    eta := min(beta/8, rho/128)
    Train independent proper ERM g using (S4.20).
    Validate B and g on the independent sample (S4.21).
    Return B if empirical_error(B) <= empirical_error(g)+alpha/2;
    otherwise return g.
\end{verbatim}\end{small}

\begin{small}\begin{verbatim}
ORDERED-VC-ONE(alpha, beta, rho):
    Fix the countable proper core and interval order of Lemma S6.1.
    Train T base ERMs and reference g using (S6.4)-(S6.5).
    Validate on the sample (S6.6).
    Shared Theta ~ Uniform[alpha/4, alpha/2], U ~ Uniform(0,1).
    Accept base h iff empirical_error(h)-empirical_error(g) <= Theta.
    If no model is accepted: return g.
    Otherwise quantile-sample the accepted multiset using U and the fixed order.
    On the quantile map's entire uncovered null set: return the single model h_0.
\end{verbatim}\end{small}

For the full class on \(d\) atoms, use \eqref{legacy:eq:6.15} and return \(+1\) at SMALL points. For the realizable corollary, use \eqref{legacy:eq:6.17}--\eqref{legacy:eq:6.18} with the same query rule as PR-VC.

\begin{small}\begin{verbatim}
SOFT-VC(H, d, alpha, beta, rho):

    Fix the countable proper core from Lemma S6.1.
    Fix one enumeration and the first-minimizer ERM rule.

    K_E    := 2^24
    L      := log(4 / beta)
    T      := ceil(32 / rho^2)
    n      := ceil((32*K_E / alpha^2) *
                   (d + log(21) + L/T))
    lambda := 4*L / alpha
    M      := ceil(32*lambda^2 / rho^2)

    Split T*n + M iid observations into:
        S[1],...,S[T], each of size n;
        V, of size M.

    For j = 1,...,T:
        h[j] := deterministic proper ERM on S[j]

    For every observed x in V:
        C[x] := sum of labels of observations in V located at x
    For every unobserved x:
        C[x] := 0

    Shared randomness:
        independent U_x ~ Uniform[-1,1], for every x in X;
        jointly independent of all observations.

    Return the binary function evaluated at a query x by:
        p := (1/T) * sum_j h[j](x)
        q := clip(p + lambda*C[x]/M, -1, 1)
        return +1 if U_x <= q, and -1 otherwise
\end{verbatim}\end{small}

\begin{small}\begin{verbatim}
SOFT-VC-WITH-REPAIR(H, d, alpha, beta, rho):

    B := SOFT-VC(H, d, alpha/4, rho/128, rho)

    eta   := min(beta/8, rho/128)
    n_ref := ceil((2^22/alpha^2) * (d + log(2/eta)))
    n_val := ceil((128/alpha^2) * log(4/eta))

    Train a deterministic proper reference ERM g
        on an independent block of size n_ref.

    On another independent block of size n_val:
        compute empirical_error(B) and empirical_error(g)

    If empirical_error(B) > empirical_error(g) + alpha/2:
        return g
    Else:
        return B
\end{verbatim}\end{small}

All population quantities \(w_x,y_x,\zeta_x\) are used only in proofs. The algorithms do not use a distribution oracle. Countable shared randomness follows the random-string convention permitted by \legacyref{2.1}{Definition}.

\subsection{Version-specific pointers for the cited comparisons}\label{legacy:section:sources}
The following pointers retain the reference annotations of the first version. They identify
the source numbering used in the preceding comparisons. They are external statement
numbers, not additional results numbered in this appendix.

\paragraph{Approximate replicability.}
The comparisons with Hopkins, Impagliazzo, and Ye~\cite{HIY} refer to version 2 of
arXiv:2510.20200, revised 7 April 2026. In particular, \legacyref{6.5}{Remark} concerns
the statements and proofs in that version.

\paragraph{Coin testing and the score argument.}
For Impagliazzo, Lei, Pitassi, and Sorrell~\cite{ILPS}, the pointer is Section 7,
Lemma 7.2 of arXiv:2201.08430. Its coin lower bound uses replicability at intermediate
biases. The separate multi-coordinate calculation and cubic potential retained here
are in \legacyref{5.1}{Theorem}.

\paragraph{Risk filtering and shared truncation.}
For Bun et al.~\cite{BGH23}, the reduction from sign-one-way marginals to agnostic
learning is in Section 5.2.2, Algorithm 9 and Theorem 5.9 of arXiv:2303.12921.
Step 3 of Algorithm 9 aborts on overflow when a shattered point occurs more than
$d/\log^c d$ times, with the parameters and exponent $c$ specified in that source.
The finite-class risk filter using a shared random order is in Section 5.3.1,
Algorithm 10, and Section 5.3.2, Theorem 5.14. For the distinction between coupling
and correlated sampling, see Kalavasis et al.~\cite{KKMV23}, Appendix A.4.

\paragraph{The classical VC-one structure.}
The pointers for Bollob\'as, Pritchard, Rothvo\ss{}, and Scott~\cite{BPRS12} refer to
arXiv:1009.6144v3, May 2012. Section 4 contains Definition 16 (cross-free), Definition 17
(laminar), and the observation after Definition 16 relating VC dimension one to a
cross-free dual hypergraph. See also Theorem 6 and the remark after Lemma 21 on the
directed-tree representation, attributed there to Edmonds and Giles. The corresponding
primal description in Ben-David~\cite{BD15} uses Definition 3 (tree ordering), Claim 2,
Theorem 4, and Lemma 5. The use of this structure for the proper learner is proved in
\legacyref{6.1}{Lemma} and \legacyref{6.2}{Theorem}.

\paragraph{Success-probability boosting and fallback.}
For Banihashem et al.~\cite{BBEGH}, see Theorem 5 for the informal statement and
Section 7, Theorem 12 and Remark 5 for the formal boosting result. The comparison uses
arXiv:2604.10423v1. For Aamand et al.~\cite{AAC}, the numbering is that of
arXiv:2507.02842v2: Section 3.4, Proposition 3.10 and its following post-processing
paragraph discuss early stopping with fallback to a nonreplicable algorithm;
Section 6.3.1, Theorem 6.15 concerns coin testing. The pointwise agnostic repair used
here is stated and proved in \legacyref{4.3}{Lemma} and \legacyref{7.6}{Theorem}.

\endgroup
\clearpage
\section*{AI methodology}
AI tools assisted with mathematical discussion and checking, literature searches, manuscript
organization, drafting, language editing, and preparation of the \LaTeX{} source.
\bibliographystyle{alphaurl}
\bibliography{references}

\newcommand{\etalchar}[1]{$^{#1}$}
\begin{thebibliography}{KKMV23}

\bibitem[AAC{\etalchar{+}}26]{AAC}
Anders Aamand, Maryam Aliakbarpour, Justin~Y. Chen, Shyam Narayanan, and
  Sandeep Silwal.
\newblock On the structure of replicable hypothesis testers.
\newblock In {\em Proceedings of the 2026 Annual ACM-SIAM Symposium on Discrete
  Algorithms}, pages 5771--5823. SIAM, 2026.
\newblock Statement numbers refer to arXiv:2507.02842v2, September 2025.
\newblock URL: \url{https://arxiv.org/abs/2507.02842v2}, \href
  {https://doi.org/10.1137/1.9781611978971.205}
  {\path{doi:10.1137/1.9781611978971.205}}.

\bibitem[Azu67]{Azu67}
Kazuoki Azuma.
\newblock Weighted sums of certain dependent random variables.
\newblock {\em Tohoku Mathematical Journal, Second Series}, 19(3):357--367,
  1967.
\newblock \href {https://doi.org/10.2748/tmj/1178243286}
  {\path{doi:10.2748/tmj/1178243286}}.

\bibitem[BBE{\etalchar{+}}26]{BBEGH}
Kiarash Banihashem, MohammadHossein Bateni, Hossein Esfandiari, Samira
  Goudarzi, and MohammadTaghi Hajiaghayi.
\newblock Replicable composition, 2026.
\newblock arXiv:2604.10423v1.
\newblock URL: \url{https://arxiv.org/abs/2604.10423v1}.

\bibitem[BD15]{BD15}
Shai Ben-David.
\newblock 2 notes on classes with {Vapnik--Chervonenkis} dimension 1, 2015.
\newblock arXiv:1507.05307.
\newblock URL: \url{https://arxiv.org/abs/1507.05307}.

\bibitem[BGH{\etalchar{+}}23]{BGH23}
Mark Bun, Marco Gaboardi, Max Hopkins, Russell Impagliazzo, Rex Lei, Toniann
  Pitassi, Satchit Sivakumar, and Jessica Sorrell.
\newblock Stability is stable: Connections between replicability, privacy, and
  adaptive generalization.
\newblock In {\em Proceedings of the 55th Annual ACM Symposium on Theory of
  Computing}, 2023.
\newblock Statement numbers refer to arXiv:2303.12921.
\newblock URL: \url{https://arxiv.org/abs/2303.12921}, \href
  {https://doi.org/10.1145/3564246.3585246}
  {\path{doi:10.1145/3564246.3585246}}.

\bibitem[BLM13]{BLM13}
St{\'e}phane Boucheron, G{\'a}bor Lugosi, and Pascal Massart.
\newblock {\em Concentration Inequalities: A Nonasymptotic Theory of
  Independence}.
\newblock Oxford University Press, 2013.
\newblock \href {https://doi.org/10.1093/acprof:oso/9780199535255.001.0001}
  {\path{doi:10.1093/acprof:oso/9780199535255.001.0001}}.

\bibitem[BPRS13]{BPRS12}
B{\'e}la Bollob{\'a}s, David Pritchard, Thomas Rothvo{\ss}, and Alex Scott.
\newblock Cover-decomposition and polychromatic numbers.
\newblock {\em SIAM Journal on Discrete Mathematics}, 27(1):240--256, 2013.
\newblock Statement numbers refer to arXiv:1009.6144v3, May 2012.
\newblock URL: \url{https://arxiv.org/abs/1009.6144v3}, \href
  {https://doi.org/10.1137/110856332} {\path{doi:10.1137/110856332}}.

\bibitem[Dud67]{Dud67}
Richard~M. Dudley.
\newblock The sizes of compact subsets of {Hilbert} space and continuity of
  {Gaussian} processes.
\newblock {\em Journal of Functional Analysis}, 1(3):290--330, 1967.
\newblock \href {https://doi.org/10.1016/0022-1236(67)90017-1}
  {\path{doi:10.1016/0022-1236(67)90017-1}}.

\bibitem[Hau95]{Haussler1995}
David Haussler.
\newblock Sphere packing numbers for subsets of the {Boolean} $n$-cube with
  bounded {Vapnik--Chervonenkis} dimension.
\newblock {\em Journal of Combinatorial Theory, Series A}, 69(2):217--232,
  1995.
\newblock \href {https://doi.org/10.1016/0097-3165(95)90052-7}
  {\path{doi:10.1016/0097-3165(95)90052-7}}.

\bibitem[HIY26]{HIY}
Max Hopkins, Russell Impagliazzo, and Christopher Ye.
\newblock Approximate replicability in learning, 2026.
\newblock arXiv:2510.20200v2, revised 7 April 2026.
\newblock URL: \url{https://arxiv.org/abs/2510.20200v2}.

\bibitem[Hoe63]{Hoe63}
Wassily Hoeffding.
\newblock Probability inequalities for sums of bounded random variables.
\newblock {\em Journal of the American Statistical Association},
  58(301):13--30, 1963.
\newblock \href {https://doi.org/10.1080/01621459.1963.10500830}
  {\path{doi:10.1080/01621459.1963.10500830}}.

\bibitem[ILPS22]{ILPS}
Russell Impagliazzo, Rex Lei, Toniann Pitassi, and Jessica Sorrell.
\newblock Reproducibility in learning.
\newblock In {\em Proceedings of the 54th Annual ACM SIGACT Symposium on Theory
  of Computing}, 2022.
\newblock Statement numbers refer to arXiv:2201.08430.
\newblock URL: \url{https://arxiv.org/abs/2201.08430}, \href
  {https://doi.org/10.1145/3519935.3519973}
  {\path{doi:10.1145/3519935.3519973}}.

\bibitem[KKMV23]{KKMV23}
Alkis Kalavasis, Amin Karbasi, Shay Moran, and Grigoris Velegkas.
\newblock Statistical indistinguishability of learning algorithms.
\newblock In {\em Proceedings of the 40th International Conference on Machine
  Learning}, volume 202 of {\em Proceedings of Machine Learning Research},
  pages 15586--15622. PMLR, 2023.
\newblock URL: \url{https://proceedings.mlr.press/v202/kalavasis23a.html}.

\bibitem[Sau72]{Sau72}
Norbert Sauer.
\newblock On the density of families of sets.
\newblock {\em Journal of Combinatorial Theory, Series A}, 13(1):145--147,
  1972.
\newblock \href {https://doi.org/10.1016/0097-3165(72)90019-2}
  {\path{doi:10.1016/0097-3165(72)90019-2}}.

\bibitem[She72]{She72}
Saharon Shelah.
\newblock A combinatorial problem; stability and order for models and theories
  in infinitary languages.
\newblock {\em Pacific Journal of Mathematics}, 41(1):247--261, 1972.
\newblock URL: \url{https://msp.org/pjm/1972/41-1/p21.xhtml}, \href
  {https://doi.org/10.2140/pjm.1972.41.247}
  {\path{doi:10.2140/pjm.1972.41.247}}.

\bibitem[VC71]{VC71}
Vladimir~N. Vapnik and Alexey~Ya. Chervonenkis.
\newblock On the uniform convergence of relative frequencies of events to their
  probabilities.
\newblock {\em Theory of Probability and Its Applications}, 16(2):264--280,
  1971.
\newblock \href {https://doi.org/10.1137/1116025} {\path{doi:10.1137/1116025}}.

\bibitem[Zha26]{Zhang2026}
Guangjian Zhang.
\newblock Pointwise replicable agnostic learning: The additive term is
  unnecessary and the fixed-confidence order is exact, 2026.
\newblock Zenodo preprint, version 1, 9 September 2026.
\newblock URL: \url{https://zenodo.org/records/22658876}, \href
  {https://doi.org/10.5281/zenodo.22658876}
  {\path{doi:10.5281/zenodo.22658876}}.

\end{thebibliography}
\end{document}